\documentclass[twoside,11pt]{article}

\usepackage{jmlr2e}
\usepackage{amsmath}
\usepackage{bm}
\usepackage{algorithm}
\usepackage{algorithmic}

\newcommand{\calD}{\mathcal{D}}

\newcommand{\calF}{\mathcal{F}}
\newcommand{\calG}{\mathcal{G}}
\newcommand{\calH}{\mathcal{H}}

\newcommand{\calK}{\mathcal{K}}
\newcommand{\calL}{\mathcal{L}}

\newcommand{\calR}{\mathcal{R}}

\newcommand{\calV}{\mathcal{V}}
\newcommand{\calW}{\mathcal{W}}
\newcommand{\calX}{\mathcal{X}}
\newcommand{\calY}{\mathcal{Y}}

\newcommand{\bA}{\mathbf{A}}

\newcommand{\bs}{\mathbf{s}}

\newcommand{\bw}{\mathbf{w}}
\newcommand{\bW}{\mathbf{W}}
\newcommand{\bx}{\mathbf{x}}
\newcommand{\bX}{\mathbf{X}}
\newcommand{\by}{\mathbf{y}}
\newcommand{\bY}{\mathbf{Y}}

\newcommand{\bbE}{\mathbb{E}}

\newcommand{\bbN}{\mathbb{N}}

\newcommand{\bbP}{\mathbb{P}}

\newcommand{\bbR}{\mathbb{R}}

\newcommand{\bbZ}{\mathbb{Z}}
\newcommand{\nn}{\nonumber} 
\newcommand{\bone}{\mathbf{1}}

\newcommand{\tilD}{\tilde{D}}

\newcommand{\hatE}{\hat{E}}

\DeclareMathOperator{\sgn}{sgn}
\newcommand{\eps}{\varepsilon}

\ShortHeadings{Generalization Bounds for Deep Learning}{Lan V. Truong}
\firstpageno{1}

\begin{document}

\title{On Rademacher Complexity-based 
	Generalization Bounds for Deep Learning}

\author{\name Lan V. Truong \email lan.truong@essex.ac.uk \\
       \addr School of Mathematics, Statistics and Actuarial Science\\
       University of Essex\\
       Colchester, CO4 3SQ, UK}
      
\editor{}

\maketitle
\begin{abstract}
We show that the Rademacher complexity-based framework can establish non-vacuous generalization bounds for Convolutional Neural Networks (CNNs) in the context of classifying a small set of image classes. A key technical advancement is the formulation of novel contraction lemmas for high-dimensional mappings between vector spaces, specifically designed for general Lipschitz activation functions. These lemmas extend and refine the Talagrand contraction lemma across a broader range of scenarios. Our Rademacher complexity bound provides an enhancement over the results presented by Golowich et al. for ReLU-based Deep Neural Networks (DNNs). Moreover, while previous works utilizing Rademacher complexity have primarily focused on ReLU DNNs, our results generalize to a wider class of activation functions.

\end{abstract}


\section{Introduction}
Deep models are typically heavily over-parametrized, while they still achieve good generalization performance. Despite the widespread use of neural networks in biotechnology, finance, health science, and business, just to name a selected few, the problem of understanding deep learning theoretically remains relatively under-explored. In 2002, Koltchinskii and Panchenko \citep{Koltchinskii2002} proposed new probabilistic upper bounds on generalization error of the combination of many complex classifiers such as deep neural networks. These bounds were developed based on the general results of the theory of Gaussian, Rademacher, and empirical processes in terms of general functions of the margins, satisfying a Lipschitz condition. However, bounding Rademacher complexity for deep learning remains a challenging task. In this work, we present new upper bounds on the Rademacher complexity in deep learning, which differ from previous studies in how they depend on the norms of the weight matrices. Furthermore, we demonstrate that our bounds are non-vacuous for CNNs with a wide range of activation functions.
\subsection{Related Papers} 

The complexity-based generalization bounds were established by traditional learning theory aiming to provide general theoretical guarantees for deep learning. \citep{Goldberg1993}, \citep{Bartlett1996}, \citep{Bartlett1998A} proposed upper bounds based on the VC dimension for DNNs. \citep{Neyshabur2015} used Rademacher complexity to prove the bound with explicit exponential dependence on the network depth for ReLU networks. \citep{Neyshabur2018APA} and \citep{Bartlett2017} uses the PAC-Bayesian analysis and the covering number to obtain bounds with explicit polynomial dependence on the network depth, respectively. \citep{Golowich2018} provided bounds with explicit square-root dependence on the depth for DNNs with positive-homogeneous activations such as ReLU. 

The standard approach to develop generalization bounds on deep learning (and machine learning) was developed in seminar papers by \citep{Vap98}, and it is based on bounding the difference between the generalization error and the training error. These bounds are expressed in terms of the so called VC-dimension of the class. However, these bounds are very loose when the VC-dimension of the class can be very large, or even infinite. In 1998, several authors \citep{Bartlett1998,Bartlett1999} suggested another class of upper bounds on generalization error that are expressed in terms of the empirical distribution of the margin of the predictor (the classifier). Later, Koltchinskii and Panchenko \citep{Koltchinskii2002} proposed new probabilistic upper bounds on the generalization error of the combination of many complex classifiers such as deep neural networks. These bounds were developed based on the general results of the theory of Gaussian, Rademacher, and empirical processes in terms of general functions of the margins, satisfying a Lipschitz condition. They improved previously known bounds on generalization error of convex combination of classifiers. Generalization bounds for deep learning and kernel learning with Markov dataset based on Rademacher and Gaussian complexity functions have recently analysed in \citep{Truong2022BO}. Analysis of machine learning algorithms for Markov and Hidden Markov datasets already appeared in research literature \citep{Duchi2011ErgodicMD, Wang2019AML,Truong2022OnLM}. 

In the context of supervised classification, PAC-Bayesian bounds have been used to explain the generalisation capability of learning algorithms \citep{Langford2003,McAllester2004,Ambroladze2007}.  Several recent works have focused
on gradient descent based PAC-Bayesian algorithms,
aiming to minimise a generalisation bound for stochastic classifiers \citep{Dziugaite2017,Zhou2019, Biggs2021}. Most of these studies use a surrogate loss to avoid dealing with the zero-gradient of the misclassification loss. Several authors used other methods to estimate of the misclassification error with a non-zero gradient by proposing new training algorithms to evaluate the optimal output distribution in PAC-Bayesian bounds analytically \citep{McAllester1998,Eugenio2021a,Eugenio2021}.  Recently, \citep{NagarajanKolter2019} showed that uniform convergence might be unable to explain generalisation in deep learning by creating some examples where the test error is bounded by $\delta$  but the (two-sided) uniform convergence on this set of classifiers will yield only a vacuous generalisation guarantee larger than $1-\delta$ for some $\delta \in (0,1)$. This result is derived from evaluating the bounds presented in \citep{Neyshabur2018APA} and \citep{Bartlett2017}. 
There have been some interesting works which use information-theoretic approach to find PAC-bounds on generalization errors for machine learning \citep{XuRaginskyNIPS17, Esposito2021} and deep learning \citep{Jakubovitz2108}.

\subsection{Contributions} In this paper, our contributions are as follows:
\begin{itemize} \item We introduce novel contraction lemmas for high-dimensional mappings between vector spaces, which extend and enhance the Talagrand contraction lemma in various scenarios. \item We apply these new contraction lemmas to the layers of ReLU-based Deep Neural Networks (DNNs), demonstrating that our derived Rademacher complexity bound provides improvements over the results of Golowich et al. \citep{Golowich2018}. 
\item We empirically validate our theoretical findings on CNNs for MNIST image classification, showing that our generalization bounds are non-vacuous when the number of classes is small. 
\end{itemize}
As far as we know, this is the first result which shows that the Rademacher complexity-based approach can lead to non-vacuous generalisation bounds on CNNs. 
\subsection{Other Notations} Vectors and matrices are in boldface. 
For any vector $\bx=(x_1,x_2,\cdots,x_n) \in \bbR^n$ where $\bbR$ is the field of real numbers, its induced-$L^p$ norm is defined as
\begin{align}
\|\bx\|_p=\bigg(\sum_{k=1}^n |x_k|^p\bigg)^{1/p}. 
\end{align} 
The $j$-th component of the vector $\bx$ is denoted as $[\bx]_j$ for all $j \in [n]$. 

The Hadamard product between two vectors $\bx=(x_1,x_2,\cdots, x_n)$ and $\by=(y_1,y_2,\cdots, y_n)$ is defined as
\begin{align}
\bx \odot \by =(x_1 y_1, x_2y_2,\cdots, x_n y_n).
\end{align}

For $\bA \in \bbR^{m\times n}$ where
 \begin{align}
 \bA=\begin{bmatrix}a_{11},&a_{12},&\cdots,&a_{1n}\\ a_{21},&a_{22},&\cdots,&a_{2n}\\ \vdots& \vdots&\ddots& \vdots \\ a_{m1},& a_{m2},&\cdots, &a_{m n}\end{bmatrix}
 \end{align}
 we defined the induced-norm of matrix $\bA$ as
 \begin{align}
 \|\bA\|_{p,q}=\sup_{\bx \neq \b0} \frac{\|\bA \bx\|_q}{\|\bx\|_p}.
 \end{align}
 For abbreviation, we also use the following notation
 \begin{align}
 \|A\|_p&:= \|A\|_{p,p}.
 \end{align}
It is known that 
\begin{align}
\|\bA\|_1&=\max_{1\leq j \leq n} \sum_{i=1}^m |a_{ij}|,\\
\|\bA\|_2&=\sqrt{\lambda_{\max} \big(\bA \bA^T\big)},\\
\|\bA\|_{\infty}&=\max_{1\leq i \leq m} \sum_{j=1}^n |a_{ij}|,
\end{align} where $\lambda_{\max}(\bA \bA^T)$ is defined as the maximum eigenvalue of the matrix $\bA \bA^T$ (or the square of the maximum singular value of $\bA$).
\section{Contraction Lemmas in High Dimensional Vector Spaces}
First, we recall the Talagrand's contraction lemma.
\begin{lemma}\cite[Theorem 4.12]{LedouxT1991book} \label{lem:tala1} Let $\calH$ be a hypothesis set of functions mapping from some set $\calX$ to $\bbR$ and $\psi$ be a $\mu$-Lipschitz function from $\bbR \to \bbR$ for some $\mu>0$. Then, for any sample $S$ of $n$ points $\bx_1,\bx_2,\cdots,\bx_n \in \calX$, the following inequality holds:
	\begin{align}
	\bbE_{\boldsymbol{\eps}}\bigg[\sup_{h \in \calH} \bigg|\frac{1}{n} \sum_{i=1}^n\eps_i (\psi \circ h)(\bx_i)\bigg| \bigg]\leq  2\mu \bbE_{\boldsymbol{\eps}}\bigg[ \sup_{h \in \calH} \bigg|\frac{1}{n}\sum_{i=1}^n \eps_i h(\bx_i)\bigg|\bigg] \label{ato},
	\end{align}
\end{lemma} where $\boldsymbol{\eps}=(\eps_1,\eps_2,\cdots,\eps_n)$,  and $\{\eps_i\}_{i=1}^n$ is a sequence of i.i.d. Rademacher random variables (taking values $+1$ and $-1$ with probability $1/2$ each), independent of $\{\bx_i\}$. 

In Theorem \ref{aux_lem} and Theorem \ref{main:thm1} below, we present new versions of Talagrand's contraction lemma for the high-dimensional mapping $\psi$ between vector spaces. The proof of the these theorems are provided in Appendix \ref{aux_lem:proof} and Appendix \ref{proof:main:thm1}.

\begin{theorem} \label{aux_lem} Let $\calH$ be a set of functions mapping $\calX$ to $\bbR^m$ and $\calH_+=\calH \cup \{|h|: h \in \calH\}$ and $\psi: \bbR \to \bbR$ such that $\psi(x)= ReLU(x)-\alpha ReLU(-x) \enspace \forall x$ for some $\alpha \in [0,1]$. Then, for any $p\geq 1$ it holds that
\begin{align}
	&\bbE_{\boldsymbol{\eps}}\bigg[\sup_{h \in \calH} \bigg\| \frac{1}{n}\sum_{i=1}^n\eps_i \psi (h(\bx_i)) \bigg\|_p  \bigg] \leq  \bbE_{\boldsymbol{\eps}}\bigg[ \sup_{h \in \calH_+}\bigg\|\frac{1}{n}\sum_{i=1}^n \eps_i h(\bx_n)\bigg\|_p \bigg] \label{atonewa4}.
	\end{align} 
\end{theorem}
Identity, ReLU, Leaky ReLU, Parametric rectified linear unit (PReLU) belong to the class of functions $\calL:=\{\psi: \psi(x)=ReLU(x)-\alpha ReLU(-x) \enspace \forall x, \enspace \mbox{for some} \enspace \alpha \in \bbR\}$. 

\begin{theorem} \label{main:thm1}
Let $\calH$ be a set of functions mapping from some set $\calX$ to $\bbR^m$ for some $m \in \bbZ_+$ and
\begin{align}
\calL=\big\{\psi_{\alpha}: \psi_{\alpha}(x)= ReLU(x)-\alpha ReLU(-x) \enspace \forall x \in \bbR, \alpha \in [0,1]\big\} \label{defLa}
\end{align} where $ReLU(x)=\max(x,0)$. 

For any $\mu>0$, let  $\psi: \bbR \to \bbR$ be a $\mu$-Lipschitz function. Define
\begin{align}
\calH_+=\begin{cases} \calH \cup \{-h: h \in \calH\},\enspace \mbox{if} \enspace  \psi-\psi(0) \enspace \mbox{is odd}\\
\calH \cup \{-h: h \in \calH\}\cup \{|h|: h\in \calH\},\enspace \mbox{if} \enspace  \psi-\psi(0) \enspace \mbox{others}  \end{cases}.
\end{align}
	Then, it holds that
	\begin{align}
	&\bbE_{\boldsymbol{\eps}}\bigg[\sup_{h \in \calH} \bigg\| \frac{1}{n}\sum_{i=1}^n\eps_i \psi (h(\bx_i)) \bigg\|_{\infty}  \bigg]\nn\\
	&\qquad \leq \gamma(\mu) \bbE_{\boldsymbol{\eps}}\bigg[ \sup_{h \in \calH_+ }\bigg\|\frac{1}{n}\sum_{i=1}^n \eps_i h(\bx_i)\bigg\|_{\infty} \bigg]+\frac{1}{\sqrt{n}} \big|\psi(0)\big|, 
	\end{align} 
where
\begin{align}
\gamma(\mu)=\begin{cases} \mu, &\enspace \mbox{if} \enspace \psi-\psi(0) \enspace \mbox{is odd or belongs to $\calL$}\\  2\mu, &\enspace \mbox{if} \enspace \psi-\psi(0) \enspace \mbox{is even}\\  3\mu, &\enspace \mbox{if} \enspace \psi-\psi(0) \enspace \mbox{others} \end{cases}.
\end{align}	
Here, we define $\psi(\bx):=(\psi(x_1), \psi(x_2),\cdots, \psi(x_m))^T$ for any $\bx=(x_1,x_2,\cdots,x_m)^T \in \bbR^m$.
\end{theorem}
\begin{remark} Some remarks are in order. 
\begin{itemize}
\item Identity, ReLU, Leaky ReLU, Parametric rectified linear unit (PReLU) belong to the class of functions $\calL$. 
\item If $\psi$ is odd or belongs to $\calL$,  then $\psi(0)=0$. Therefore, Theorem \ref{main:thm1} improves Lemma \ref{lem:tala1} in the special case where $m=1$. This enhancement is achieved by leveraging the unique properties of certain function classes.
\item Our results are based on a novel approach, which shows that tighter contraction lemmas can be obtained when both the class of functions $\calH$ and the activation functions possess certain special properties.
 More specifically, in this work, we extend the class of functions $\calH$ by adding more functions, resulting in a new class $\calH_+$, which possesses certain special properties. Additionally, we restrict the class of activation functions to $\calL \cup \{\psi:\bbR \to \bbR: \psi(x)-\psi(0)=-(\psi(-x)-\psi(0)),\enspace \forall x \in \bbR\}$.
\end{itemize}
\end{remark}
Now, the following result can be easily proved (See Appendix~\ref{proof:lem:linear}).  
\begin{theorem}\label{lem:linear}
Let $\calG$ be a class of functions from $\bbR^{d_0} \to \bbR^{d_1}$ and $\calV$ be a class of matrices $\bW$ on $\bbR^{d_2 \times d_1}$ such that $\sup_{\bW \in \calV} \|\bW\|_p \leq \nu$ for some $p\geq 1$. Then, it holds that
\begin{align}
\bbE_{\boldsymbol{\eps}}\bigg[\sup_{\bW \in \calV}\sup_{f \in \calG} \bigg\|\frac{1}{n} \sum_{i=1}^n \eps_i \bW f(\bx_i) \bigg\|_p  \bigg] \leq \nu \bbE_{\boldsymbol{\eps}}\bigg[ \sup_{f \in \calG}\bigg\|\frac{1}{n}\sum_{i=1}^n \eps_i f(\bx_i)\bigg\|_p  \bigg] \label{alo2}. 
\end{align}   
\end{theorem} 
In addition, we can prove the following extensions of Talagrand's contraction lemma:
\begin{lemma} \label{mato1a}  Let $\psi:\bbR^L \to \bbR^K$ be an odd function such that  $\|\psi(\bx)-\psi(\by)\|_1 \leq \mu \|\bx-\by\|_1, \enspace \forall \bx, \by \in \bbR^L$ and $\calH:=\big\{h: \bbR^m \to \bbR^L\big\}$.   Let $\calH_+=\bigcup_{\sigma \in \{-1,+1\}^K} \calH_{\sigma} $ where
\begin{align}
\calH_{\sigma}=\big\{h_{\sigma}: h_{\sigma}(\bx)= (\sigma_1 h_1(\bx), \sigma_2 h_2(\bx),\cdots, \sigma_L h_L(\bx)\big\}.
\end{align}
Then, it holds that
\begin{align}
&\frac{1}{n}\bbE\bigg[\sup_{h \in \calH} \bigg\|\sum_{i=1}^n \eps_i \psi(h(\bx_i))\bigg\|_1 \bigg]\leq \frac{1}{n}\bbE\bigg[\sup_{h \in \calH_+} \bigg\|\sum_{i=1}^n \eps_i \psi(h(\bx_i))\bigg\|_1 \bigg]\nn\\
&\qquad \leq \frac{\mu}{n} \bbE\bigg[ \sup_{h \in \calH_+} \bigg\| \sum_{i=1}^n \eps_i h(\bx_i)\bigg\|_1\bigg] \label{A1}. 
\end{align}
\end{lemma}
\begin{proof}
See Appendix \ref{mato1a:proof}. 
\end{proof}
\begin{definition}  A function $f: \bbR^m \to \bbR$ is called even if and only if 
$$
f(\bx)=f(|\bx|) \qquad \forall \bx \in \bbR^m,
$$ where $|\bx|=(|x_1|,|x_2|,\cdots, |x_m|)^T$ if $\bx=(x_1,x_2,\cdots, x_m)$. 
\end{definition}
\begin{lemma} \label{mato1b}  Let $\psi:\bbR^L \to \bbR^K$ be an even function such that  $\|\psi(\bx)-\psi(\by)\|_1 \leq \mu \|\bx-\by\|_1, \enspace \forall \bx, \by \in \bbR^L$ and $\calH:=\big\{h: \bbR^m \to \bbR^L\big\}$.   Let $\calH_+=\bigcup_{\sigma \in \{-1,+1\}^K} \calH_{\sigma} $ where
\begin{align}
\calH_{\sigma}=\big\{h_{\sigma}: h_{\sigma}(\bx)= (\sigma_1 h_1(\bx), \sigma_2 h_2(\bx),\cdots, \sigma_L h_L(\bx)\big\}.
\end{align}  Then, it holds that
\begin{align}
\bbE\bigg[\sup_{h \in \calH} \frac{1}{n}\bigg\|\sum_{i=1}^n \eps_i \psi(h(\bx_i))\bigg\|_1 \bigg]\leq 2\mu  \bbE\bigg[ \sup_{h \in \calH_+} \frac{1}{n} \bigg\| \sum_{i=1}^n \eps_i h(\bx_i)\bigg\|_1\bigg] \label{A11}. 
\end{align}
\end{lemma}
\begin{proof} Since $\psi(x)$ is even, it holds that
\begin{align}
\bbE\bigg[\sup_{h \in \calH} \frac{1}{n}\bigg\|\sum_{i=1}^n \eps_i \psi(h(\bx_i))\bigg\|_1 \bigg]  =\bbE\bigg[\sup_{h \in \calH} \frac{1}{n}\bigg\|\sum_{i=1}^n \eps_i \psi(\big|h(\bx_i)\big| \big)\bigg\|_1 \bigg]  \label{anot1},
\end{align} 
Define
\begin{align}
\tilde{\psi}(\bx):=\psi\big(\bx \odot \bone\{\bx> 0 \}\big) - \psi\big(-\bx \odot \bone\{\bx< 0 \} \big)\qquad \forall \bx \in \bbR^L \label{A1},
\end{align} where we define $\bone\{\bx > 0\}$ and $\bone\{\bx < 0\}$ are vectors where for all $i \in [L]$ the $i$-th component is $1\{x_i > 0\}$ and $1\{x_i< 0\}$, respectively. Then, it is easy to see that $\tilde{\psi}$ is an odd function. 

On the other hand, we  also have
\begin{align}
\tilde{\psi}(|\bx|)= \psi(|\bx|), \qquad \forall \bx \in \bbR^L,
\end{align} so
\begin{align}
\bbE\bigg[\sup_{h \in \calH} \frac{1}{n}\bigg\|\sum_{i=1}^n \eps_i \psi(\big|h(\bx_i)\big| \big)\bigg\|_1 \bigg]=\bbE\bigg[\sup_{h \in \calH} \frac{1}{n}\bigg\|\sum_{i=1}^n \eps_i \tilde{\psi}(\big|h(\bx_i)\big| \big)\bigg\|_1 \bigg] \label{amu}. 
\end{align}

Furthermore, for all $\bx, \by \in \bbR^L$ we have
\begin{align}
&\big\|\tilde{\psi}(\bx)-\tilde{\psi}(\by)\big\|_1\nn\\
&\quad \leq \big\|\psi\big(\bx\odot \bone\{\bx> 0 \}\big)-\psi\big(\by\odot \bone\{\by> 0 \}\big)\big\|_1+ \big\|\psi\big(\bx\odot \bone\{\bx< 0 \}\big)-\psi\big(\by\odot \bone\{\by< 0 \}\big)\big\|_1\\
&\quad \leq \mu \big\|\bx\odot \bone\{\bx> 0 \}-\by\odot \bone\{\by > 0 \}\big\|_1+ \mu \big\|\bx\odot \bone\{\bx < 0 \}-\by\odot \bone\{\by < 0 \}\big\|_1 \label{Y1}
\end{align}
Now, observe that
\begin{align}
\big\|\bx\odot \bone\{\bx>0 \}-\by\odot \bone\{\by > 0 \}\big\|_1&=\sum_{i=1}^L \big|x_i \bone \{x_i > 0\}- y_i \bone\{y_i > 0\}\big|\\
&= \sum_{i=1}^L \bigg|\frac{x_i+|x_i|}{2}-\frac{ y_i+|y_i|}{2}\big|\\
&\leq \frac{1}{2} \sum_{i=1}^L |x_i-y_i|+ \frac{1}{2} \sum_{i=1}^L ||x_i|-|y_i||\\
&\leq \big\|\bx-\by\big\|_1 \label{Y2}
\end{align}
Similarly, we also have
\begin{align}
\big\|\bx \odot \bone\{\bx <0 \}-\by \odot \bone\{\by < 0 \}\big\|_1 \leq  \big\|\bx-\by\big\|_1 \label{Y3}.
\end{align}

From \eqref{Y1}, \eqref{Y2}, and \eqref{Y3} we obtain
\begin{align}
\big\|\tilde{\psi}(\bx)-\tilde{\psi}(\by)\big\|_1 \leq 2 \mu \big\|\bx-\by\big\|_1, \qquad \forall \bx, \by \in \bbR^L.
\end{align}

Hence, by Lemma \ref{mato1a} we have
\begin{align}
&\bbE\bigg[\sup_{h \in \calH} \frac{1}{n}\bigg\|\sum_{i=1}^n \eps_i \tilde{\psi}(\big|h(\bx_i)\big| \big)\bigg\|_1 \bigg]\nn\\
&\qquad \leq 2\mu \bbE\bigg[\sup_{h \in \calH_+} \frac{1}{n}\bigg\|\sum_{i=1}^n \eps_i \big|h(\bx_i)\big| \bigg\|_1 \bigg]\\
&\qquad= 2\mu \bbE\bigg[\sup_{h \in \calH_+} \frac{1}{n}\bigg\|\sum_{i=1}^n \eps_i h(\bx_i) \bigg\|_1 \bigg]
\label{amutab},
\end{align} where \eqref{amutab} follows by using the fact that $|h|\in \calH$ if $h \in \calH_+$. 

Hence, finally we have
\begin{align}
\bbE\bigg[\sup_{h \in \calH} \frac{1}{n}\bigg\|\sum_{i=1}^n \eps_i \psi(h(\bx_i))\bigg\|_1 \bigg] \leq 2\mu \bbE\bigg[\sup_{h \in \calH_+} \frac{1}{n}\bigg\|\sum_{i=1}^n \eps_i h(\bx_i) \bigg\|_1 \bigg] \label{amau}. 
\end{align} 
\end{proof}
\section{Rademacher Complexity Bounds  for Deep Neural Networks (DNNs)}
\subsection{General Deep Neural Network Models} \label{sec:dnm}
Let $d_0, d_1,\cdots,d_L ,d_{L+1}$ be a sequence of positive integer numbers such that $d_0=d$ for some fixed $d \in \bbZ_+$. 
We define a class of function $\calF$ as follows:
\begin{align}
\calF:=\big\{f=f_L\circ f_{L-1}\circ \cdots \circ f_1\circ f_0: f_i \in \calG_i \subset \{g_i: \bbR^{d_i} \to \bbR^{d_{i+1}}\}, \quad \forall i \in \{1,2,\cdots,L\} \big\},
\end{align}  where $f_0: [0,1]^d \to \bbR^{d_1}$ is a fixed function and $d_{L+1}=M$ for some $M\in \bbZ_+$. A Deep Neural Network (DNN) with network-depth $L$ is defined as a composition map $f \in \calF$ where
\begin{align}
f_i(\bx)=\sigma_i(\bW_i \bx), \quad \forall \bx \in \bbR^{d_i}.
\end{align} Here, $\bW_i \in  \calW_i$ where $\calW_i$ is a set of matrices in $\bbR^{d_{i+1} \times d_i}$, and $\sigma_i$ is a mapping from $\bbR^{d_{i+1}} \to \bbR^{d_{i+1}}$.
  
Given a function $f \in \calF$, a function $g \in \bbR^M \times [M]$ predicts a label $y \in [M]$ for an example $\bx \in \bbR^d$ if and only if 
\begin{align}
g(f(\bx),y)> \max_{y'\neq y} g(f(\bx),y')
\end{align} where $g(f(\bx),y)=\bw_y^T f(\bx)$ with $\bw_y=\underbrace{(0,0,\cdots,0,1,0,\cdots,0)}_{\bw_y(y)=1}$. 

For a training set $\{\bx_i\}_{i=1}^n$, the $\infty$-norm \emph{Rademacher complexity} for the class function $\calF$ is defined as
\begin{align}
\calR_n(\calF):=\bbE_{\boldsymbol{\eps}}\bigg[\sup_{f\in \calF} \bigg\|\frac{1}{n}\sum_{i=1}^n \eps_i f (\bx_i)\bigg\|_{\infty} \bigg] \label{defRM},
\end{align} 
where $\{\eps_i\}$ is a sequence of i.i.d. Rademacher random variables (taking values $+1$ and $-1$ with probability $1/2$ each), independent of $\{\bx_i\}$. 

The \emph{loss generalization gap} for the class function $\calF$ associated with a loss function $\psi$  is defined as
\begin{align}
L(\psi \circ \calF)=\sup_{f \in \calF} \bigg|\bbE_{(\bx,y) \sim \calD} \big[\psi(f(\bx),y)]-\frac{1}{n}\sum_{i=1}^n \psi(f(\bx_i),y_i)\bigg|. 
\end{align} 
Here,  $\{y_i\}_{i=1}^n$ is a sequence of labels, and $\calD$ is the distribution of $(\bx_i,y_i)$ for all $i \in [n]$.

For example,  for the square loss, i.e., $\psi(x,y)=(x-y)^2$, the loss generalization gap becomes
\begin{align}
L(\psi \circ \calF)=\sup_{f \in \calF} \bigg|\bbE_{(\bx,y) \sim \calD} \big[(f(\bx)-y)^2]-\frac{1}{n}\sum_{i=1}^n (f(\bx_i-y_i)^2 \bigg|. 
\end{align}

\subsection{Rademacher complexity bounds for ReLU-DNNs}

In this section, we aim to derive Rademacher complexity bounds for ReLU-based Deep Neural Networks (DNNs). We begin with the following result.
\begin{lemma} \label{lem:new}
	Let $\calH$ be a set of functions mapping $\calX$ to $\bbR^m$. Let $\b0$ be the zero-function on $\calX$, i.e., $\b0(\bx)=\b0$ for all $\bx \in \calX$.  Let $\calH_+=\calH \cup \{\b0\}$. 
	For any $\mu>0$, let $\Psi_{\mu}:=\{\tilde{\psi}: \calX \times \calY  \to \bbR: \big|\tilde{\psi}(x,y)-\tilde{\psi}(x',y)\big| \leq \mu |x-x'|, \enspace \forall (x, x' ,y) \in \calX \times \calX \times \calY \}$. Then, it holds that
	\begin{align}
	&\frac{1}{n}\bbE_{\boldsymbol{\eps}}\bigg[\sup_{h \in \calH} \bigg\| \sum_{i=1}^n\eps_i \tilde{\psi} (h(\bx_i),y_i) \bigg\|_{\infty}  \bigg]\nn\\
	&\qquad \leq \frac{c \mu}{n} \bbE_{\boldsymbol{\eps}}\bigg[ \sup_{h \in \calH_+ }\bigg\|\sum_{i=1}^n \eps_i h(\bx_i)\bigg\|_{\infty} \bigg]+\frac{1}{\sqrt{n}}\max_{y \in \calY}  \big|\tilde{\psi}(0,y)\big| 
	\end{align} for any fixed $\tilde{\psi} \in \Psi_{\mu}$. Here, we define $\tilde{\psi}(\bx):=(\tilde{\psi}(x_1), \tilde{\psi}(x_2),\cdots, \tilde{\psi}(x_m))^T$ for any $\bx=(x_1,x_2,\cdots,x_m)^T \in \bbR^m$.  In addition, $c=4$ if $m>1$ and $c=2$ if $m=1$. 	
\end{lemma}
\begin{proof}
	See Appendix \ref{lem:new:proof}.
\end{proof}

One such tool that can be used to upper bound the loss generalization gap is the Rademacher complexity. 
It is well known that (cf. \cite{ShalevShwartz2014UnderstandingML}) if the magnitude of our loss function is bounded above by $c$, with probability greater than $1-\delta$ for all $g \in \hat{\calF}$, we have
\begin{align}
\bigg|\bbE_{(\bx,y) \sim \calD} \big[\psi(g(\bx),y)]-\frac{1}{n}\sum_{i=1}^n \psi(g(\bx_i),y_i)\bigg| \leq 2 \calR_n(\psi \circ \hat{\calF})+4c \sqrt{\frac{2 \log \frac{4}{\delta}}{n}},
\end{align} where $\psi \circ \hat{\calF}=\{\psi(g(\bx),y): \bx, y \in \calX \times \calY, f \in \hat{\calF}\}$. Therefore, if we have an upper bound on the Rademacher complexity, we can have an upper bound on the loss generalization gap. 
\begin{lemma} \label{alem} If $\psi$ is the square loss, i.e., $\psi(x,y)=(x-y)^2$, for a DNN as an $M$-class classifier it holds that
\begin{align}
L(\psi \circ \hat{\calF})  \leq 4M \bbE_{\boldsymbol{\eps}}\bigg[ \sup_{g \in \hat{\calF}_+ }\bigg|\frac{1}{n}\sum_{i=1}^n \eps_i g(\bx_i)\bigg| \bigg]+ \frac{2M^2}{\sqrt{n}}+ 4M^2 \sqrt{\frac{2 \log \frac{4}{\delta}}{n}}, 
\end{align} where $\hat{\calF}_{+}=\hat{\calF} \cup \{0\}$, and $\hat{\calF}$ is the class of output functions of the DNN, i.e., $g=\|f\|_{\infty}$ for some $f \in \calF$ (see the modelling of ReLU-DNNs in Section \ref{sec:dnm}).   
\end{lemma}
\begin{proof} Observe that
\begin{align}
\calR_n(\psi\circ  \hat{\calF})= \bbE_{\boldsymbol{\eps}}\bigg[\sup_{g \in \hat{\calF}} \frac{1}{n}\bigg|\sum_{i=1}^n \eps_i \psi\big(g(\bx_i),y_i\big)\bigg|\bigg] \label{P1}. 
\end{align}
Now, we have
\begin{align}
\big|\psi(x,y)-\psi(x',y)\big|&=\big|(x-y)^2-(x'-y)^2\big|\\
&=\big|(x-x')(x+x'-2y)\big|\\
&=\big|x-x'\big|\big|x+x'-2y\big|\\
&\leq 2M \big|x-x'\big|. 
\end{align}
Hence, by applying Lemma \ref{lem:new} we have
\begin{align}
\bbE_{\boldsymbol{\eps}}\bigg[\sup_{g \in \hat{\calF}} \frac{1}{n}\bigg|\sum_{i=1}^n \eps_i \psi\big(g(\bx_i),y_i\big)\bigg|\bigg]\leq \frac{4M}{n} \bbE_{\boldsymbol{\eps}}\bigg[ \sup_{g \in \hat{\calF}_+ }\bigg|\sum_{i=1}^n \eps_i g(\bx_i)\bigg| \bigg]+\frac{2M^2}{\sqrt{n}}.  
\end{align}
\end{proof}
\begin{lemma} \label{max:lem} Assume that $\psi$ is the max activation function, i.e., $\psi(\bx)=\|\bx\|_{\infty}$ and $f:\bbR^K \to \bbR$.  Then, it holds that
\begin{align}
 \bbE_{\boldsymbol{\eps}}\bigg[ \sup_{f \in \calF}\bigg|\sum_{i=1}^n \eps_i \psi\circ f(\bx_i)\bigg| \bigg] \leq  2M\bbE_{\boldsymbol{\eps}}\bigg[ \sup_{f \in \calF_+}\bigg\|\sum_{i=1}^n \eps_i f(\bx_i)\bigg\|_{\infty} \bigg]. \end{align}
\end{lemma}
\begin{proof}
Observe that
\begin{align}
|\psi(\bx)-\psi(\by)|&=| \|\bx\|_{\infty}-\|\by\|_{\infty}|\\
&\leq \|\bx-\by\|_{\infty}\\
&\leq \|\bx-\by\|_1. 
\end{align} 
In addition, the function $\psi(\bx)$ is even since $\psi(|\bx|)=\psi(\bx)$. 
Hence, by Lemma \ref{mato1b}, we have
\begin{align}
 \bbE_{\boldsymbol{\eps}}\bigg[ \sup_{f \in \calF}\bigg|\sum_{i=1}^n \eps_i \psi\circ f(\bx_i)\bigg| \bigg]&\leq  2  \bbE_{\boldsymbol{\eps}}\bigg[ \sup_{f \in \calF}\bigg\|\sum_{i=1}^n \eps_i  f(\bx_i)\bigg\|_1 \bigg]\\
 &\leq 2M \bbE_{\boldsymbol{\eps}}\bigg[ \sup_{f \in \calF}\bigg\|\sum_{i=1}^n \eps_i  f(\bx_i)\bigg\|_{\infty} \bigg]. 
\end{align}
\end{proof}

The following Rademacher complexity bounds for ReLU-Dense Layers is a direct application of Theorem \ref{aux_lem} and Theorem \ref{lem:linear}. 
\begin{lemma} \label{softmax:lem} Assume that $\psi$ is the softmax activation function.  Then, it holds that
\begin{align}
 \bbE_{\boldsymbol{\eps}}\bigg[ \sup_{f \in \calF}\bigg\|\sum_{i=1}^n \eps_i \psi\circ f(\bx_i)\bigg\|_{\infty} \bigg] \leq  \bbE_{\boldsymbol{\eps}}\bigg[ \sup_{f \in \calF_+}\bigg\|\sum_{i=1}^n \eps_i f(\bx_i)\bigg\|_{\infty} \bigg]. \end{align}
\end{lemma}
\begin{proof}
This is a direct application of Lemma \ref{lem:new} with noting that
\begin{align}
\tilde{\psi}(x,y):=\frac{e^{x}}{e^x+y}
\end{align} has Lipschitz constant being equal to $1/4$ for each fixed $y$. 
\end{proof}

\begin{lemma}[ReLU-Dense Layers] \label{goodlemt} Recall the definition of $\calL$ in \eqref{defLa}. Let $\calV$ be a class of matrices $\bW$ on $\bbR^{d \times d'}$ such that $\sup_{\bW \in \calV} \|\bW\|_p \leq \beta$. For any vector $\bx=(x_1,x_2,\cdots,x_{d'})$, we denote by $\sigma(\bx):=(\sigma(x_1),\sigma(x_2),\cdots,\sigma(x_{d'}))^T$ where $\sigma\in \calL$.  Then, it holds that
\begin{align}
\bbE_{\boldsymbol{\eps}}\bigg[\sup_{\bW \in \calV}\sup_{f \in \calG} \bigg\|\frac{1}{n} \sum_{i=1}^n \eps_i \sigma(\bW f(\bx_i))\bigg\|_p  \bigg] 
\leq  \beta \bbE_{\boldsymbol{\eps}}\bigg[ \sup_{f \in \calG}\bigg\|\frac{1}{n}\sum_{i=1}^n \eps_i f(\bx_i)\bigg\|_p \bigg].
\end{align} 
\end{lemma}
Then, we can show the following result.

\begin{theorem} \label{lem:survive0} 
Let
\begin{align}
\calL=\big\{\psi_{\alpha}: \psi_{\alpha}(x)= ReLU(x)-\alpha ReLU(-x) \enspace \forall x \in \bbR, \alpha \in [0,1]\big\} \label{defLL}.
\end{align}
 Consider the DNN defined in Section \ref{sec:dnm} where $$[f_i(\bx)]_j=\sigma_i\big(\bw_{j,i}^T f_{i-1}(\bx)\big) \enspace \forall j \in [d_{i+1}]$$ and $\sigma_i \in \calL$. 
In addition, $f_0(\bx)=[\bx^T,1]^T, \enspace \forall \bx \in \bbR^d$ and $\bx$ is normalised such that $\|\bx\|_{\infty}\leq 1$.
Then, the Rademacher complexity, $\calR_n(\calF)$, satisfies 
\begin{align}
\calR_n(\calF) \leq   \inf_{p\geq 1} \inf_{L' \in [1,L]} 2\bigg(\frac{d+1}{n}\bigg)^{1/p} \bigg(\prod_{l=L'+1}^L  \|\bW_l\|_{\infty} \bigg) \bigg(\prod_{l=1}^{L'} \|\bW_l\|_p\bigg)  \label{defrademacherReLU}.
\end{align}
\begin{remark} By choosing $L'=L$, it holds from \eqref{defrademacherReLU} that
\begin{align}
\calR_n(\calF) \leq \inf_{p\geq 1} 2\bigg(\frac{d+1}{n}\bigg)^{1/p} \bigg(\prod_{l=1}^L  \|\bW_l\|_{\infty} \bigg),
\end{align} which improves Golowich et al.'s bound (cf. Section \ref{comp:sec}). 
\end{remark}
\end{theorem}
\begin{proof}
This is a direct application of Lemma \ref{goodlemt}.  
By the modelling of ReLU-DNNs in Section \ref{sec:dnm}, it holds that
\begin{align}
\calF_k:=\big\{f=f_k\circ f_{k-1}\circ \cdots \circ f_1\circ f_0: f_i \in \calG_i \subset \{g_i: \bbR^{d_i} \to \bbR^{d_{i+1}}\}, \quad \forall i \in \{1,2,\cdots,k\} \big\}
\end{align} 
and $\calF:=\calF_L$. 

For ReLU-DNNs, $f_l(\bx)=\sigma_l(W_l \bx))$ for all $l \in [L]$ where $W_l \in \calW_l$ (a set of matrices) and $\sigma_l \in \calL$.

Then, since $|\sigma_l|, -\sigma_l \in \Psi_l$, it is easy to see that
\begin{align}
\calF_{l,+} \subset \Psi_l (\calW_l  \calF_{l-1,+}), \qquad \forall l \in [L],
\end{align} where $\calF_{l,+}:=\calF_l  \cup \{|f|: f\in \calF_l\}$ is a supplement of $\calF_l$. 

On the other hand, by using Lemma \ref{goodlemt} and peeling layer by layer we have
\begin{align}
&\bbE_{\boldsymbol{\eps}}\bigg[\sup_{f \in \calF_{L,+}}\bigg\|\frac{1}{n}\sum_{i=1}^n \eps_i f(\bx_i)\bigg\|_\infty  \bigg] \nn\\
&\qquad  \leq \inf_{L' \in [1,L]} \bigg(\prod_{l=L'+1}^L  \|\bW_l\|_{\infty} \bigg) \bigg(\prod_{l=1}^{L'} \|\bW_l\|_p\bigg)  \bbE_{\boldsymbol{\eps}}\bigg[\sup_{f \in \calH_+}\bigg\|\frac{1}{n}\sum_{i=1}^n \eps_i f(\bx_i)\bigg\|_p \bigg] \label{hk1},
 \end{align}
where $\calH_+$ is the extended set of inputs to the CNN, i.e.,
\begin{align}
\calH_+=f_0 \cup \{|f_0|\}.
\end{align}
Now, observe that
\begin{align}
 &\bbE_{\boldsymbol{\eps}}\bigg[\sup_{f \in \calH_+}\bigg\|\frac{1}{n}\sum_{i=1}^n \eps_i f(\bx_i)\bigg\|_p \bigg]\nn\\
 &\qquad \leq  \bbE_{\boldsymbol{\eps}}\bigg[\max\bigg\{\bigg\|\frac{1}{n}\sum_{i=1}^n \eps_i f_0(\bx_i)\bigg\|_p,  \bigg\|\frac{1}{n}\sum_{i=1}^n \eps_i \big|f_0(\bx_i)\big| \bigg\|_p\bigg\}\bigg]\\
 &\qquad \leq   \bbE_{\boldsymbol{\eps}}\bigg[\bigg\|\frac{1}{n}\sum_{i=1}^n \eps_i f_0(\bx_i)\bigg\|_p\bigg]+ \bbE_{\boldsymbol{\eps}}\bigg[ \bigg\|\frac{1}{n}\sum_{i=1}^n \eps_i \big|f_0(\bx_i)\big| \bigg\|_p \bigg] \label{k1}. 
\end{align}
Furthermore, we have
\begin{align}
\bigg\|\frac{1}{n}\sum_{i=1}^n \eps_i f_0(\bx_i)\bigg\|_{\infty} \leq \frac{1}{n}\sum_{i=1}^n \|f_0(\bx_i)\|_{\infty} \leq 1. 
\end{align}
Hence, we have
\begin{align}
\bbE_{\boldsymbol{\eps}}\bigg[\bigg\|\frac{1}{n}\sum_{i=1}^n \eps_i f_0(\bx_i)\bigg\|_p\bigg] &\leq \bigg(\bbE_{\boldsymbol{\eps}}\bigg[\bigg\|\frac{1}{n}\sum_{i=1}^n \eps_i f_0(\bx_i)\bigg\|_p^p\bigg]\bigg)^{1/p}\\
&\leq \bigg(\bbE_{\boldsymbol{\eps}}\bigg[\bigg\|\frac{1}{n}\sum_{i=1}^n \eps_i f_0(\bx_i)\bigg\|_2^2 \bigg]\bigg)^{1/p} \label{k10}\\
&\leq \bigg(\frac{d+1}{n}\bigg)^{1/p} \label{k10b}, 
\end{align} where \eqref{k10} follows from the fact that $\|\bx\|_p^p \leq \|\bx\|_{\infty}^{p-2} \|\bx\|_2^2$ for any vector $\bx$. 

Similarly, we also have
\begin{align}
\bbE_{\boldsymbol{\eps}}\bigg[\bigg\|\frac{1}{n}\sum_{i=1}^n \eps_i \big|f_0(\bx_i)\big|\bigg\|_p\bigg] \leq \bigg(\frac{d+1}{n}\bigg)^{1/p} \label{k11}. 
\end{align}

By combining \eqref{hk1}, \eqref{k10b}, and \eqref{k11} we obtain
\begin{align}
\bbE_{\boldsymbol{\eps}}\bigg[\sup_{f \in \calF_{L,+}}\bigg\|\frac{1}{n}\sum_{i=1}^n \eps_i f(\bx_i)\bigg\|_p  \bigg] \leq \inf_{p\geq 1} \inf_{L' \in [1,L]} 2\bigg(\frac{d+1}{n}\bigg)^{1/p} \bigg(\prod_{l=L'+1}^L  \|\bW_l\|_{\infty} \bigg) \bigg(\prod_{l=1}^{L'} \|\bW_l\|_p\bigg) \label{pig}.  
\end{align}
\end{proof}


\subsection{Rademacher complexity bounds for CNNs}
\subsubsection{Some Contraction Lemmas for CNNs}
Based on Theorem \ref{main:thm1} and Theorem \ref{lem:linear}, the following versions of Talagrand's contraction lemma for different layers of CNN are derived. 

\begin{definition}[Convolutional Layer with Average Pooling] \label{defconvlavg}
Let $\calG$ be a class of $\mu$-Lipschitz function $\sigma$ from $\bbR \to \bbR$ such that $\sigma(0)$ is fixed. Let $C, Q \in \bbZ_+$, $\{r_l,\tau_l\}_{l \in [Q]}$ be two tuples of positive integer numbers, and  $\{W_{l,c}  \in \bbR^{r_l \times r_l}, c \in [C], l \in [Q]\}$ be a set of kernel matrices. 
A convolutional layer with average pooling, $C$ input channels, and $Q$ output channels is defined as a set of  $Q \times C$ mappings $\Psi=\{\psi_{l,c}, l \in [Q], c\in [C]\}$ from $\bbR^{d \times d} $ to $\bbR^{\lceil (d-r_l+1)/\tau_l\rceil\times\lceil (d-r_l+1)/\tau_l\rceil }$
such that
\begin{align}
\psi_{l,c}(\bx)=\sigma_{\rm{avg}}\circ \sigma_{l,c}(\bx),
\end{align}
where
\begin{align}
\sigma_{\rm{avg}}(\bx)&=\frac{1}{\tau_l^2}\bigg(\sum_{k=1}^{\tau_l^2} x_k, \cdots, \sum_{k=(j-1)\tau_l^2+1 }^{j\tau_l^2} x_k, \cdots, \sum_{k=\lceil (d-r_l+1)^2/\tau_l^2\rceil-r_l^2+1}^{ \lceil (d-r_l+1)^2/\tau_l^2\rceil \tau_l^2} x_k\bigg), \nn\\
&\qquad \qquad \qquad \forall \bx \in \bbR^{ \lceil (d-r_l+1)^2/\tau_l^2\rceil \tau_l^2},
\end{align}
and for all $\bx \in \bbR^{d\times d \times C}$, 
\begin{align}
\sigma_{l,c}(\bx)&=\{\hat{x}_c(a,b)\}_{a,b=1}^{d-r_l+1},\\
\hat{x}_c(a,b)&=\sigma\bigg(\sum_{u=0}^{r_l-1}\sum_{v=0}^{r_l-1} x(a+u,b+v,c) W_{l,c}(u+1,v+1)\bigg). 
\end{align}
\end{definition}

\begin{lemma} [Convolutional Layer with Average Pooling] \label{goodlem0inftyto1} 
Let $\calF$ be a set of functions mapping from some set $\calX$ to $\bbR^m$ for some $m \in \bbZ_+$. Consider a convolutional layer with average pooling defined in Definition \ref{defconvlavg}. Recall the definition of $\calL$ in \eqref{defLa}. 
Then, it hold that
\begin{align}
&\bbE_{\boldsymbol{\eps}}\bigg[\sup_{c \in [C]} \sup_{l \in [Q]}\sup_{\psi_l \in \Psi} \sup_{f \in \calF} \bigg\|\frac{1}{n}\sum_{i=1}^n \eps_i \psi_{l,c} \circ f(\bx_i)\bigg\|_{\infty}   \bigg]\nn\\
&\quad \leq \bigg[\gamma(\mu) \sup_{c \in [C]} \sup_{l \in [Q]}\bigg(\sum_{u=0}^{r_l-1}\sum_{v=0}^{r_l-1}  \big|W_{l,c}(u+1,v+1)  \big|\bigg) \bigg]   \bbE\bigg[ \sup_{f \in \calF_+} \bigg\|\frac{1}{n}\sum_{i=1}^n \eps_i f(\bx_i)\bigg\|_{\infty} \bigg]+\frac{|\sigma(0)|}{\sqrt{n}},
\end{align}
where
\begin{align}
\gamma(\mu)=\begin{cases} \mu, &\enspace \mbox{if} \enspace \sigma-\sigma(0) \enspace \mbox{is odd or belongs to $\calL$}\\  2\mu, &\enspace \mbox{if} \enspace \sigma-\sigma(0) \enspace \mbox{is even}\\  3\mu, &\enspace \mbox{if} \enspace \sigma-\sigma(0) \enspace \mbox{others} \end{cases}.
\end{align}
Here,
\begin{align}
\calF_+=\begin{cases} \calF \cup \{-f: f \in \calF\},\enspace \mbox{if} \enspace  \sigma-\sigma(0) \enspace \mbox{is odd}\\
\calF \cup \{-f: f \in \calF\}\cup \{|f|: f\in \calF\},\enspace \mbox{if} \enspace  \sigma-\sigma(0) \enspace \mbox{others}  \end{cases} \label{defFplus}.
\end{align}
\end{lemma} 

For Dropout layer, the following holds:
\begin{lemma}[Dropout Layers]\label{dropoutlem} Let $\psi(\bx)$ is the output of the $\bx$ via the Dropout layer.  Then, it holds that
\begin{align}
\bbE_{\boldsymbol{\eps}}\bigg[\sup_{f \in \calH}\bigg\| \frac{1}{n}\sum_{i=1}^n  \eps_i \psi\circ f(\bx_i)\bigg\|_{\infty}\bigg]\leq \bbE\bigg[\sup_{f \in \calH}\bigg\| \frac{1}{n}\sum_{i=1}^n \eps_i f(\bx_i)\bigg\|_{\infty}\bigg]. 
\end{align}
\end{lemma}

The following Rademacher complexity bounds for Dense Layers.

\begin{lemma}[Dense Layers] \label{goodlem} Recall the definition of $\calL$ in \eqref{defLa}. Let $\calG$ be a class of  $\mu$-Lipschitz function, i.e.,
\begin{align}
\big|\sigma(x)-\sigma(y)\big| \leq \mu |x-y|, \qquad \forall x, y \in \bbR,
\end{align} such that $\sigma(0)$ is fixed. 
Let $\calV$ be a class of matrices $\bW$ on $\bbR^{d \times d'}$ such that $\sup_{\bW \in \calV} \|\bW\|_{\infty} \leq \beta$. For any vector $\bx=(x_1,x_2,\cdots,x_{d'})$, we denote by $\sigma(\bx):=(\sigma(x_1),\sigma(x_2),\cdots,\sigma(x_{d'}))^T$.  Then, it holds that
\begin{align}
&\bbE_{\boldsymbol{\eps}}\bigg[\sup_{\bW \in \calV}\sup_{f \in \calG} \bigg\|\frac{1}{n} \sum_{i=1}^n \eps_i \sigma(\bW f(\bx_i))\bigg\|_{\infty}  \bigg] \nn\\
&\qquad \qquad \leq \gamma(\mu)  \beta \bbE_{\boldsymbol{\eps}}\bigg[ \sup_{f \in \calG}\bigg\|\frac{1}{n}\sum_{i=1}^n \eps_i f(\bx_i)\bigg\|_{\infty}  \bigg] +   \frac{|\sigma(0)|}{\sqrt{n}},
\end{align} 
where
\begin{align}
\gamma(\mu)=\begin{cases} \mu, &\enspace \mbox{if} \enspace \sigma-\sigma(0) \enspace \mbox{is odd or belongs to $\calL$}\\  2\mu, &\enspace \mbox{if} \enspace \sigma-\sigma(0) \enspace \mbox{is even}\\  3\mu, &\enspace \mbox{if} \enspace \sigma-\sigma(0) \enspace \mbox{others} \end{cases}.
\end{align}
\end{lemma}
\begin{remark}  The convolutional layer with average pooling, dropout layers, and dense layers can be viewed as compositions of linear mappings and pointwise activation functions. Therefore, Lemmas \ref{goodlem0inftyto1}, \ref{dropoutlem}, and \ref{goodlem} are derived by applying Theorem \ref{main:thm1} to the pointwise mappings and Theorem \ref{lem:linear} to the linear mappings. 
\end{remark}

\subsubsection{Rademacher complexity bounds for CNNs}
In this section, we show the following result. 
\begin{theorem} \label{lem:survive} 
Let
\begin{align}
\calL=\big\{\psi_{\alpha}: \psi_{\alpha}(x)= ReLU(x)-\alpha ReLU(-x) \enspace \forall x \in \bbR, \alpha \in [0,1]\big\} \label{defL}.
\end{align}
 Consider the CNN defined in Section \ref{sec:dnm} where $$[f_i(\bx)]_j=\sigma_i\big(\bw_{j,i}^T f_{i-1}(\bx)\big) \enspace \forall j \in [d_{i+1}]$$ and $\sigma_i$ is $\mu_i$-Lipschitz. 
In addition, $f_0(\bx)=[\bx^T,1]^T, \enspace \forall \bx \in \bbR^d$ and $\bx$ is normalised such that $\|\bx\|_{\infty}\leq 1$.
 Let
 \begin{align}
 \calK&=\{i \in [L]: \mbox{layer} \enspace  i \enspace \mbox{is a convolutional layer with average pooling}\},\\
  \calD&=\{i \in [L]: \mbox{layer} \enspace  i \enspace \mbox{is a dropout layer}\}.
 \end{align}
 We assume that there are $Q_i$ kernel matrices $W_i^{(l)}$'s of size $r_i^{(l)}\times r_i^{(l)}$ for the $i$-th convolutional layer. 
 For all the (dense) layers that are not convolutional, we define $\bW_i$ as their coefficient matrices.  In addition, define
	\begin{align}
 \gamma_{\rm{cvl,i}}&=\gamma(\mu_i) \sup_{l \in [Q_i]}\sum_{u=1}^{r_{i,l}} \sum_{v=1}^{r_{i,l}} \big|W_i^{(l)}(u,v)\big|, \\
 \gamma_{\rm{dl,i}}&=\gamma(\mu_i) \big\|\bW_i \big\|_{\infty}  \qquad  i \notin \calK.
	\end{align}    
where
\begin{align}
\gamma(\mu_i)=\begin{cases} \mu_i, &\enspace \mbox{if} \enspace \sigma_i-\sigma_i(0) \enspace \mbox{is odd or belongs to $\calL$}\\  2\mu, &\enspace \mbox{if} \enspace \sigma_i-\sigma_i(0) \enspace \mbox{is even}\\  3\mu, &\enspace \mbox{if} \enspace \sigma_i-\sigma_i(0) \enspace \mbox{others} \end{cases}.
\end{align}
Then, the Rademacher complexity, $\calR_n(\calF)$, satisfies 
\begin{align}
\calR_n(\calF)&:= \bbE_{\boldsymbol{\eps}}\bigg[\sup_{f \in \calF_+}\bigg\|\frac{1}{n}\sum_{i=1}^n \eps_i f(\bx_i)\bigg\|_{\infty} \bigg] \nn\\
&\qquad \leq F_L \label{defrademacher},
\end{align}
where $F_L$ is estimated by the following recursive expression:
\begin{align}
F_i=\begin{cases} F_{i-1} \gamma_{\rm{cvl,i}} + \frac{|\sigma_i(0)|}{\sqrt{n}},&  i \in \calK\\ F_{i-1} \gamma_{\rm{dl,i}} + \frac{|\sigma_i(0)|}{\sqrt{n}},&  i \notin (\calK\cup \calD) \\F_{i-1},&  i \in \calD \end{cases}
\end{align} and $F_0=  \sqrt{\frac{d+1}{n}} $. 
\end{theorem}
\begin{proof}
This is a direct application of Lemmas \ref{goodlem0inftyto1}, \ref{dropoutlem}, and \ref{goodlem}.  
\end{proof}
\section{Generalization Bounds for DNNs}
\subsection{Loss Generalization Bounds for ReLU-DNNs}
In this section, we present a generalization bound for the loss in ReLU-DNNs. Specifically, the following result provides a bound on the square loss generalization gap for ReLU-DNNs as $M$-class classifiers. 
\begin{theorem} \label{thm:main3} The square loss generalization gap for a ReLU-DNN acting as an 
$M$-class classifier, where the final layer activation function is the softmax or sigmoid, satisfies the following:
\begin{align}
&\sup_{g \in \hat{\calF}} \bigg|\bbE_{(\bx,y) \sim \calD} \big[\psi(g(\bx),y)]-\frac{1}{n}\sum_{i=1}^n \psi(g(\bx_i),y_i)\bigg| \nn\\
&\qquad  \leq 16M^2\inf_{p\geq 1} \inf_{L' \in [1,L]} \bigg(\frac{d+1}{n}\bigg)^{1/p} \bigg(\prod_{l=L'+1}^L  \|\bW_l\|_{\infty} \bigg) \bigg(\prod_{l=1}^{L'} \|\bW_l\|_p\bigg)\nn\\
&\qquad \qquad + \frac{2M^2}{\sqrt{n}}+ 4M^2 \sqrt{\frac{2 \log \frac{4}{\delta}}{n}}.
\end{align}
\end{theorem}
\begin{proof}
By Lemma \ref{softmax:lem} and Lemma \ref{max:lem},  we have
\begin{align}
&\bigg|\bbE_{(x,y) \sim \calD} \big[\psi(g(\bx),y)]-\frac{1}{n}\sum_{i=1}^n \psi(g(\bx_i),y_i)\bigg| \nn\\
&\qquad  \leq 4M \bbE_{\boldsymbol{\eps}}\bigg[ \sup_{g \in \hat{\calF}_+ }\bigg\|\frac{1}{n}\sum_{i=1}^n \eps_i g(\bx_i)\bigg\|_{\infty} \bigg]+ \frac{2M^2}{\sqrt{n}}+ 4M^2 \sqrt{\frac{2 \log \frac{4}{\delta}}{n}}\\
&\qquad  \leq 8M^2 \bbE_{\boldsymbol{\eps}}\bigg[ \sup_{f \in \calF_+ }\bigg\|\frac{1}{n}\sum_{i=1}^n \eps_i f(\bx_i)\bigg\|_{\infty} \bigg]+ \frac{2M^2}{\sqrt{n}}+ 4M^2 \sqrt{\frac{2 \log \frac{4}{\delta}}{n}}\\
&\qquad \leq 16 M^2 \inf_{p\geq 1} \inf_{L' \in [1,L]} \bigg(\frac{d+1}{n}\bigg)^{1/p} \bigg(\prod_{l=L'+1}^L  \|\bW_l\|_{\infty} \bigg) \bigg(\prod_{l=1}^{L'} \|\bW_l\|_p\bigg)\nn\\
&\qquad \qquad + \frac{2M^2}{\sqrt{n}}+ 4M^2 \sqrt{\frac{2 \log \frac{4}{\delta}}{n}} \label{cocy},
\end{align} where \eqref{cocy} follows from Theorem \ref{lem:survive0}. 
\end{proof}
\subsection{Generalization Error Bounds for CNNs}

\begin{definition}
Recall that the CNN model is a special case of DNNs as discussed in Section \ref{sec:dnm}. The margin of a labelled example $(\bx,y)$ is defined as
\begin{align}
m_f(\bx,y):=g(f(\bx),y)-\max_{y'\neq y} g(f(\bx),y'),
\end{align}  so $f$ mis-classifies the labelled example $(\bx,y)$ if and only if $m_f(\bx,y)< 0$.
The generalisation error is defined as $\bbP(m_f(\bx,y)< 0)$.  It is easy to see that 
	$\bbP(m_f(\bx,y)< 0)=\bbP\big(\bw_y^T f(\bx)< \max_{y'\in \calY} \bw_{y'}^T f(\bx)\big)$. 
\end{definition}
\begin{remark} In case $$\{g(f(\bx),y)< \max_{y'\neq y} g(f(\bx),y')\} \subset \{\tilde{g}(f_k(\bx),y))<\max_{y'\neq y} \tilde{g}(f_k(\bx),y'))\}$$ for some $k \in [L]$ where $\tilde{g}$ is a specific function, we have $$\bbP(m_f(\bx,y)< 0)\leq \bbP(\tilde{g}(f_k(\bx),y))<\max_{y'\neq y} \tilde{g}(f_k(\bx),y'))).$$ Hence, we can bound the generalisation error by using only a part of CNN networks (from layer $0$ to layer $k$). More specifically, if the last layers of CNN is the softmax,  it holds that
\begin{align}
\bbP(m_f(\bx,y)<0)\leq \bbP(g(f_{L-1}(\bx),y))<\max_{y'\neq y} g(f_{L-1}(\bx),y'))),
\end{align} where $f_{L-1}$ represents the output of the last hidden layer, just before the final activation function (i.e., softmax).
\end{remark}

Now, we prove the following lemma.
	
\begin{lemma} \label{lem:abound} Let $\calF$ be a class of function from $\calX$ to $\bbR^m$.  For CNNs for classification, it holds that
\begin{align}
\bbE_{\boldsymbol{\eps}}\bigg[\sup_{f \in \calF} \bigg|\frac{1}{n}\sum_{i=1}^n \eps_i m_f(\bx_i,y_i)\bigg|\bigg]\leq \beta(M) \bbE_{\boldsymbol{\eps}}\bigg[\sup_{f \in \calF} \bigg\|\frac{1}{n}\sum_{i=1}^n \eps_i m_f(\bx_i)\bigg\|_{\infty} \bigg] \label{kbound},
\end{align}
where
\begin{align}
\beta(M)=\begin{cases} M(2M-1),& \qquad M>2 \\ 2M,& \qquad M=2\end{cases} \label{defbetaM}.
\end{align}
\end{lemma}
For $M>2$,  \eqref{kbound} is a result of \citep[Proof of Theorem 11]{Koltchinskii2002}. We improve this constant for $M=2$. 
Based on the above Rademacher complexity bounds and a justified application of McDiarmid's inequality, we obtains the following generalization for deep learning with i.i.d. datasets.
\begin{theorem} \label{main}  Let $\gamma>0$ and define the following function (the $\gamma$-margin cost):
	\begin{align}
	\zeta(x):=\begin{cases} 0, &\gamma\leq x \\ 1- x/\gamma,&0\leq x \leq \gamma \\ 1, & x\leq 0  \end{cases} \label{def1}.
	\end{align}
Recall the definition of the average Rademacher complexity $\calR_n(\calF)$ in \eqref{defrademacher} and the definition of $\beta(M)$ in \eqref{defbetaM}. Let $\{(\bx_i,y_i)\}_{i=1}^n \sim P_{\bx y}$ for some joint distribution $P_{\bx y}$ on $\calX \times \calY$. Then, for any $t>0$, the following holds:
\begin{align}
&\bbP\bigg\{\exists f \in \calF: \bbP\big(m_f(\bx,y)\leq 0 \big)> \inf_{\gamma\in (0,1]}\bigg[  \frac{1}{n}\sum_{i=1}^n \zeta(m_f(\bx_i,y_i))\nn\\
&\quad +  \frac{2\beta(M)}{\gamma} \calR_n(\calF) + \frac{2t+\sqrt{\log \log_2 (2 \gamma^{-1})}}{\sqrt{n}}\bigg]\bigg\} \leq 2\exp(-2t^2).
\end{align} 
\end{theorem}
\begin{corollary}(PAC-bound)\label{PACbound} Recall the definition of the average Rademacher complexity $\calR_n(\calF)$ in \eqref{defrademacher} and the definition of $\beta(M)$ in \eqref{defbetaM}. Let $\{(\bx_i,y_i)\}_{i=1}^n \sim P_{\bx y}$ for some joint distribution $P_{\bx y}$ on $\calX \times \calY$. Then, for any $\delta \in (0,1]$, with probability at least $1-\delta$, it holds that
\begin{align}
&\bbP\big(m_f(\bx,y)\leq 0 \big)\leq \inf_{\gamma\in (0,1]} \bigg[ \frac{1}{n}\sum_{i=1}^n \bone\big\{m_f(\bx_i,y_i)\leq \gamma \big\}\nn\\
&\quad +   \frac{2\beta(M)}{\gamma} \calR_n(\calF)
+  \sqrt{\frac{\log \log_2 (2\gamma^{-1})}{n}} +\sqrt{\frac{2}{n}\log \frac{3}{\delta}}\bigg], \qquad \forall f \in \calF \label{outmatt}.
\end{align} 
\end{corollary}
\begin{proof}		
This result is obtain from Theorem \ref{main} by choosing $t>0$ such that $3\exp(-2t^2)=\delta$. 
\end{proof}

\section{Numerical Results}
\subsection{Experiment 1}
\begin{figure}[htbp]
\begin{center}
\includegraphics[scale=0.7]{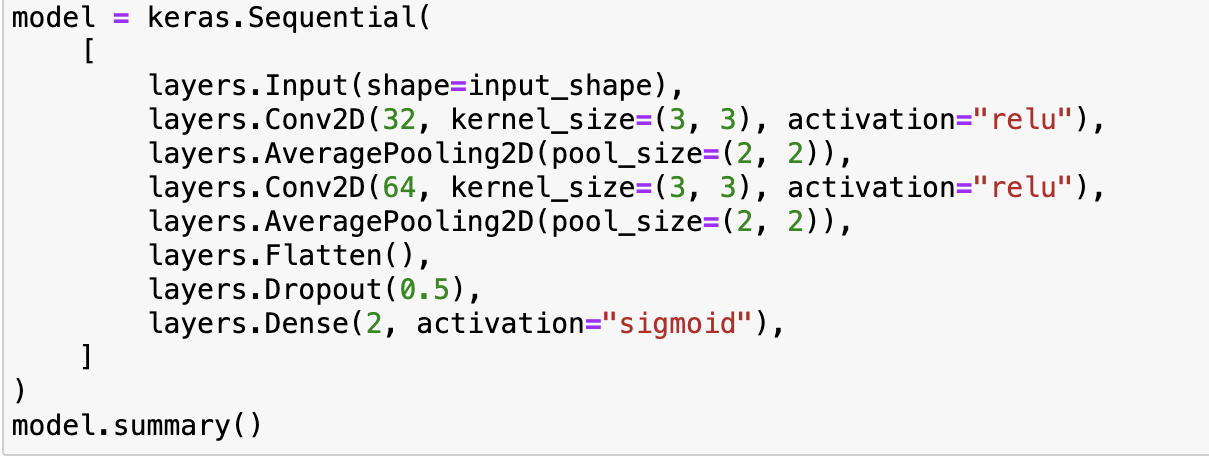}
\caption{CNN model with ReLU activations}
\label{CNNsigmoid2}
\end{center}
\end{figure}
In this experiment, we use a CNN (cf. Fig.~\ref{CNNsigmoid2}) for classifying MNIST images (class 0 and class 1), i.e., $M=2$, which consists of $n=12665$ training examples.

For this model,  we use ReLU for the first two convolutional layers, and the sigmoid $\sigma$ for the dense layer which satisfies $\sigma(x)-\sigma(0)= \frac{1}{2}\tanh\big(\frac{x}{2}\big)$ (an odd function with Lipschitz constant $1/4$). 

Hence, by Theorem \ref{lem:survive0} and Lemma \ref{lem:exttalaa} it holds that
$\calR_n(\calF) \leq F_3$,
where
\begin{align}
F_3 &\leq \underbrace{\frac{1}{4} \|\bW\|_{\infty}F_2 +\frac{1}{2\sqrt{n}}}_{\mbox{Dense layer}},\\
F_2 &\leq \underbrace{\sup_{l \in [64]}\big\|\bW_2^{(l)}\big\|_2 F_1 }_{\mbox{The second convolutional layer}},\\
F_1&\leq \underbrace{ \sup_{l \in [32]}\big\|\bW_1^{(l)}\big\|_2 F_0 }_{\mbox{The first convolutional layer}},\\
F_0&= \sqrt{\frac{d+1}{n}}. 
\end{align}
Numerical estimation of $F_3$ gives $\calR_n(\calF)\leq 0.01594$. 

By Corollary \ref{PACbound} with probability at least $1-\delta$, it holds that
\begin{align}
&\bbP\big(m_f(\bx,y)\leq 0 \big)\leq \inf_{\gamma\in (0,1]} \bigg[ \frac{1}{n}\sum_{i=1}^n \zeta\big(m_f(\bx_i,y_i) \big)\nn\\
&\quad +   \frac{4M}{\gamma} \calR_n(\calF)+   \sqrt{\frac{\log \log_2 (2\gamma^{-1})}{n}} +\sqrt{\frac{2}{n}\log \frac{3}{\delta}}\bigg]
\end{align}
By setting $\delta=5\%$, $\gamma=1$,  the generalisation error can be upper bounded by
\begin{align}
\bbP\big(m_f(\bx,y)\leq 0 \big)\leq 0.158446. 
\end{align}
For this model, the reported test error is $0.0009456$.  

By Theorem \ref{thm:main3} with $M=2$, the square loss generalization gap satisfies the following:
\begin{align}
\mbox{MSE}:=&\sup_{g \in \hat{\calF}} \bigg|\bbE_{(\bx,y) \sim \calD} \big[\psi(g(\bx),y)]-\frac{1}{n}\sum_{i=1}^n \psi(g(\bx_i),y_i)\bigg| \nn\\
&\qquad  \leq 16 F_2+ \frac{2}{\sqrt{n}}+ 4 \sqrt{\frac{2 \log \frac{4}{\delta}}{n}}.
\end{align}
At $\delta=0.05$, it holds that
\begin{align}
\mbox{MSE}\leq 0.1865. 
\end{align}
For this model, the reported MSE is $0.00032$.

\subsection{Experiment 2}
In this experiment, we use a CNN (cf. Fig.~\ref{CNNsigmoid}) for classifying MNIST images (class 0 and class 1), i.e., $M=2$, which consists of $n=12665$ training examples.

For this model,  the sigmoid activation $\sigma$ satisfies $\sigma(x)-\sigma(0)= \frac{1}{2}\tanh\big(\frac{x}{2}\big)$ which is odd and has the Lipschitz constant $1/4$. 
In addition, for the dense layer, the sigmoid activation satisfies
\begin{align}
\big|\sigma(x)-\sigma(y)\big|\leq \frac{1}{4}\big|x-y|, \qquad \forall x,y \in \bbR. 
\end{align}
Hence, by Theorem \ref{lem:survive} 
it holds that
$\calR_n(\calF) \leq F_3$,
where
\begin{align}
F_3 &\leq \underbrace{\frac{1}{4} \|\bW\|_{\infty}F_2 +\frac{1}{2\sqrt{n}}}_{\mbox{Dense layer}},\\
F_2 &\leq \underbrace{\bigg(\frac{1}{4} \sup_{l \in [64]}\sum_{u=1}^{3} \sum_{v=1}^{3} \big|W_2^{(l)}(u,v)\big|\bigg)F_1+\frac{1}{2\sqrt{n}} }_{\mbox{The second convolutional layer}},\\
F_1&\leq \underbrace{\bigg(\frac{1}{4} \sup_{l \in [32]}\sum_{u=1}^{3} \sum_{v=1}^{3} \big|W_1^{(l)}(u,v)\big|\bigg)F_0+\frac{1}{2\sqrt{n}} }_{\mbox{The first convolutional layer}},\\
F_0&= \sqrt{\frac{d+1}{n}}. 
\end{align}
Numerical estimation of $F_3$ gives $\calR_n(\calF)\leq 0.00859$. 

By Corollary \ref{PACbound} with probability at least $1-\delta$, it holds that
\begin{align}
&\bbP\big(m_f(\bx,y)\leq 0 \big)\leq \inf_{\gamma\in (0,1]} \bigg[ \frac{1}{n}\sum_{i=1}^n \zeta\big(m_f(\bx_i,y_i) \big)\nn\\
&\quad +   \frac{4M}{\gamma} \calR_n(\calF)+   \sqrt{\frac{\log \log_2 (2\gamma^{-1})}{n}} +\sqrt{\frac{2}{n}\log \frac{3}{\delta}}\bigg]
\end{align}
By setting $\delta=5\%$, $\gamma=0.5$,  the generalisation error can be upper bounded by
\begin{align}
\bbP\big(m_f(\bx,y)\leq 0 \big)\leq 0.189492.  
\end{align}
For this model, the reported test error is $0.0028368$.  
\begin{figure}[htbp]
\begin{center}
\includegraphics[scale=0.7]{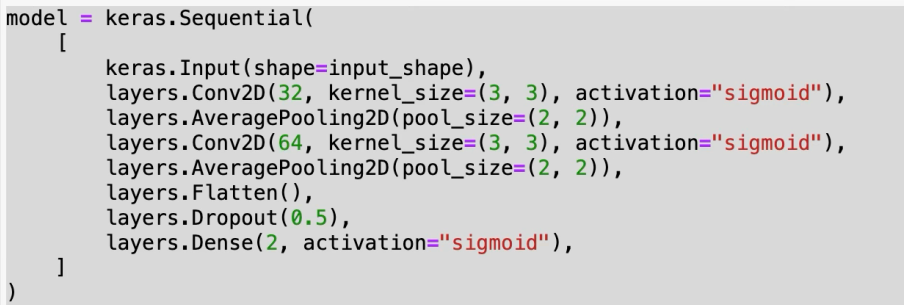}
\caption{CNN model with sigmoid activations}
\label{CNNsigmoid}
\end{center}
\end{figure}\\

\subsection{Experiment 3}
\begin{figure}[htbp]
\begin{center}
\includegraphics[scale=0.7]{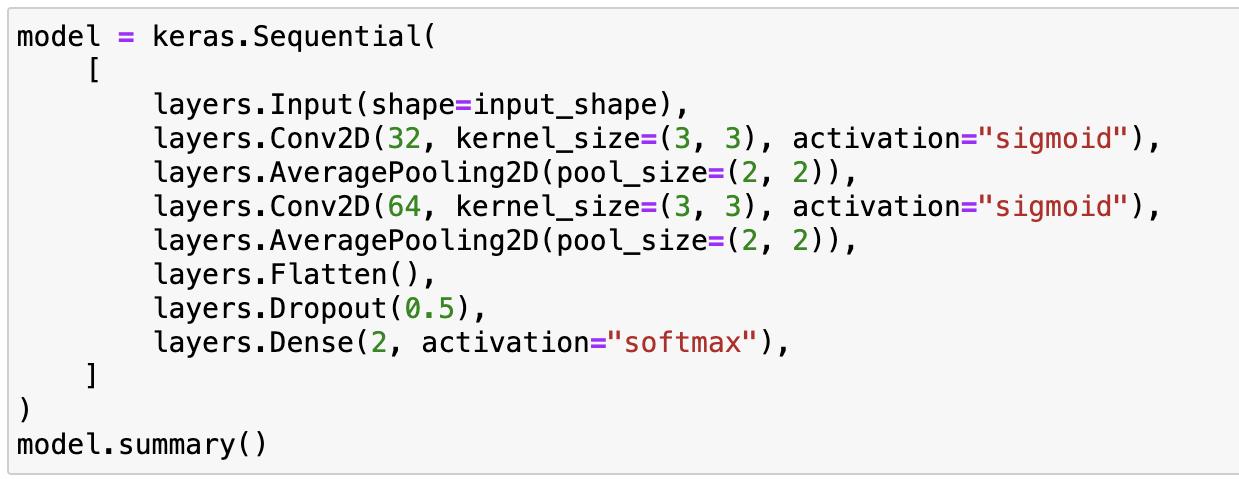}
\caption{CNN model with sigmoid activations}
\label{CNNsigmoid3}
\end{center}
\end{figure}

In this experiment, we use a CNN (cf. Fig.~\ref{CNNsigmoid3}) for classifying MNIST images (class 0 and class 1), i.e., $M=2$, which consists of $n=12665$ training examples.

For this model,  the sigmoid activation $\sigma$ satisfies $\sigma(x)-\sigma(0)= \frac{1}{2}\tanh\big(\frac{x}{2}\big)$ which is odd and has the Lipschitz constant $1/4$. 
In addition, for the dense layer, the sigmoid activation satisfies
\begin{align}
\big|\sigma(x)-\sigma(y)\big|\leq \frac{1}{4}\big|x-y|, \qquad \forall x,y \in \bbR. 
\end{align}
For this example, we assume that we compare the outputs at the layer right before the softmax layer to bound the generalisation error. 
Then, by Theorem \ref{lem:survive} and Lemma \ref{lem:exttalaa} it holds that
$\calR_n(\calF) \leq F_2$,
where
\begin{align}
F_2 &\leq \underbrace{\bigg(\frac{1}{4} \sup_{l \in [64]}\sum_{u=1}^{3} \sum_{v=1}^{3} \big|W_2^{(l)}(u,v)\big|\bigg)F_1+\frac{1}{2\sqrt{n}} }_{\mbox{The second convolutional layer}},\\
F_1&\leq \underbrace{\bigg(\frac{1}{4} \sup_{l \in [32]}\sum_{u=1}^{3} \sum_{v=1}^{3} \big|W_1^{(l)}(u,v)\big|\bigg)F_0+\frac{1}{2\sqrt{n}} }_{\mbox{The first convolutional layer}},\\
F_0&= \sqrt{\frac{d+1}{n}}. 
\end{align}
Numerical estimation of $F_2$ gives $\calR_n(\calF)\leq 0.03074$. 

By Corollary \ref{PACbound} with probability at least $1-\delta$, it holds that
\begin{align}
&\bbP\big(m_f(\bx,y)\leq 0 \big)\leq \inf_{\gamma\in (0,1]} \bigg[ \frac{1}{n}\sum_{i=1}^n \zeta\big(m_f(\bx_i,y_i) \big)\nn\\
&\quad +   \frac{4M}{\gamma} \calR_n(\calF)+   \sqrt{\frac{\log \log_2 (2\gamma^{-1})}{n}} +\sqrt{\frac{2}{n}\log \frac{3}{\delta}}\bigg]
\end{align}
By setting $\delta=5\%$, $\gamma=1$,  the generalisation error can be upper bounded by
\begin{align}
\bbP\big(m_f(\bx,y)\leq 0 \big)\leq 0.2775.  
\end{align}
For this model, the reported test error is $0.001418$. 
 
\section{Comparision with Golowich et al.'s bound \citep{Golowich2018}} \label{comp:sec}
In \citep[Section 4]{Golowich2018}, the authors present an upper bound on Rademacher complexity for DNNs with ReLU activation functions as follows:
\begin{align}
\calR_n(\calF)=O\bigg(\prod_{j=1}^L \|\bW_j\|_F \max\bigg\{1,\log\bigg( \prod_{j=1}^L \frac{\|\bW_j\|_F}{\|\bW_j\|_2}\ \bigg)\bigg\}\min\bigg\{\frac{\max\{1,\log n\}^{3/4}}{n^{1/4}}, \sqrt{\frac{L}{n}} \bigg\}\bigg) \label{amut}
\end{align}
where $\bW_1,\bW_2,\cdots, \bW_L$ are the parameter matrices of the $L$ layers. \\
Now, let $\Gamma $ be the term inside the bracket in \eqref{amut}, and define
\begin{align}
\beta=\min_j \frac{\|\bW_j\|_F}{\|\bW_j\|_2}\geq 1.
\end{align}
Then, from \eqref{amut} we have
\begin{align}
\Gamma \geq \prod_{j=1}^L \|\bW_j\|_F \min\bigg\{\frac{\max\{1,\log n\}^{3/4} \sqrt{\max\{1,L \log \beta\}}}{n^{1/4}}, \sqrt{\frac{L}{n}} \bigg\} \label{eq65}.
\end{align}
For the general case, it holds that $\beta>1$. Hence, from \eqref{eq65} we have
\begin{align}
\calR_n(\calF)=O\bigg( \sqrt{\frac{L}{n}} \prod_{j=1}^L  \|\bW_j\|_F\bigg) \label{eqh1}.
\end{align} 
As analysed in \citep{Golowich2018}, this bound improves many previous bounds, including Neyshabur et al.'s bound \cite{Neyshabur2015},  \cite{Neyshabur2018APA} which are known to be vacuous for certain ReLU DNNs \citep{NagarajanKolter2019}. 

By using Theorem \ref{lem:survive0} with $p=2$ we have
\begin{align}
\calR_n(\calF)=O \bigg(\sqrt{\frac{1}{n}} \prod_{j=1}^L \|\bW_j\|_2\bigg) \label{eqh2}
\end{align}
for ReLU-DNNs, i.e., all the activation functions of the DNN belong to ReLU family.  Note that $\|\bW_l\|_2 \leq \|\bW_l\|_F$ for all $l \in [L]$. Hence, our bound improves upon \citep{Golowich2018} in both the product of the norm and the constant factor preceding the norm, particularly when $L>>$. 

Additionally, by Theorem \ref{lem:survive} and Lemma \ref{goodlem}, we can show that
\begin{align}
\calR_n(\calF)=O\bigg( \sqrt{\frac{1}{n}} \prod_{j=1}^L\mu_j  \|\bW_j\|_{\infty} \bigg) \label{eqh2b}
\end{align}
for DNNs with some special classes of activation functions, including old activation functions, where $\mu_j$ is the Lipschitz constant of the $j$-layer activation function. 
The bound  in \eqref{eqh2b} is applicable to a broad range of activation functions. While the works of \citep{Golowich2018}, \cite{Neyshabur2015}, and \cite{Neyshabur2018APA} primarily focus on ReLU-based DNNs, our approach generalizes to a broader range of activation functions.

\clearpage 



\bibliography{isitbib}  

\begin{thebibliography}{33}
\providecommand{\natexlab}[1]{#1}
\providecommand{\url}[1]{\texttt{#1}}
\expandafter\ifx\csname urlstyle\endcsname\relax
  \providecommand{\doi}[1]{doi: #1}\else
  \providecommand{\doi}{doi: \begingroup \urlstyle{rm}\Url}\fi

\bibitem[A.~Ambroladze \& ShaweTaylor(2007)A.~Ambroladze and
  ShaweTaylor]{Ambroladze2007}
E.~Parrado-Hern'{'}andez A.~Ambroladze and J.~ShaweTaylor.
\newblock Tighter {PAC-Bayes} bounds.
\newblock In \emph{NIPS}, 2007.

\bibitem[Bartlett \& Shawe-Taylor(1999)Bartlett and Shawe-Taylor]{Bartlett1999}
Peter Bartlett and John Shawe-Taylor.
\newblock \emph{Generalization Performance of Support Vector Machines and Other
  Pattern Classifiers}, pp.\  43–54.
\newblock MIT Press, 1999.

\bibitem[Bartlett et~al.(1998{\natexlab{a}})Bartlett, Freund, Lee, and
  Schapire]{Bartlett1998}
Peter Bartlett, Yoav Freund, Wee~Sun Lee, and Robert~E. Schapire.
\newblock {Boosting the margin: a new explanation for the effectiveness of
  voting methods}.
\newblock \emph{The Annals of Statistics}, 26\penalty0 (5):\penalty0 1651 --
  1686, 1998{\natexlab{a}}.

\bibitem[Bartlett \& Williamson(1996)Bartlett and Williamson]{Bartlett1996}
Peter~L. Bartlett and Robert~C. Williamson.
\newblock The vc dimension and pseudodimension of two-layer neural networks
  with discrete inputs.
\newblock \emph{Neural Computation}, 8:\penalty0 625--628, 1996.

\bibitem[Bartlett et~al.(1998{\natexlab{b}})Bartlett, Maiorov, and
  Meir]{Bartlett1998A}
Peter~L. Bartlett, Vitaly Maiorov, and Ron Meir.
\newblock Almost linear vc-dimension bounds for piecewise polynomial networks.
\newblock \emph{Neural Computation}, 10:\penalty0 2159--2173,
  1998{\natexlab{b}}.

\bibitem[Bartlett et~al.(2017)Bartlett, Foster, and Telgarsky]{Bartlett2017}
Peter~L. Bartlett, Dylan~J. Foster, and Matus Telgarsky.
\newblock Spectrally-normalized margin bounds for neural networks.
\newblock In \emph{NIPS}, 2017.

\bibitem[Biggs \& Guedj(2021)Biggs and Guedj]{Biggs2021}
F.~Biggs and B.~Guedj.
\newblock Differentiable {PAC-Bayes} objectives with partially aggregated
  neural networks.
\newblock \emph{Entropy}, 23, 2021.

\bibitem[Clerico et~al.(2021{\natexlab{a}})Clerico, Deligiannidis, and
  Doucet]{Eugenio2021}
Eugenio Clerico, George Deligiannidis, and Arnaud Doucet.
\newblock {Conditional Gaussian PAC-Bayes}.
\newblock \emph{Arxiv: 2110.1188}, 2021{\natexlab{a}}.

\bibitem[Clerico et~al.(2021{\natexlab{b}})Clerico, Deligiannidis, and
  Doucet]{Eugenio2021a}
Eugenio Clerico, George Deligiannidis, and Arnaud Doucet.
\newblock {Wide stochastic networks: Gaussian limit and PACBayesian training}.
\newblock \emph{Arxiv: 2106.09798}, 2021{\natexlab{b}}.

\bibitem[Duchi et~al.(2011)Duchi, Agarwal, Johansson, and
  Jordan]{Duchi2011ErgodicMD}
John~C. Duchi, Alekh Agarwal, Mikael Johansson, and Michael~I. Jordan.
\newblock Ergodic mirror descent.
\newblock \emph{2011 49th Annual Allerton Conference on Communication, Control,
  and Computing (Allerton)}, pp.\  701--706, 2011.

\bibitem[Dziugaite \& Roy.(2017)Dziugaite and Roy.]{Dziugaite2017}
G.~K. Dziugaite and D.~M. Roy.
\newblock { Computing nonvacuous generalization bounds for deep (stochastic)
  neural networks with many more parameters than training data.}
\newblock In \emph{Uncertainty in Artificial Intelligence (UAI)}, 2017.

\bibitem[Esposito et~al.(2021)Esposito, Gastpar, and Issa]{Esposito2021}
Amedeo~Roberto Esposito, Michael Gastpar, and Ibrahim Issa.
\newblock Generalization error bounds via {Rényi}-f-divergences and maximal
  leakage.
\newblock \emph{IEEE Transactions on Information Theory}, 67\penalty0
  (8):\penalty0 4986--5004, 2021.

\bibitem[Goldberg \& Jerrum(1993)Goldberg and Jerrum]{Goldberg1993}
Paul~W. Goldberg and Mark Jerrum.
\newblock Bounding the vapnik-chervonenkis dimension of concept classes
  parameterized by real numbers.
\newblock \emph{Machine Learning}, 18:\penalty0 131--148, 1993.

\bibitem[Golowich et~al.(2018)Golowich, Rakhlin, and Shamir]{Golowich2018}
Noah Golowich, Alexander Rakhlin, and Ohad Shamir.
\newblock Size-independent sample complexity of neural networks.
\newblock In \emph{COLT}, 2018.

\bibitem[Jakubovitz et~al.(2018)Jakubovitz, Giryes, and
  Rodrigues]{Jakubovitz2108}
D.~Jakubovitz, R.~Giryes, and M.~R.~D. Rodrigues.
\newblock {Generalization Error in Deep Learning}.
\newblock \emph{Arxiv: 1808.01174}, 30, 2018.

\bibitem[Koltchinskii \& Panchenko(2002)Koltchinskii and
  Panchenko]{Koltchinskii2002}
V.~Koltchinskii and D.~Panchenko.
\newblock {Empirical Margin Distributions and Bounding the Generalization Error
  of Combined Classifiers}.
\newblock \emph{The Annals of Statistics}, 30\penalty0 (1):\penalty0 1 -- 50,
  2002.

\bibitem[Langford \& Shawe-Taylor(2003)Langford and Shawe-Taylor]{Langford2003}
J.~Langford and J.~Shawe-Taylor.
\newblock {PAC-Bayes and Margins}.
\newblock In \emph{Advances of Neural Information Processing Systems (NIPS)},
  2003.

\bibitem[Ledoux \& Talagrand(1991)Ledoux and Talagrand]{LedouxT1991book}
M.~Ledoux and M.~Talagrand.
\newblock \emph{Probability in Banach Spaces}.
\newblock Springer, New York., 1991.

\bibitem[McAllester(1998)]{McAllester1998}
A.~McAllester.
\newblock Some {PAC-Bayesian} theorems.
\newblock In \emph{{Conference on Learning Theory (COLT)}}, 1998.

\bibitem[McAllester(2004)]{McAllester2004}
D.~A. McAllester.
\newblock {PAC-Bayesian} stochastic model selection.
\newblock \emph{Machine Learning}, 51, 2004.

\bibitem[Nagarajan \& Kolter(2019)Nagarajan and Kolter]{NagarajanKolter2019}
V.~Nagarajan and Z.~Kolter.
\newblock {Uniform convergence may be unable to explain generalization in deep
  learning }.
\newblock In \emph{Advances of Neural Information Processing Systems
  (NeurIPS)}, 2019.

\bibitem[Neyshabur et~al.(2015)Neyshabur, Tomioka, and Srebro]{Neyshabur2015}
Behnam Neyshabur, Ryota Tomioka, and Nathan Srebro.
\newblock Norm-based capacity control in neural networks.
\newblock In \emph{COLT}, 2015.

\bibitem[Neyshabur et~al.(2018)Neyshabur, Bhojanapalli, McAllester, and
  Srebro]{Neyshabur2018APA}
Behnam Neyshabur, Srinadh Bhojanapalli, David~A. McAllester, and Nathan Srebro.
\newblock A {PAC}-bayesian approach to spectrally-normalized margin bounds for
  neural networks.
\newblock \emph{ArXiv}, abs/1707.09564, 2018.

\bibitem[Raginsky \& Sason(2013)Raginsky and Sason]{RagSason}
M.~Raginsky and I.~Sason.
\newblock \emph{Concentration of measure inequalities in information theory,
  communications and coding}, volume~10 of \emph{Foundations and Trends in
  Communications and Information Theory}.
\newblock Now Publishers Inc, 2013.

\bibitem[Royden \& Fitzpatrick(2010)Royden and Fitzpatrick]{Royden}
H.~Royden and P.~Fitzpatrick.
\newblock \emph{Real Analysis}.
\newblock Pearson, 4th edition, 2010.

\bibitem[Shalev-Shwartz \& Ben-David(2014)Shalev-Shwartz and
  Ben-David]{ShalevShwartz2014UnderstandingML}
Shai Shalev-Shwartz and Shai Ben-David.
\newblock \emph{Understanding Machine Learning - From Theory to Algorithms}.
\newblock Cambridge University Press, 2014.

\bibitem[Truong(2022{\natexlab{a}})]{Truong2022BO}
Lan~V. Truong.
\newblock {Generalization Bounds on Multi-Kernel Learning with Mixed Datasets}.
\newblock \emph{ArXiv}, 2205.07313, 2022{\natexlab{a}}.

\bibitem[Truong(2022{\natexlab{b}})]{Truong2022DL}
Lan~V. Truong.
\newblock {Generalization Error Bounds on Deep Learning with Markov Datasets}.
\newblock In \emph{Advances of Neural Information Processing Systems
  (NeurIPS)}, 2022{\natexlab{b}}.

\bibitem[Truong(2022{\natexlab{c}})]{Truong2022OnLM}
Lan~V. Truong.
\newblock On linear model with markov signal priors.
\newblock In \emph{AISTATS}, 2022{\natexlab{c}}.

\bibitem[Vapnik(1998)]{Vap98}
V.~N. Vapnik.
\newblock \emph{Statistical Learning Theory}.
\newblock Wiley, New York, 1998.

\bibitem[W.~Zhou \& Orbanz.(2019)W.~Zhou and Orbanz.]{Zhou2019}
M.~Austern R. P.~Adams W.~Zhou, V.~Veitch and P.~Orbanz.
\newblock {Non-vacuous generalization bounds at the imagenet scale: a
  PAC-Bayesian compression approach.}
\newblock In \emph{The International Conference on Learning Representations
  (ICLR)}, 2019.

\bibitem[Wang et~al.(2019)Wang, Li, and Giannakis]{Wang2019AML}
Gang Wang, Bingcong Li, and Georgios~B. Giannakis.
\newblock A multistep lyapunov approach for finite-time analysis of biased
  stochastic approximation.
\newblock \emph{ArXiv}, abs/1909.04299, 2019.

\bibitem[Xu \& Raginsky(2017)Xu and Raginsky]{XuRaginskyNIPS17}
A.~Xu and M.~Raginsky.
\newblock Information-theoretic analysis of generalization capability of
  learning algorithms.
\newblock In \emph{Advances of Neural Information Processing Systems (NIPS)},
  2017.

\end{thebibliography}
\bibliographystyle{iclr2025_conference}

\newpage 
\appendix
\section{Proof of Theorem \ref{aux_lem}}\label{aux_lem:proof}
Observe that
\begin{align}
\psi (x)&=ReLU(\bx)-\alpha ReLU(-x)\\
&= \frac{x+|x|}{2} - \alpha \frac{-x+|x|}{2} \\
&= \frac{1+\alpha}{2} x + \frac{(1-\alpha)}{2} |x|. 
\end{align}

Then, for any $p\geq 1$ we have
\begin{align}
&\frac{1}{n}\bbE_{\boldsymbol{\eps}}\bigg[\sup_{h \in \calH}\bigg\|\sum_{i=1}^n \eps_i \psi (h(\bx_i)) \bigg\|_p  \bigg]\\
&\qquad \leq \bigg(\frac{1+\alpha}{2}\bigg) \frac{1}{n}\bbE_{\boldsymbol{\eps}}\bigg[\sup_{h \in \calH}\bigg\|\sum_{i=1}^n \eps_i h(\bx_i) \bigg\|_p \bigg]\nn\\
&\qquad \qquad +\bigg(\frac{1-\alpha}{2} \bigg) \frac{1}{n}\bbE_{\boldsymbol{\eps}}\bigg[\sup_{h \in \calH}\bigg\|\sum_{i=1}^n \eps_i \big|h(\bx_i)\big|  \bigg\|_p \bigg] \label{as3}\\
&\qquad \leq \frac{1}{n}\bbE_{\boldsymbol{\eps}}\bigg[\sup_{h \in \calH_+}\bigg\|\sum_{i=1}^n \eps_i h(\bx_i)\bigg\|_p \bigg] \label{as4},
\end{align} where \eqref{as3} follows from Minkowski's inequality \cite{Royden}, and \eqref{as4} follows from the fact that $|h| \in \calH_+$ if $h \in \calH$.

\section{Proof of Theorem \ref{main:thm1}}\label{proof:main:thm1}
The proof of Theorem \ref{main:thm1} is a combination of the following contraction lemmas.

\begin{lemma} \label{lem:exttalaa}
	Let $\calH$ be a set of functions mapping $\calX$ to $\bbR^m$. Define
	\begin{align}
	\calH_+=\calH \cup \big\{-h: h \in \calH\big\}.
	\end{align} 
For any $\mu>0$, let  $\psi: \bbR^m \to \bbR^m$ such that $ \big|\psi_j(\bx)-\psi_j(\bx')\big| \leq \mu |x_j-x'_j|, \enspace \forall (\bx, \bx' ) \in \bbR^m \times \bbR^m \}, \forall j \in [m]$ and $\psi-\psi(\b0)$ is odd. In addition, $\psi_j(\b0)$ does not depend on $j$. Then, it holds that
	\begin{align}
	&\bbE_{\boldsymbol{\eps}}\bigg[\sup_{h \in \calH} \bigg\| \frac{1}{n}\sum_{i=1}^n\eps_i \psi (h(\bx_i)) \bigg\|_{\infty}  \bigg]\nn\\
	&\qquad \leq \mu \bbE_{\boldsymbol{\eps}}\bigg[ \sup_{h \in \calH_+ }\bigg\|\frac{1}{n}\sum_{i=1}^n \eps_i h(\bx_i)\bigg\|_{\infty} \bigg]+\frac{1}{\sqrt{n}} \sup_{j \in [m]}\big|\psi_j(\b0)\big|. 
	\end{align} Here, we define $\psi(\bx):=(\psi(x_1), \psi(x_2),\cdots, \psi(x_m))^T$ for any $\bx=(x_1,x_2,\cdots,x_m)^T \in \bbR^m$.  
\end{lemma}

Then, the following is a direct result of Lemma \ref{lem:exttalaa} by setting $\psi_j(\bx)=\psi(x_j)$ for all $j \in [m], \bx \in \bbR^m$ for some $\mu$-Lipschitz function $\psi: \bbR \to \bbR$. 
\begin{corollary} \label{lem:exttalaaex}
Let $\calH$ be a set of functions mapping $\calX$ to $\bbR^m$. Define
	\begin{align}
	\calH_+=\calH \cup \big\{-h: h \in \calH\big\}.
	\end{align} 
For any $\mu>0$, let  $\psi: \bbR \to \bbR$ such that $ \big|\psi(x)-\psi(x')\big| \leq \mu |x-x'|, \enspace \forall (x, x' ) \in \bbR \times \bbR \}$ and $\psi-\psi(0)$ is odd. Then, it holds that
	\begin{align}
	&\bbE_{\boldsymbol{\eps}}\bigg[\sup_{h \in \calH} \bigg\| \frac{1}{n}\sum_{i=1}^n\eps_i \psi (h(\bx_i)) \bigg\|_{\infty}  \bigg]\nn\\
	&\qquad \leq \mu \bbE_{\boldsymbol{\eps}}\bigg[ \sup_{h \in \calH_+ }\bigg\|\frac{1}{n}\sum_{i=1}^n \eps_i h(\bx_i)\bigg\|_{\infty} \bigg]+\frac{1}{\sqrt{n}} \big|\psi(0)\big|. 
	\end{align} Here, we define $\psi(\bx):=(\psi(x_1), \psi(x_2),\cdots, \psi(x_m))^T$ for any $\bx=(x_1,x_2,\cdots,x_m)^T \in \bbR^m$.  
\end{corollary}

\begin{lemma} \label{lem:exttalaa1}
	Let $\calH$ be a set of functions mapping $\calX$ to $\bbR^m$. Define
	\begin{align}
	\calH_+=\calH \cup \big\{-h: h \in \calH\big\} \cup  \big\{|h|: h \in \calH\big\}.
	\end{align} 
	For any $\mu>0$, let  $\psi: \bbR \to \bbR$ such that $ \big|\psi(x)-\psi(x')\big| \leq \mu |x-x'|, \enspace \forall (x, x' ) \in \bbR \times \bbR \}$ and $\psi-\psi(0)$ is even. Then, it holds that
	\begin{align}
	&\bbE_{\boldsymbol{\eps}}\bigg[\sup_{h \in \calH} \bigg\| \frac{1}{n}\sum_{i=1}^n\eps_i \psi (h(\bx_i)) \bigg\|_{\infty}  \bigg]\nn\\
	&\qquad \leq 2\mu \bbE_{\boldsymbol{\eps}}\bigg[ \sup_{h \in \calH_+ }\bigg\|\frac{1}{n}\sum_{i=1}^n \eps_i h(\bx_i)\bigg\|_{\infty} \bigg]+\frac{1}{\sqrt{n}} \big|\psi(0)\big|. 
	\end{align} Here, we define $\psi(\bx):=(\psi(x_1), \psi(x_2),\cdots, \psi(x_m))^T$ for any $\bx=(x_1,x_2,\cdots,x_m)^T \in \bbR^m$.  
\end{lemma}

\begin{lemma} \label{lem:exttalab1}
	Let $\calH$ be a set of functions mapping $\calX$ to $\bbR^m$. Define
	\begin{align}
	\calH_+=\calH \cup \big\{-h: h \in \calH\big\} \cup  \big\{|h|: h \in \calH\big\}.
	\end{align} 
	For any $\mu>0$, let  $\psi: \bbR \to \bbR$ such that $ \big|\psi(x)-\psi(x')\big| \leq \mu |x-x'|, \enspace \forall (x, x' ) \in \bbR \times \bbR \}$. Then, it holds that
	\begin{align}
	&\bbE_{\boldsymbol{\eps}}\bigg[\sup_{h \in \calH} \bigg\| \frac{1}{n}\sum_{i=1}^n\eps_i \psi (h(\bx_i)) \bigg\|_{\infty}  \bigg]\nn\\
	&\qquad \leq 3\mu \bbE_{\boldsymbol{\eps}}\bigg[ \sup_{h \in \calH_+ }\bigg\|\frac{1}{n}\sum_{i=1}^n \eps_i h(\bx_i)\bigg\|_{\infty} \bigg]+\frac{1}{\sqrt{n}} \big|\psi(0)\big|. 
	\end{align} Here, we define $\psi(\bx):=(\psi(x_1), \psi(x_2),\cdots, \psi(x_m))^T$ for any $\bx=(x_1,x_2,\cdots,x_m)^T \in \bbR^m$.  
\end{lemma}
These lemmas are proved in the next appendices. 

\section{Proof of Lemma \ref{lem:exttalaa}} \label{lem:exttala:proof}
First, we have
\begin{align}
&\bbE_{\boldsymbol{\eps}}\bigg[\sup_{h \in \calH}\bigg\|\frac{1}{n}\sum_{i=1}^n \eps_i \psi(h(\bx_i))\bigg\|_{\infty}  \bigg]\nn\\
&\qquad \leq \frac{1}{n}\bigg(\bbE_{\boldsymbol{\eps}}\bigg[\sup_{h \in \calH}\bigg\|\sum_{i=1}^n \eps_i \bigg(\psi(h(\bx_i) )-\psi(\b0 )\bigg)\bigg\|_{\infty} \bigg] + \bbE_{\boldsymbol{\eps}}\bigg[\sup_{h \in \calH} \bigg\| \sum_{i=1}^n \eps_i \psi(\b0) \bigg\|_{\infty} \bigg]\bigg)\label{M1}\\
&\qquad \leq \frac{1}{n}\bigg(\bbE_{\boldsymbol{\eps}}\bigg[\sup_{h \in \calH}\bigg\|\sum_{i=1}^n \eps_i \bigg(\psi(h(\bx_i) )-\psi(\b0)\bigg)\bigg\|_{\infty}  \bigg]+ \bbE_{\boldsymbol{\eps}}\bigg[\bigg\|\sum_{i=1}^n \eps_i \psi(\b0) \bigg\|_{\infty} \bigg]\bigg)\label{M1b}\\
&\qquad \leq \frac{1}{n}\bigg(\bbE_{\boldsymbol{\eps}}\bigg[\sup_{h \in \calH}\bigg\|\sum_{i=1}^n \eps_i \bigg(\psi(h(\bx_i) )-\psi(\b0)\bigg)\bigg\|_{\infty}  \bigg]+ \sup_{j\in [m]} \sqrt{\bbE_{\boldsymbol{\eps}}\bigg[\bigg( \sum_{i=1}^n \eps_i \psi_j(\b0) \bigg)^2}\bigg]\bigg)\label{M1c}\\
&\qquad\leq \frac{1}{n}\bbE_{\boldsymbol{\eps}}\bigg[\sup_{h \in \calH}\bigg\|\sum_{i=1}^n \eps_i \bigg(\psi (h(\bx_i))-\psi(\b0)\bigg)\bigg\|_{\infty} \bigg]+  \sup_{j \in [m]}\big|\psi_j(\b0)\big|\frac{1}{\sqrt{n}} \label{M2},
\end{align} where \eqref{M1} follows from the triangular property of the $\infty$-norm \cite{Royden},  and \eqref{M1c} follows from Cauchy-Schwarz inequality and the assumption that $\psi_j(\b0)$ does not depend on $j$.   

Define $\tilde{\psi}(\bx):=\psi(\bx)-\psi(\b0)$ for all $\bx \in \bbR^m$. Then, we have $\tilde{\psi}(\b0)=\b0$, and $\tilde{\psi}$ satisfies $|\tilde{\psi}_j(\bx)-\tilde{\psi}_j(\bx')|\leq \mu |x_j-x'_j|$ for all $\bx, \bx' \in \bbR^m, j \in [m]$. In addition, by our assumption, $\tilde{\psi}$ is odd. 

Let
\begin{align}
\Psi=\big\{\tilde{\psi}:\bbR^m \to\bbR^m, \mbox{st.} \enspace \tilde{\psi}(-\bx)=-\tilde{\psi}(\bx), |\tilde{\psi}_j(\bx)-\tilde{\psi}(\by)|\leq \mu |x_j-y_j|\enspace \forall \bx,\by \in \bbR^m, j \in [m] \big\}. 
\end{align}

It follows that
\begin{align}
&\bbE_{\boldsymbol{\eps}}\bigg[\sup_{h \in \calH}\bigg\|\frac{1}{n}\sum_{i=1}^n \eps_i \tilde{\psi} \big(h(\bx_i)\big) \bigg\|_{\infty} \bigg]\\
&\qquad = \frac{1}{n}\bbE_{\boldsymbol{\eps}}\bigg[ \sup_{j \in [m]} \sup_{h \in \calH}\bigg|\sum_{i=1}^n \eps_i\tilde{\psi}_j \big(h(\bx_i)\big)\bigg| \bigg]\\
&\qquad \leq \frac{1}{n}\bbE_{\boldsymbol{\eps}}\bigg[\sup_{s \in \{-1,+1\}^m}\sup_{j \in [m]}  \sup_{h \in \calH} s_j\bigg(\sum_{i=1}^n \eps_i \tilde{\psi}_j \big(h(\bx_i)\big)\bigg) \bigg]\\
&\qquad = \frac{1}{n}\bbE_{\boldsymbol{\eps}}\bigg[\sup_{s \in \{-1,+1\}^m}\sup_{j \in [m]} \sup_{h \in \calH} \sum_{i=1}^n \eps_i s_j  \tilde{\psi}_j\big( h(\bx_i)\big) \bigg] \label{magge}\\
&\qquad = \frac{1}{n}\bbE_{\boldsymbol{\eps}}\bigg[\sup_{s \in \{-1,+1\}^m}\sup_{j \in [m]} \sup_{h \in \calH} \sum_{i=1}^n \eps_i  \tilde{\psi}_j^{(\bs)} \big( h(\bx_i)\big) \bigg] \label{magge1}\\
&\qquad \leq  \frac{1}{n}\bbE_{\boldsymbol{\eps}}\bigg[\sup_{\tilde{\psi}\in \Psi} \sup_{s \in \{-1,+1\}^m}\sup_{j \in [m]} \sup_{h \in \calH} \sum_{i=1}^n \eps_i \tilde{\psi}_j \big(h(\bx_i)\big) \bigg] \label{magge2}\\
&\qquad \leq \frac{1}{n}\bbE_{\boldsymbol{\eps}}\bigg[\sup_{\tilde{\psi}\in \Psi}\sup_{j \in [m]} \sup_{h \in \calH_+} \sum_{i=1}^n \eps_i \tilde{\psi} _j\big(h(\bx_i)\big) \bigg] \label{espa},
\end{align} where \eqref{magge1} follows by defining $\tilde{\psi}^{(\bs)}=(s_1 \tilde{\psi}_1, s_2 \tilde{\psi}_2, \cdots, s_m \tilde{\psi}_m)$ for any $\bs \in \{-1,+1\}^m$, \eqref{magge2} follows from the fact that $\tilde{\psi}^{(\bs)} \in \Psi$ for any fixed $\bs$,  and \eqref{espa} follows from the definition of $\calH_+$. 

Now, we have
\begin{align}
& \bbE_{\boldsymbol{\eps}}\bigg[ \sup_{\tilde{\psi}\in \Psi} \sup_{j \in [m]} \sup_{h \in \calH_+} \sum_{i=1}^n \eps_i \tilde{\psi}_j \big(h(\bx_i)\big) \bigg]\nn\\
&\qquad =\bbE_{\eps_1,\eps_2,\cdots,\eps_{n-1}}\bigg[\bbE_{\eps_n}\bigg[ \sup_{\tilde{\psi}\in \Psi}\sup_{j \in [m]} \sup_{h \in \calH_+}u_{n-1}(h,j)+ \eps_n \tilde{\psi}_j \big(h(\bx_n)\big)  \bigg]\bigg],
\end{align}
where
\begin{align}
u_{n-1}(h,j):=\sum_{i=1}^{n-1} \eps_i \tilde{\psi}_j \big(h(\bx_i)\big) \label{eq116}.
\end{align}

Since $\eps_n$ is uniformly distributed over $\{-1,1\}$, we have
\begin{align}
&\bbE_{\eps_n}\bigg[ \sup_{\tilde{\psi}\in \Psi}\sup_{j \in [m]} \sup_{h \in \calH_+}u_{n-1}(h,j)+ \eps_n \tilde{\psi}_j \big(h(\bx_n)\big)   \bigg]\nn\\
&\qquad =\frac{1}{2}\bigg( \sup_{\tilde{\psi}\in \Psi}\sup_{j \in [m]} \sup_{h \in \calH_+}u_{n-1}(h,j)
+  \tilde{\psi}_j (h(\bx_n)) \bigg)  \nn\\
&\qquad \qquad  + \frac{1}{2}\bigg( \sup_{\tilde{\psi}\in \Psi}\sup_{j \in [m]} \sup_{h \in \calH_+}u_{n-1}(h,j)- \tilde{\psi}_j(h(\bx_n))\bigg) \label{c1}.
\end{align} 
Hence, we have
\begin{align}
&\bbE_{\boldsymbol{\eps}}\bigg[  \sup_{\tilde{\psi}\in \Psi}\sup_{j \in [m]} \sup_{h \in \calH_+} \sum_{i=1}^n \eps_i \tilde{\psi}_j \big(h(\bx_i)\big) \bigg]\nn\\
&\qquad= \frac{1}{2}\bbE_{\eps_1,\eps_2,\cdots,\eps_{n-1}}\bigg[ \sup_{\tilde{\psi}\in \Psi}\sup_{j \in [m]} \sup_{h \in \calH_+}u_{n-1}(h,j)
+  \tilde{\psi}_j (h(\bx_n))\bigg]   \nn\\
&\qquad \qquad  + \frac{1}{2}\bbE_{\eps_1,\eps_2,\cdots,\eps_{n-1}}\bigg[ \sup_{\tilde{\psi}\in \Psi}\sup_{j \in [m]} \sup_{h \in \calH_+}u_{n-1}(h,j)- \tilde{\psi}_j(h(\bx_n))\bigg]\\
&\qquad= \frac{1}{2}\bbE_{\eps_1,\eps_2,\cdots,\eps_{n-1}}\bigg[ \sup_{\tilde{\psi}\in \Psi}\sup_{j \in [m]} \sup_{h \in \calH_+}u_{n-1}(h,j)
+  \tilde{\psi}_j (h(\bx_n))\bigg]   \nn\\
&\qquad \qquad  + \frac{1}{2}\bbE_{\eps_1,\eps_2,\cdots,\eps_{n-1}}\bigg[ \sup_{\tilde{\psi}\in \Psi}\sup_{j \in [m]} \sup_{h \in \calH_+} -u_{n-1}(h,j)- \tilde{\psi}_j(h(\bx_n))\bigg] \label{keykey}\\
&\qquad= \bbE_{\eps_1,\eps_2,\cdots,\eps_{n-1}}\bigg[\frac{1}{2}\bigg( \sup_{\tilde{\psi}\in \Psi}\sup_{j \in [m]} \sup_{h \in \calH_+}u_{n-1}(h,j)
+  \tilde{\psi}_j (h(\bx_n))\bigg)\nn\\
&\qquad \qquad + \frac{1}{2} \bigg( \sup_{\tilde{\psi}\in \Psi}\sup_{j \in [m]} \sup_{h \in \calH_+} -u_{n-1}(h,j)- \tilde{\psi}_j(h(\bx_n))\bigg)\bigg] \label{key1key},
\end{align} where \eqref{keykey} follows from the fact that $(-\eps_1,-\eps_2,\cdots,-\eps_{n-1})$ is a tuple of independent Rademacher random variables which has the same distribution as $(\eps_1,\eps_2,\cdots,\eps_{n-1})$. 

Now, given any $j \in [m]$ and $\tilde{\psi}\in \Psi$ we have
\begin{align}
&\sup_{h \in \calH_+}u_{n-1}(h,j) +  \tilde{\psi}_j (h(\bx_n))\nn\\
&\qquad=\sup_{h \in \calH_+}u_{n-1}(-h,j) +  \tilde{\psi}_j(-h(\bx_n))\label{abfact}\\
&\qquad=\sup_{h \in \calH_+} -u_{n-1}(h,j) - \tilde{\psi}_j (h(\bx_n))\label{abfact2},
\end{align} where \eqref{abfact} follows from the assumption that $h \in \calH_+$ if and only if $-h \in \calH_+$, and \eqref{abfact2} follows from the assumption that $\tilde{\psi}$ is odd for any $\tilde{\psi}\in \Psi$.  
 
Hence, for any arbitrarily small $\delta>0$ there exists $j_0 \in [m], \tilde{\psi}_0 \in \Psi$ and $h_1, h_2 \in \calH$ such that
 \begin{align}
 \sup_{\tilde{\psi}\in \Psi} \sup_{j \in [m]} \sup_{h \in \calH_+}u_{n-1}(h,j)
+  \tilde{\psi}_j (h(\bx_n))\leq u_{n-1}(h_1,j_0)+ \tilde{\psi}_{0,j_0} (h_1(\bx_n)) +\delta \label{K1},
 \end{align}
 and
 \begin{align}
  \sup_{\tilde{\psi}\in \Psi}\sup_{j \in [m]} \sup_{h \in \calH_+} -u_{n-1}(h,j)
-  \tilde{\psi} ([h(\bx_n)]_j)\leq - u_{n-1}(h_2,j_0)- \tilde{\psi}_{0,j_0} (h_2(\bx_n))+\delta \label{K2}. 
 \end{align}
 
 It follows that
 \begin{align}
& \frac{1}{2}\bigg( \sup_{\tilde{\psi}\in \Psi}\sup_{j \in [m]} \sup_{h \in \calH_+}u_{n-1}(h,j)
+  \tilde{\psi}_j (h(\bx_n))\bigg)\nn\\
&\qquad \qquad + \frac{1}{2} \bigg( \sup_{\tilde{\psi}\in \Psi}\sup_{j \in [m]} \sup_{h \in \calH_+} -u_{n-1}(h,j)- \tilde{\psi}_j(h(\bx_n))\bigg)\nn\\
&\qquad \leq \frac{1}{2}\bigg(u_{n-1}(h_1,j_0)+ \tilde{\psi}_{0,j_0} (h_1(\bx_n))\bigg)\nn\\
&\qquad \qquad + \frac{1}{2}\bigg(- u_{n-1}(h_2,j_0)- \tilde{\psi}_{0,j_0} (h_2(\bx_n))\bigg)+\delta\\
&\qquad = \frac{1}{2} \big(u_{n-1}(h_1,j_0)-u_{n-1}(h_2,j_0)\big)\nn\\
&\qquad \qquad + \frac{1}{2} \big(\tilde{\psi}_{0,j_0} (h_1(\bx_n))-\tilde{\psi}_{0,j_0} (h_2(\bx_n))\big)+\delta\\
&\qquad \leq \frac{1}{2} \big(u_{n-1}(h_1,j_0)-u_{n-1}(h_2,j_0)\big)+\frac{\mu}{2} \big|[h_1(\bx_n)]_{j_0}-[h_2(\bx_n)]_{j_0}\big|\\
&\qquad = \frac{1}{2} \big(u_{n-1}(h_1,j_0)-u_{n-1}(h_2,j_0)\big)+\frac{\mu}{2} s_{12,n} \big([h_1(\bx_n)]_{j_0}-[h_2(\bx_n)]_{j_0}\big) \label{abh1}\\
&\qquad= \frac{1}{2} \big(u_{n-1}(h_1,j_0)+ \mu s_{12,n} [h_1(\bx_n)]_{j_0}\big)+ \frac{1}{2}\big(-u_{n-1}(h_2,j_0)-\mu s_{12,n} [h_2(\bx_n)]_{j_0}\big)\\
&\qquad\leq  \sup_{s_{12} \in \{-1,+1\}} \frac{1}{2} \big(u_{n-1}(h_1,j_0)+ \mu s_{12} [h_1(\bx_n)]_{j_0}\big)+ \frac{1}{2}\big(-u_{n-1}(h_2,j_0)-\mu s_{12} [h_2(\bx_n)]_{j_0}\big)\\
&\qquad \leq \sup_{s_{12} \in \{-1,+1\}} \frac{1}{2} \sup_{\tilde{\psi}\in \Psi}\sup_{j \in [m]} \sup_{h \in \calH_+} u_{n-1}(h,j)+\mu s_{12} [h(\bx_n)]_j\nn\\
&\qquad \qquad + \frac{1}{2} \sup_{\tilde{\psi}\in \Psi}\sup_{j \in [m]} \sup_{h \in \calH_+} -u_{n-1}(h,j)-\mu s_{12} [h(\bx_n)]_j\label{kpmod}\\
&\qquad \leq \sup_{s_{12} \in \{-1,+1\}} \frac{1}{2} \sup_{\tilde{\psi}\in \Psi}\sup_{j \in [m]} \sup_{h \in \calH_+} u_{n-1}(h,j)+\mu s_{12} [h(\bx_n)]_j\nn\\
&\qquad \qquad + \frac{1}{2} \sup_{\tilde{\psi}\in \Psi}\sup_{j \in [m]} \sup_{h \in \calH_+} u_{n-1}(h,j)-\mu s_{12} [h(\bx_n)]_j\label{kpmod2},
 \end{align} where
$
 s_{12,n}:= \sgn \big([h_1(\bx_n)]_{j_0}-[h_2(\bx_n)]_{j_0}\big)
 $ in \eqref{abh1}, and \eqref{kpmod2} follows from the fact that $-\tilde{\psi} \in \Psi$ if $\tilde{\psi} \in \Psi$. 
 
 From \eqref{kpmod2} we obtain
 \begin{align}
 & \frac{1}{2}\bigg( \sup_{\tilde{\psi}\in \Psi}\sup_{j \in [m]} \sup_{h \in \calH_+}u_{n-1}(h,j)
+  \tilde{\psi}_j (h(\bx_n))\bigg)\nn\\
&\qquad \qquad + \frac{1}{2} \bigg( \sup_{\tilde{\psi}\in \Psi}\sup_{j \in [m]} \sup_{h \in \calH_+} -u_{n-1}(h,j)- \tilde{\psi}_j(h(\bx_n))\bigg)\\
 &\qquad\leq   \sup_{s_{12} \in \{-1,+1\}} \bbE_{\tilde{\eps}_n} \bigg[\sup_{\tilde{\psi}\in \Psi}\sup_{j \in [m]} \sup_{h \in \calH_+} u_{n-1}(h,j)+\mu \tilde{\eps}_n s_{12} [h(\bx_n)]_j  \bigg] \label{laybay}
 \end{align} for some Rademacher random variable $\tilde{\eps}_n$ which is independent of $(\eps_1,\eps_2,\cdots,\eps_{n-1})$. 
 
 Since $\tilde{\eps}_n s_{12} \sim \tilde{\eps}_n$ for any fixed $s_{12} \in \{-1,+1\}$, from \eqref{laybay} we have
 \begin{align}
 & \frac{1}{2}\bigg( \sup_{\tilde{\psi}\in \Psi}\sup_{j \in [m]} \sup_{h \in \calH_+}u_{n-1}(h,j)
+  \tilde{\psi}_j(h(\bx_n))\bigg)\nn\\
&\qquad \qquad + \frac{1}{2} \bigg( \sup_{\tilde{\psi}\in \Psi}\sup_{j \in [m]} \sup_{h \in \calH_+} -u_{n-1}(h,j)- \tilde{\psi}_j(h(\bx_n))\bigg)\\
&\qquad \leq \bbE_{\tilde{\eps}_n} \bigg[\sup_{\tilde{\psi}\in \Psi}\sup_{j \in [m]} \sup_{h \in \calH_+} u_{n-1}(h,j)+\mu \tilde{\eps}_n [h(\bx_n)]_j  \bigg]  \label{laybay2}. 
 \end{align}
 
 From \eqref{key1key} and \eqref{laybay2} we obtain
 \begin{align}
 &\bbE_{\boldsymbol{\eps}}\bigg[  \sup_{\tilde{\psi}\in \Psi}\sup_{j \in [m]} \sup_{h \in \calH_+} \sum_{i=1}^n \eps_i \tilde{\psi}_j \big(h(\bx_i)\big) \bigg]\nn\\
 &\qquad \leq \bbE_{\eps_1,\eps_2,\cdots, \eps_{n-1}}\bigg[ \bbE_{\tilde{\eps}_n} \bigg[\sup_{\tilde{\psi}\in \Psi}\sup_{j \in [m]} \sup_{h \in \calH_+} u_{n-1}(h,j)+\mu \tilde{\eps}_n [h(\bx_n)]_j  \bigg] \bigg]\\
 &\qquad = \bbE_{\tilde{\eps}_n} \bigg[ \bbE_{\eps_1,\eps_2,\cdots, \eps_{n-1}}\bigg[\sup_{\tilde{\psi}\in \Psi}\sup_{j \in [m]} \sup_{h \in \calH_+} u_{n-1}(h,j)+\mu \tilde{\eps}_n [h(\bx_n)]_j \bigg]\bigg] \label{lame}.
 \end{align}

By continuing this process  (peeling) for $n-1$ more times, we have
\begin{align}
&\bbE_{\boldsymbol{\eps}}\bigg[\sup_{\tilde{\psi}\in \Psi}\sup_{j \in [m]} \sup_{h \in \calH_+}  u_{n-1}(h,j)+ \tilde{\eps}_n\mu [h(\bx_n)]_j\bigg]\nn\\
&\qquad \leq \mu \bbE_{\tilde{\eps}_1,\tilde{\eps}_2,\cdots, \tilde{\eps}_n} \bigg[\sup_{j \in [m]} \sup_{h \in \calH_+}\sum_{i=1}^n \tilde{\eps}_i [h(\bx_n)]_j\bigg]\\
&\qquad =\mu  \bbE_{\boldsymbol{\eps}}\bigg[\sup_{j \in [m]} \sup_{h \in \calH_+}\sum_{i=1}^n \eps_i [h(\bx_n)]_j\bigg] \\
&\qquad \leq \mu  \bbE_{\boldsymbol{\eps}}\bigg[\sup_{j \in [m]} \sup_{h \in \calH_+}\bigg|\sum_{i=1}^n \eps_i [h(\bx_n)]_j\bigg|\bigg] \\
&\qquad= \mu \bbE_{\boldsymbol{\eps}}\bigg[ \sup_{h \in \calH_+}\bigg\|\sum_{i=1}^n \eps_i h(\bx_n)\bigg\|_{\infty}\bigg] \label{pkey1p}. 
\end{align}

From \eqref{M2} and \eqref{pkey1p}, we obtain
\begin{align}
&\bbE_{\boldsymbol{\eps}}\bigg[\sup_{h \in \calH}\bigg\|\frac{1}{n}\sum_{i=1}^n \eps_i \psi(h(\bx_i))\bigg\|_{\infty}  \bigg]\nn\\
&\qquad \leq \mu \bbE_{\boldsymbol{\eps}}\bigg[ \sup_{h \in \calH_+}\bigg\|\frac{1}{n}\sum_{i=1}^n \eps_i h(\bx_n)\bigg\|_{\infty}\bigg]
+ \sup_{j \in [m]}\big|\psi_j(0)\big|\frac{1}{\sqrt{n}} \label{pkey3p}.
\end{align}

This concludes our proof of Lemma \ref{lem:exttalaa}.

\section{Proof of Lemma \ref{lem:exttalaa1}} \label{proof:lem:exttalaa1}
 Since $\psi(x)$ is even, it holds that
\begin{align}
\bbE\bigg[\sup_{h \in \calH} \frac{1}{n}\bigg\|\sum_{i=1}^n \eps_i \psi(h(\bx_i))\bigg\|_{\infty} \bigg]  =\bbE\bigg[\sup_{h \in \calH} \frac{1}{n}\bigg\|\sum_{i=1}^n \eps_i \psi(\big|h(\bx_i)\big| \big)\bigg\|_{\infty} \bigg]  \label{anot1mod},
\end{align} 
Define
\begin{align}
\tilde{\psi}(x):=\psi\big(x \bone\{x> 0 \}\big) - \psi\big(-x  \bone\{x< 0 \} \big)\qquad \forall x \in \bbR \label{A1}. 
\end{align}  Then, it is easy to see that $\tilde{\psi}$ is an odd function. 

On the other hand, we  also have
\begin{align}
\tilde{\psi}(|x|)= \psi(|x|), \qquad \forall x \in \bbR,
\end{align} so
\begin{align}
\bbE\bigg[\sup_{h \in \calH} \frac{1}{n}\bigg\|\sum_{i=1}^n \eps_i \psi(\big|h(\bx_i)\big| \big)\bigg\|_{\infty} \bigg]=\bbE\bigg[\sup_{h \in \calH} \frac{1}{n}\bigg\|\sum_{i=1}^n \eps_i \tilde{\psi}(\big|h(\bx_i)\big| \big)\bigg\|_{\infty} \bigg] \label{amumod}. 
\end{align}

Furthermore, for all $x, y \in \bbR$ we have
\begin{align}
&\big|\tilde{\psi}(x)-\tilde{\psi}(y)\big|\nn\\
&\qquad \leq \big|\psi\big(x \bone\{x> 0 \}\big)-\psi\big(y \bone\{y> 0 \}\big)\big|+ \big|\psi\big(x \bone\{x< 0 \}\big)-\psi\big(y \bone\{y< 0 \}\big)\big|\\
&\qquad \leq \mu \big|x \bone\{x> 0 \}-y \bone\{y > 0 \}\big|+ \mu \big|x \bone\{x < 0 \}-y \bone\{y < 0 \}\big| \label{Y1mod}
\end{align}
Now, observe that
\begin{align}
&\big|x \bone\{x>0 \}-y \bone\{y > 0 \}\big|\nn\\
&\qquad =  \bigg|\frac{x+|x|}{2}-\frac{ y+|y_|}{2}\big|\\
&\qquad \leq \frac{1}{2}  |x-y|+ \frac{1}{2} \sum_{i=1}^L ||x|-|y||\\
&\qquad \leq \big|x-y\big| \label{Y2mod}
\end{align}
Similarly, we also have
\begin{align}
\big|x \bone\{x <0 \}-y \bone\{y < 0 \}\big| \leq  |x-y| \label{Y3mod}.
\end{align}

From \eqref{Y1mod}, \eqref{Y2mod}, and \eqref{Y3mod} we obtain
\begin{align}
\big|\tilde{\psi}(x)-\tilde{\psi}(y)\big| \leq 2 \mu \big|x-y\big|, \qquad \forall x, y \in \bbR.
\end{align}

Hence, by Lemma \ref{lem:exttalaaex} we have
\begin{align}
&\bbE\bigg[\sup_{h \in \calH} \frac{1}{n}\bigg\|\sum_{i=1}^n \eps_i \tilde{\psi}(\big|h(\bx_i)\big| \big)\bigg\|_{\infty} \bigg]\nn\\
&\qquad \leq 2\mu \bbE\bigg[\sup_{h \in \calH_+} \frac{1}{n}\bigg\|\sum_{i=1}^n \eps_i \big|h(\bx_i)\big| \bigg\|_{\infty} \bigg]\\
&\qquad\leq  2\mu \bbE\bigg[\sup_{h \in \calH_+} \frac{1}{n}\bigg\|\sum_{i=1}^n \eps_i h(\bx_i) \bigg\|_{\infty} \bigg]
\label{amuta},
\end{align} where \eqref{amuta} follows by using the fact that $|h|\in \calH$ if $h \in \calH_+$. 

Hence, finally we have
\begin{align}
\bbE\bigg[\sup_{h \in \calH} \frac{1}{n}\bigg\|\sum_{i=1}^n \eps_i \psi(h(\bx_i))\bigg\|_{\infty} \bigg] \leq 2\mu \bbE\bigg[\sup_{h \in \calH_+} \frac{1}{n}\bigg\|\sum_{i=1}^n \eps_i h(\bx_i) \bigg\|_{\infty} \bigg] \label{amau}. 
\end{align}

\section{Proof of Lemma \ref{lem:exttalab1}}
For any general function $\psi$, we can represent as
\begin{align}
\psi(x)= \frac{\psi(x)+ \psi(-x)}{2}+\frac{\psi(x)- \psi(-x)}{2}, \qquad \forall \bx \in \bbR \label{amomod}. 
\end{align}
It is easy to see that $ \frac{\psi(x)+ \psi(-x)}{2}$ is an even function with $\mu$-Lipschitz. Besides, $\frac{\psi(x)- \psi(-x)}{2}$ is an odd function with $\mu$-Lischitz. Hence, by using triangle inequality, Lemma \ref{lem:exttalaa} and Lemma  \ref{lem:exttalaa1}, we have
\begin{align}
\bbE\bigg[\sup_{h \in \calH} \frac{1}{n}\bigg\|\sum_{i=1}^n \eps_i \psi(h(\bx_i))\bigg\|_{\infty} \bigg]\leq (2\mu+ \mu)  \bbE\bigg[ \sup_{h \in \calH_+} \frac{1}{n} \bigg\| \sum_{i=1}^n \eps_i h(\bx_i)\bigg\|_{\infty}\bigg]. 
\end{align} 

\section{Proof of Theorem \ref{lem:linear}} \label{proof:lem:linear}

For any $\bW \in \calV$, observe that
\begin{align}
\bigg\|\frac{1}{n} \sum_{i=1}^n \eps_i \bW f(\bx_i) \bigg\|_{\infty} 
& = \bigg\|\bW\bigg(\frac{1}{n} \sum_{i=1}^n \eps_i f(\bx_i)\bigg)\bigg\|_{\infty} \label{B0c}\\
&\leq \big\|\bW\big\|_{\infty}  \bigg\|\frac{1}{n} \sum_{i=1}^n \eps_i f(\bx_i)\bigg\|_{\infty} \\
& \leq \nu \bigg\|\frac{1}{n} \sum_{i=1}^n \eps_i f(\bx_i)\bigg\|_{\infty} \label{alo1}. 
\end{align}  
	
Hence, \eqref{alo2} is a direct application of this fact.

This concludes our proof of Theorem \ref{lem:linear}.

\section{Proof of Lemma \ref{goodlem0inftyto1}} \label{goodlem0inftyto1:proof}
Let
\begin{align}
\bone_{\tau_l^2} &=\underbrace{\begin{bmatrix} 1&1& \cdots &1\end{bmatrix}}_{\tau_l^2 },\\
0_{\tau_l^2}&=\underbrace{\begin{bmatrix} 0&0& \cdots &0\end{bmatrix}}_{\tau_l^2 },
\end{align}
and
\begin{align}
\bA=\frac{1}{\tau_l^2} \begin{bmatrix}\bone_{\tau_l^2}& 0_{\tau_l^2}&0_{\tau_l^2}& \cdots& 0_{\tau_l^2}& 0_{\tau_l^2} \\ 0_{\tau_l^2}& \bone_{\tau_l^2}&0_{\tau_l^2}& \cdots&0_{\tau_l^2}& 0_{\tau_l^2}\\ \vdots & \vdots &\vdots & \ddots& \vdots &\vdots\\ 0_{\tau_l^2}& 0_{\tau_l^2}&0_{\tau_l^2}& \cdots& 0_{\tau_l^2}& \bone_{\tau_l^2}     \end{bmatrix} \in \bbR^{\lceil(d-r_l+1)^2/\tau_l^2\rceil \tau_l^2 \times \lceil(d-r_l+1)^2/\tau_l^2\rceil \tau_l^2 }.
\end{align}
Then, for all $\bx \in \bbR^{d\times d \times C}$ and $l \in [Q], c \in [C]$, we have 
\begin{align}
\psi_{l,c}(\bx)=\sigma_{\rm{avg}}\circ \sigma_{l,c}(\bx),
\end{align}
where
\begin{align}
\sigma_{\rm{avg}}(\bx)=\bA \bx, \qquad \forall \bx \in \bbR^{\lceil(d-r_l+1)^2/\tau_l^2\rceil \tau_l^2 }. 
\end{align}
Now, for all  $\bx, \by \in \bbR^{ \lceil (d-r_l+1)^2/\tau_l^2\rceil \tau_l^2}$  we have
\begin{align}
&\big\|\sigma_{\rm{avg}}(\bx)- \sigma_{\rm{avg}}(\by)\big\|_{\infty} \nn\\
&\qquad \leq  \frac{1}{\tau_l^2} \max_{j\in [\lceil (d-r_l+1)^2/\tau_l^2\rceil]}  \sum_{k=(j-1)\tau_l^2+1 }^{j\tau_l^2} \big|x_k-y_k\big|\\
&\qquad \leq \big\|\bx-\by\big\|_{\infty}. 
\end{align}
Hence, we have
\begin{align}
\|\bA\|_{\infty} \leq 1 \label{mott1}. 
\end{align}
Hence, by Lemma \ref{lem:linear} we have
\begin{align}
&\bbE\bigg[ \sup_{c \in [C]}\sup_{l \in [Q]}\sup_{\psi_{l,c} \in \Psi} \sup_{f \in \calF} \bigg\|\frac{1}{n}\sum_{i=1}^n \eps_i \psi_{l,c} \circ f(\bx_i)\bigg\|_{\infty}   \bigg]\nn\\
&\qquad =\bbE\bigg[ \sup_{c \in [C]}\sup_{l \in [Q]}\sup_{\sigma_{\rm{avg}}} \sup_{\sigma_{l,c}} \sup_{f \in \calF} \bigg\|\frac{1}{n}\sum_{i=1}^n \eps_i \sigma_{\rm{avg}} \circ \sigma_{l,c}\circ  f(\bx_i)\bigg\|_{\infty} \bigg]\\
&\qquad \leq \bbE\bigg[  \sup_{c \in [C]} \sup_{l \in [Q]} \sup_{\sigma_{l,c}} \sup_{f \in \calF} \bigg\|\frac{1}{n}\sum_{i=1}^n \eps_i \sigma_{l,c}\circ f(\bx_i)\bigg\|_{\infty} \bigg] \label{lem1}. 
\end{align}
In addition, for all $\bx \in \bbR^{d\times d \times C}$, 
\begin{align}
\sigma_{l,c}(\bx)&=\{\hat{x}_c(a,b)\}_{a,b=1}^{d-r_l+1},\\
\hat{x}_c(a,b)&=\sigma\bigg(\sum_{u=0}^{r_l-1}\sum_{v=0}^{r_l-1} x(a+u,b+v,c) W_{l,c}(u+1,v+1)\bigg). 
\end{align}
Hence, we have
\begin{align}
&\big\|\sigma_{l,c}(\bx)-\sigma_{l,c}(\by)\big\|_{\infty} \nn\\
&\qquad \leq \mu \max_{a \in [d-r_l+1]}\max_{b\in [d-r_l+1]} \sum_{u=0}^{r_l-1}\sum_{v=0}^{r_l-1}  \big|W_{l,c}(u+1,v+1)x(a+u,b+v,c)\nn\\
&\qquad \qquad \qquad -W_{l,c}(u+1,v+1)y(a+u,b+v,c)\big|\\
&\qquad \leq \mu  \sum_{u=0}^{r_l-1}\sum_{v=0}^{r_l-1}  \big|W_{l,c}(u+1,v+1)  \big| \|\bx-\by\|_{\infty}  \label{A2B}.
\end{align}
Since the convolution is linear, it is also easy to see that $\sigma_{l,c}$ is the composition of a linear map and a point-wise activation map.  
Hence, by Lemma \ref{lem:linear}  and Theorem \ref{main:thm1} we have
\begin{align}
& \bbE\bigg[ \sup_{c \in [C]} \sup_{l \in [Q]} \sup_{\sigma_{l,c}} \sup_{f \in \calF} \bigg\|\frac{1}{n}\sum_{i=1}^n \eps_i \sigma_{l,c} \circ f(\bx_i)\bigg\|_{\infty} \bigg] \nn\\
&\quad \leq \bigg[\gamma(\mu) \sup_{c \in [C]} \sup_{l \in [Q]} \bigg(\sum_{u=0}^{r_l-1}\sum_{v=0}^{r_l-1}  \big|W_{l,c} (u+1,v+1)  \big|\bigg) \bigg] \bbE\bigg[ \sup_{f \in \calF_+} \bigg\|\frac{1}{n}\sum_{i=1}^n \eps_i f(\bx_i)\bigg\|_{\infty} \bigg] + \frac{|\sigma(0)|}{\sqrt{n}}\label{lem2}.
\end{align}
Finally, from \eqref{lem1} and \eqref{lem2} we obtain
\begin{align}
&\bbE\bigg[ \sup_{c \in [C]} \sup_{l \in [Q]}\sup_{\psi_{l,c} \in \Psi} \sup_{f \in \calF} \bigg\|\frac{1}{n}\sum_{i=1}^n \eps_i \psi_{l,c} \circ f(\bx_i)\bigg\|_{\infty}   \bigg]\nn\\
&\quad \leq \bigg[\gamma(\mu) \sup_{c \in [C]} \sup_{l \in [Q]} \bigg(\sum_{u=0}^{r_l-1}\sum_{v=0}^{r_l-1}  \big|W_{l,c}(u+1,v+1)  \big|\bigg) \bigg]  \bbE\bigg[ \sup_{f \in \calF_+} \bigg\|\frac{1}{n}\sum_{i=1}^n \eps_i f(\bx_i)\bigg\|_{\infty} \bigg]+ \frac{|\sigma(0)|}{\sqrt{n}}. 
\end{align}

\section{Proof of Lemma \ref{dropoutlem}}\label{proof:dropoutlem}
This is a direct result of Lemma \ref{lem:exttalaa}, where $\tilde{\psi}_j(\bx)=x_j$ or $0$ at each fixed $j$. Hence, we have
\begin{align}
\big|\tilde{\psi}_j(\bx)-\tilde{\psi}_j(\by)\big|\leq |x_j-y_j| 
\end{align}  for all vectors $\bx$ and $\by$.

\section{Proof of Lemma \ref{goodlem}} \label{proof:goodlem}
This is a direct result of Theorem \ref{main:thm1} and Lemma \ref{lem:linear}.
\section{Proof of Lemma \ref{lem:abound}}\label{proof:lem:abound}
For $M>2$,  \eqref{kbound} is a result of \citep[Proof of Theorem 11]{Koltchinskii2002}. Now, we prove  \eqref{kbound} for $M=2$. Observe that
\begin{align}
&\bbE_{\boldsymbol{\eps}}\bigg[\sup_{f \in \calF} \bigg|\frac{1}{n}\sum_{i=1}^n \eps_i m_f(\bx_i,y_i)\bigg|\bigg]\nn\\
&\qquad =\bbE_{\boldsymbol{\eps}}\bigg[\sup_{f \in \calF} \bigg|\frac{1}{n}\sum_{i=1}^n \eps_i \bigg( [f(\bx_i)]_{y_i}-\sup_{y'\neq y_i} [f(\bx_i)]_{y'} \bigg)\bigg|\bigg]\\
&\qquad \leq \bbE_{\boldsymbol{\eps}}\bigg[\sup_{f \in \calF} \bigg|\frac{1}{n}\sum_{i=1}^n \eps_i [f(\bx_i)]_{y_i}\bigg|\bigg]+ \bbE_{\boldsymbol{\eps}}\bigg[\sup_{f \in \calF} \bigg|\frac{1}{n}\sum_{i=1}^n \eps_i \sup_{y'\neq y_i} [f(\bx_i)]_{y'} \bigg|\bigg] \label{C1}. 
\end{align}
Now, we have
\begin{align}
&\bbE_{\boldsymbol{\eps}}\bigg[\sup_{f \in \calF} \bigg|\frac{1}{n}\sum_{i=1}^n \eps_i [f(\bx_i)]_{y_i}\bigg|\bigg]\nn\\
&\qquad = \bbE_{\boldsymbol{\eps}}\bigg[\sup_{f \in \calF} \bigg|\frac{1}{n}\sum_{i=1}^n \eps_i [f(\bx_i)]_{y_i}\sum_{y=1}^M \bone_{\{y_i=y\}} \bigg|\bigg]\\
&\qquad = \bbE_{\boldsymbol{\eps}}\bigg[\sup_{f \in \calF} \bigg|\frac{1}{n}\sum_{y=1}^M\sum_{i=1}^n \eps_i [f(\bx_i)]_{y} \bone_{\{y_i=y\}} \bigg|\bigg]\\
&\qquad \leq \sum_{y=1}^M \bbE_{\boldsymbol{\eps}}\bigg[\sup_{f \in \calF} \bigg|\frac{1}{n}\sum_{i=1}^n \eps_i [f(\bx_i)]_{y} \bone_{\{y_i=y\}} \bigg|\bigg]\\
&\qquad \leq \frac{1}{2} \sum_{y=1}^M \bbE_{\boldsymbol{\eps}}\bigg[\sup_{f \in \calF} \bigg|\frac{1}{n}\sum_{i=1}^n \eps_i [f(\bx_i)]_{y} (2 \bone_{\{y_i=y\}} -1)\bigg|\bigg]\nn\\
&\qquad \qquad +  \frac{1}{2} \sum_{y=1}^M \bbE_{\boldsymbol{\eps}}\bigg[\sup_{f \in \calF} \bigg|\frac{1}{n}\sum_{i=1}^n \eps_i [f(\bx_i)]_{y} \bigg|\bigg]\\
&\qquad = \frac{1}{2} \sum_{y=1}^M \bbE_{\boldsymbol{\eps}}\bigg[\sup_{f \in \calF} \bigg|\frac{1}{n}\sum_{i=1}^n \eps_i [f(\bx_i)]_{y} \bigg|\bigg]\nn\\
&\qquad \qquad +  \frac{1}{2} \sum_{y=1}^M \bbE_{\boldsymbol{\eps}}\bigg[\sup_{f \in \calF} \bigg|\frac{1}{n}\sum_{i=1}^n \eps_i [f(\bx_i)]_{y} \bigg|\bigg] \label{kpar1}\\
&\qquad=  \sum_{y=1}^M \bbE_{\boldsymbol{\eps}}\bigg[\sup_{f \in \calF} \bigg|\frac{1}{n}\sum_{i=1}^n \eps_i [f(\bx_i)]_{y} \bigg|\bigg] \label{kpar2a}\\
&\qquad\leq  \sum_{y=1}^M \bbE_{\boldsymbol{\eps}}\bigg[\sup_{f \in \calF} \bigg\|\frac{1}{n}\sum_{i=1}^n \eps_i f(\bx_i) \bigg\|_{\infty} \bigg]\\
&\qquad= M \bbE_{\boldsymbol{\eps}}\bigg[\sup_{f \in \calF} \bigg\|\frac{1}{n}\sum_{i=1}^n \eps_i f(\bx_i) \bigg\|_{\infty} \bigg]
 \label{kpar2},
\end{align} where \eqref{kpar1} follows from the fact that $(2 \bone_{\{y_1=y\}}-1) \eps_1, (2 \bone_{\{y_2=y\}}-1) \eps_2, \cdots, (2 \bone_{\{y_n=y\}}-1) \eps_n)$ has the same distribution as $(\eps_1,\eps_2,\cdots,\eps_n)$.

On the other hand, we also have
\begin{align}
&\bbE_{\boldsymbol{\eps}}\bigg[\sup_{f \in \calF} \bigg|\frac{1}{n}\sum_{i=1}^n \eps_i \sup_{y'\neq y_i} [f(\bx_i)]_{y'} \bigg|\bigg] \nn\\
&\qquad = \bbE_{\boldsymbol{\eps}}\bigg[\sup_{f \in \calF} \bigg|\frac{1}{n}\sum_{i=1}^n \eps_i \sup_{y'\neq y_i} [f(\bx_i)]_{y'}\sum_{y=1}^M \bone_{\{y_i=y\}} \bigg|\bigg]\\
&\qquad = \bbE_{\boldsymbol{\eps}}\bigg[\sup_{f \in \calF} \bigg|\frac{1}{n}\sum_{y=1}^M\sum_{i=1}^n \eps_i \sup_{y'\neq y} [f(\bx_i)]_{y'} \bone_{\{y_i=y\}} \bigg|\bigg]\\
&\qquad \leq \sum_{y=1}^M  \bbE_{\boldsymbol{\eps}}\bigg[\sup_{f \in \calF} \bigg|\frac{1}{n}\sum_{i=1}^n \eps_i \sup_{y'\neq y} [f(\bx_i)]_{y'} \bone_{\{y_i=y\}} \bigg|\bigg]\\
&\qquad \leq \frac{1}{2} \sum_{y=1}^M  \bbE_{\boldsymbol{\eps}}\bigg[\sup_{f \in \calF} \bigg|\frac{1}{n}\sum_{i=1}^n \eps_i \sup_{y'\neq y} [f(\bx_i)]_{y'} (2\bone_{\{y_i=y\}}-1) \bigg|\bigg]\nn\\
&\qquad \qquad + \frac{1}{2} \sum_{y=1}^M  \bbE_{\boldsymbol{\eps}}\bigg[\sup_{f \in \calF} \bigg|\frac{1}{n}\sum_{i=1}^n \eps_i \sup_{y'\neq y} [f(\bx_i)]_{y'}  \bigg|\bigg]\\
&\qquad= \frac{1}{2} \sum_{y=1}^M  \bbE_{\boldsymbol{\eps}}\bigg[\sup_{f \in \calF} \bigg|\frac{1}{n}\sum_{i=1}^n \eps_i \sup_{y'\neq y} [f(\bx_i)]_{y'}  \bigg|\bigg]\nn\\
&\qquad \qquad + \frac{1}{2} \sum_{y=1}^M  \bbE_{\boldsymbol{\eps}}\bigg[\sup_{f \in \calF} \bigg|\frac{1}{n}\sum_{i=1}^n \eps_i \sup_{y'\neq y} [f(\bx_i)]_{y'}  \bigg|\bigg] \label{kpar3}\\
&\qquad= \sum_{y=1}^M  \bbE_{\boldsymbol{\eps}}\bigg[\sup_{f \in \calF} \bigg|\frac{1}{n}\sum_{i=1}^n \eps_i \sup_{y'\neq y} [f(\bx_i)]_{y'}  \bigg|\bigg] \label{kpar4},
\end{align} where \eqref{kpar3} follows from the fact that $(2 \bone_{\{y_1=y\}}-1) \eps_1, (2 \bone_{\{y_2=y\}}-1) \eps_2, \cdots, (2 \bone_{\{y_n=y\}}-1) \eps_n)$ has the same distribution as $(\eps_1,\eps_2,\cdots,\eps_n)$.

Now, for each fixed $y \in [M]$ and $M=2$, let $\hat{y}=[M]\setminus \{y\}$ we have
\begin{align}
&\bbE_{\boldsymbol{\eps}}\bigg[\sup_{f \in \calF} \bigg|\frac{1}{n}\sum_{i=1}^n \eps_i \sup_{y'\neq y} [f(\bx_i)]_{y'}  \bigg|\bigg]\nn\\
&\qquad= \bbE_{\boldsymbol{\eps}}\bigg[\sup_{f \in \calF} \bigg|\frac{1}{n}\sum_{i=1}^n \eps_i [f(\bx_i)]_{\hat{y}} \bigg|\bigg]\\
&\qquad \leq \bbE_{\boldsymbol{\eps}}\bigg[\sup_{f \in \calF} \bigg\|\frac{1}{n}\sum_{i=1}^n \eps_i f(\bx_i)  \bigg\|_{\infty} \bigg]\label{par6}.
\end{align} 

It follows from  \eqref{kpar4} and \eqref{par6} that
\begin{align}
&\bbE_{\boldsymbol{\eps}}\bigg[\sup_{f \in \calF} \bigg|\frac{1}{n}\sum_{i=1}^n \eps_i \sup_{y'\neq y_i} [f(\bx_i)]_{y'} \bigg|\bigg] \nn\\
&\qquad \leq M  \bbE_{\boldsymbol{\eps}}\bigg[\sup_{f \in \calF} \bigg\|\frac{1}{n}\sum_{i=1}^n \eps_i f(\bx_i)  \bigg\|_{\infty} \bigg]\label{par7}. 
\end{align}

From\eqref{C1}, \eqref{kpar2}, and \eqref{par7}, for $M=2$ we have
\begin{align}
\bbE_{\boldsymbol{\eps}}\bigg[\sup_{f \in \calF} \bigg|\frac{1}{n}\sum_{i=1}^n \eps_i m_f(\bx_i,y_i)\bigg|\bigg] \leq 2M \bbE_{\boldsymbol{\eps}}\bigg[\sup_{f \in \calF} \bigg|\frac{1}{n}\sum_{i=1}^n \eps_i f(\bx_i)\bigg|\bigg]. 
\end{align}

 \section{Proof of Theorem \ref{main}}\label{main:proof}	

Let $(\bx_1',y_1'),(\bx_2',y_2'),\cdots, (\bx_n',y_n')$ is an i.i.d. sequence with distribution $P_{XY}$ which is independent of $X^nY^n$. Define
	\begin{align}
	E(f):= \bbE_{\bX'\bY'}\bigg[\frac{1}{n}\sum_{i=1}^n \zeta(m_f(\bx_i',y_i'))\bigg] \label{defEf}.
	\end{align}	
	Now, let $D=\{(\bx_i,y_i): i \in [n]\}$, and let $\tilD=\{(\bx_i,y_i): i \in [n]\}$ be a set with only one sample different from $D$, i.e. the $k$-th sample is replaced by 
	$(\tilde{\bx}_k,\tilde{y}_k)$. Define
	\begin{align}
	\hatE_D(f):= \frac{1}{n}\sum_{i=1}^n \zeta(m_f(\bx_i,y_i)) \label{defEDf}
	\end{align}
	and
	\begin{align}
	\Phi(D):=\sup_{f \in \calF} E(f)-\hatE_D(f),
	\end{align} which is a function of $n$ independent random vectors $(\bx_1,y_1),(\bx_2,y_2),\cdots,(\bx_n,y_n)$ where $(\bx_i,y_i) \sim P_{XY}$ for all $i \in [n]$.  Since $0\leq \zeta(x)\leq 1$ for all $x \in \bbR$, from \eqref{defEf} and \eqref{defEDf} we have
    \begin{align}
    \big|\Phi(\tilD)-\Phi(D)\big|&\leq \sup_{f \in\ \calF} \frac{|\zeta(m_f(\bx_k,y_k))- \zeta(m_f(\tilde{\bx}_k,\tilde{y}_k))|}{n}\\
    &\leq \frac{1}{n} \label{mato1}.
    \end{align}
By McDiarmid's inequality \cite{RagSason}, with probability at least $1-\exp(-2t^2)$ we have
	\begin{align}
	&\sup_{f \in \calF}\bigg(\frac{1}{n}\bbE_{\bX'\bY'}\bigg[\sum_{i=1}^n \zeta(m_f(\bx_i',y_i'))\bigg]-\frac{1}{n}\sum_{i=1}^n \zeta(m_f(\bx_i,y_i))\bigg)\nn\\
	&\qquad \leq \bbE_{\bX\bY}\bigg[ \sup_{f \in \calF}\bigg(\bbE_{\bX'\bY'}\bigg[\frac{1}{n}\sum_{i=1}^n \zeta(m_f(\bx_i',y_i'))\bigg]-\frac{1}{n}\sum_{i=1}^n \zeta(m_f(\bx_i,y_i))\bigg) \bigg]+ \frac{t}{\sqrt{n}}  \label{ab3}.
	\end{align}
	
	Now, let $\bar{\zeta}(x):=\zeta(x)-\zeta(0)$, which is a $1/\gamma$-Lipschitz function with $\bar{\zeta}(0)=0$. 
	Then, we have
	\begin{align}
	& \bbE_{\bX \bY}\bigg[\sup_{f \in \calF}\bigg(\bbE_{\bX' \bY'}\bigg[\frac{1}{n}\sum_{i=1}^n \zeta(m_f(\bx_i',y_i'))\bigg]-\frac{1}{n}\sum_{i=1}^n \zeta(m_f(\bx_i,y_i))\bigg)\bigg]\\
	&\qquad\leq \bbE_{\bX \bY} \bigg[ \sup_{f \in \calF}\bigg|\bbE_{\bX'\bY'}\bigg[\frac{1}{n}\sum_{i=1}^n \bar{\zeta}(m_f(\bx_i',y_i'))\bigg]-\frac{1}{n}\sum_{i=1}^n \bar{\zeta}(m_f(\bx_i,y_i))\bigg| \bigg]\\
	&\qquad=\bbE_{\bX \bY} \bigg[ \sup_{f \in \calF}\bigg|\bbE_{\bX'\bY'}\bigg[\frac{1}{n}\sum_{i=1}^n \big(\bar{\zeta}(m_f(\bx_i',y_i'))- \bar{\zeta}(m_f(\bx_i,y_i))\big)\bigg]\bigg|\bigg]\\
	&\qquad\leq\bbE_{\bX\bY}\bigg[\bbE_{\bX'\bY'}\bigg[ \sup_{f \in \calF}\bigg|\frac{1}{n}\sum_{i=1}^n \big(\bar{\zeta}(m_f(\bx_i',y_i'))- \bar{\zeta}(m_f(\bx_i,y_i))\big)\bigg|\bigg] \bigg] \label{mo4} \\
	&\qquad\leq \frac{1}{\gamma} \bbE_{\bX\bY}\bigg[\bbE_{\bX'\bY'}\bigg[ \sup_{f \in \calF}\bigg|\frac{1}{n}\sum_{i=1}^n \big(m_f(\bx_i',y_i')- m_f(\bx_i,y_i)\big)\bigg|\bigg] \bigg] \label{mo4b} \\
	&\qquad=\frac{1}{\gamma} \bbE_{\boldsymbol{\eps}}\bigg[\bbE_{\bX\bY}\bigg[\bbE_{\bX'\bY'}\bigg[ \sup_{f \in \calF}\bigg|\frac{1}{n}\sum_{i=1}^n \eps_i\big(m_f(\bx_i',y_i')- m_f(\bx_i,y_i)\big)\bigg|\bigg] \bigg]\label{mode}\\
	&\qquad\leq  \frac{1}{\gamma} \bbE_{\boldsymbol{\eps}}\bigg[\bbE_{\bX'\bY'}\bigg[ \sup_{f \in \calF}\bigg|\frac{1}{n}\sum_{i=1}^n \eps_i m_f(\bx_i',y_i')\bigg|\bigg]\bigg]\nn\\
	&\qquad \qquad +\frac{1}{\gamma} \bbE_{\boldsymbol{\eps}}\bigg[\bbE_{\bX \bY}\bigg[ \sup_{f \in \calF}\bigg|\frac{1}{n}\sum_{i=1}^n \eps_i m_f(\bx_i,y_i)\bigg|\bigg]\bigg] \label{mot2}\\
	&\qquad= \frac{2}{\gamma} \bbE_{\boldsymbol{\eps}}\bigg[\bbE_{\bX\bY}\bigg[ \sup_{f \in \calF}\bigg|\frac{1}{n}\sum_{i=1}^n \eps_i m_f(\bx_i,y_i)\bigg|\bigg]\bigg] \label{mot3}\\
	&\qquad =\frac{2}{\gamma} \bbE_{\bX\bY}\bigg[ \bbE_{\boldsymbol{\eps}}\bigg[\sup_{f \in \calF}\bigg|\frac{1}{n}\sum_{i=1}^n \eps_i m_f(\bx_i,y_i)\bigg|\bigg]\bigg] \label{mot4}\\
	&\qquad \leq \frac{2\beta(M)}{\gamma} \bbE_{\bX\bY}\bigg[ \bbE_{\boldsymbol{\eps}}\bigg[\sup_{f \in \calF}\bigg\|\frac{1}{n}\sum_{i=1}^n \eps_i f(\bx_i)\bigg\|_{\infty} \bigg]\bigg] \label{mot4b}
	\end{align} where \eqref{mode} follows from \cite[Lemma 25]{Truong2022DL}, 
	 and \eqref{mot4b} follows from Lemma \ref{lem:abound}. 
	
	From \eqref{mot4b}, with probability at least $1- \exp(-2t^2)$ we have
	\begin{align}
	& \sup_{f \in \calF}\bigg(\bbE\bigg[\frac{1}{n}\sum_{i=1}^n \zeta(m_f(\bx_i',y_i'))\bigg]-\frac{1}{n}\sum_{i=1}^n \zeta(m_f(\bx_i,y_i))\bigg)\\
	&\qquad \qquad \leq  \frac{2\beta(M)}{\gamma}  \bbE\bigg[\sup_{f \in \calF}\bigg\|\frac{1}{n}\sum_{i=1}^n \eps_i f(\bx_i)\bigg\|_{\infty} \bigg]+ \frac{t}{\sqrt{n}} \label{g10}.
	\end{align}
	
	It follows that, with probability at least $1-\exp(-2t^2)$,
	\begin{align}
	&\bbE_{\bX',\bY'}\bigg[\frac{1}{n}\sum_{i=1}^n \zeta(m_f(\bx_i',y_i'))\bigg]\leq \frac{1}{n}\sum_{i=1}^n \zeta(m_f(\bx_i,y_i))\nn\\
	&\qquad +   \frac{2\beta(M)}{\gamma} \bbE\bigg[\sup_{f \in \calF}\bigg\|\frac{1}{n}\sum_{i=1}^n \eps_i f (\bx_i)\bigg\|_{\infty} \bigg]+ \frac{t}{\sqrt{n}} \qquad \forall f\in \calF  \label{ab13},
	\end{align} 
	or
	\begin{align}
	&\bbE[\zeta(m_f(\bx,y))] \leq \frac{1}{n}\sum_{i=1}^n \zeta(m_f(\bx_i,y_i))\nn\\
	&\qquad \qquad +   \frac{2\beta(M)}{\gamma} \bbE\bigg[\sup_{f \in \calF}\bigg\|\frac{1}{n}\sum_{i=1}^n \eps_i f (\bx_i)\bigg\|_{\infty} \bigg] + \frac{t}{\sqrt{n}} \qquad \forall f\in \calF  \label{ab14}.
	\end{align}
	Now, observe that
	\begin{align}
	&\bbE[\zeta(m_f(\bx,y))]\nn\\
	&\qquad =  \bbP\big[m_f(\bx,y)\leq 0 \big]+\bbE[\zeta(m_f(\bx,y))|0\leq m_f(\bx,y)\leq \gamma] \bbP[0\leq m_f(\bx,y)\leq \gamma]\\
	&\qquad \geq \bbP\big(m_f(\bx,y)\leq 0 \big) \label{mota2}.
	\end{align}
	From \eqref{ab14} and \eqref{mota2}, with probability at least $1-\exp(-2t^2)$, 
	\begin{align}
	&\bbP\big[m_f(\bx,y) \leq 0 \big]\leq \frac{1}{n}\sum_{i=1}^n \zeta(m_f(\bx_i,y_i))\nn\\
	&\qquad \qquad +   \frac{2\beta(M)}{\gamma}  \bbE\bigg[\sup_{f \in \calF}\bigg\|\frac{1}{n}\sum_{i=1}^n \eps_i f (\bx_i)\bigg\|_{\infty} \bigg]+ \frac{t}{\sqrt{n}} \qquad \forall f\in \calF \label{mottam1}. 
	\end{align}
Now, let $\gamma_k=2^{-k}$ for all $k \in \bbN$. For any $\gamma \in (0,1]$, there exists a $k \in \bbN$ such that $\gamma \in (\gamma_k,\gamma_{k-1}]$. Then, by applying \eqref{mottam1} with $t$ being replaced by $t+\sqrt{\log k}$ and $\zeta(\cdot)=\zeta_k(\cdot)$ where
\begin{align}
\zeta_k(x):=\begin{cases}0,& \gamma_k \leq x \\ 1-\frac{x}{\gamma_k} &0\leq x \leq \gamma_k\\1,&x \leq 0 \end{cases},
\end{align}
with probability at least $1-\exp(-2(t+\sqrt{\log k})^2)$, we have
\begin{align}
&\bbP\big[m_f(\bx,y) \leq 0 \big]\leq \frac{1}{n}\sum_{i=1}^n \zeta_k(m_f(\bx_i,y_i)) \nn\\
&\qquad \qquad  +   \frac{2\beta(M)}{\gamma}\bbE\bigg[\sup_{f \in \calF}\bigg\|\frac{1}{n}\sum_{i=1}^n \eps_i f (\bx_i)\bigg\|_{\infty} \bigg] +  \frac{t+\sqrt{\log k}}{\sqrt{n}}, \quad \forall f \in \calF \label{mottam1b}. 
\end{align}
By using the union bound, from \eqref{mottam1b}, with probability at least $1-\sum_{k\geq 1} \exp(-2(t+\sqrt{\log k})^2)$, it holds that
\begin{align}
&\bbP\big[m_f(\bx,y) \leq 0 \big]\leq \inf_{k\geq 1} \bigg[\frac{1}{n}\sum_{i=1}^n \zeta_k(m_f(\bx_i,y_i))\nn\\
&\qquad \qquad +   \frac{2\beta(M)}{\gamma} \bbE\bigg[\sup_{f \in \calF}\bigg\|\frac{1}{n}\sum_{i=1}^n \eps_i f(\bx_i)\bigg\|_{\infty} \bigg]+  \frac{t+\sqrt{\log k}}{\sqrt{n}}\bigg], \quad \forall f \in \calF \label{mottam1d}. 
\end{align}
On the other hand, it is easy to see that
\begin{align}
\frac{1}{\gamma_k} &\leq \frac{2}{\gamma} \label{x1},\\
\frac{1}{n}\sum_{i=1}^n \zeta_k(m_f(\bx_i,y_i))&\leq \frac{1}{n}\sum_{i=1}^n \zeta(m_f(\bx_i,y_i)) \label{x2},\\
\sqrt{\log k}&\leq \sqrt{\log \log_2 \frac{1}{\gamma_k}}\leq \sqrt{\log \log_2 \frac{2}{\gamma}} \label{x3},\\
\sum_{k\geq 1} \exp(-2(t+\sqrt{\log k})^2)&\leq \sum_{k\geq 1} k^2 e^{-2t^2}=\frac{\pi^2}{6} e^{-2t^2}\leq 2 e^{-2t^2} \label{x4}.
\end{align}
Hence, by combining \eqref{x1}--\eqref{x4}, and \eqref{mottam1d}, with probability at least $1-2\exp(-2t^2)$, it holds that
\begin{align}
&\bbP\big[m_f(\bx,y) \leq 0 \big]\leq \inf_{\gamma \in (0,1]}\bigg[\frac{1}{n}\sum_{i=1}^n \zeta(m_f(\bx_i,y_i))\nn\\
& \qquad +   \frac{2\beta(M)}{\gamma}\bbE\bigg[\sup_{f \in \calF}\bigg\|\frac{1}{n}\sum_{i=1}^n \eps_i f (\bx_i)\bigg\|_{\infty} \bigg] +  \frac{t+\sqrt{\log \log_2 (2 \gamma^{-1})}}{\sqrt{n}}\bigg], \forall f \in \calF \label{mottam1c}. 
\end{align}
From \eqref{mottam1c} we have
\begin{align}
&\bbP\big[m_f(\bx,y) \leq 0 \big]\leq \inf_{\gamma \in (0,1]}\bigg[\frac{1}{n}\sum_{i=1}^n \zeta\big(m_f(\bx_i,y_i)\big)\nn\\
&\qquad +     \frac{2\beta(M)}{\gamma}\bbE\bigg[\sup_{f \in \calF}\bigg\|\frac{1}{n}\sum_{i=1}^n \eps_i f (\bx_i)\bigg\|_{\infty} \bigg]  +  \frac{t+\sqrt{\log \log_2 (2 \gamma^{-1})}}{\sqrt{n}}\bigg], \quad \forall f \in \calF \label{mottam1f}.  
\end{align} 

This concludes our proof of Theorem \ref{main}. 
\section{Proof of Lemma \ref{lem:new}}\label{lem:new:proof}
First, we have
\begin{align}
&\frac{1}{n}\bbE_{\boldsymbol{\eps}}\bigg[\sup_{h \in \calH}\bigg\|\sum_{i=1}^n \eps_i \tilde{\psi}(h(\bx_i),y_i)\bigg\|_{\infty}  \bigg]\nn\\
&\qquad \leq \frac{1}{n}\bigg(\bbE_{\boldsymbol{\eps}}\bigg[\sup_{h \in \calH}\bigg\|\sum_{i=1}^n \eps_i \bigg(\tilde{\psi}(h(\bx_i),y_i )-\tilde{\psi}(\b0,y_i )\bigg)\bigg\|_{\infty} \bigg] + \bbE_{\boldsymbol{\eps}}\bigg[\sup_{h \in \calH} \bigg\| \sum_{i=1}^n \eps_i \tilde{\psi}(\b0,y_i ) \bigg\|_{\infty} \bigg]\bigg)\label{M1}\\
&\qquad \leq \frac{1}{n}\bigg(\bbE_{\boldsymbol{\eps}}\bigg[\sup_{h \in \calH}\bigg\|\sum_{i=1}^n \eps_i \bigg(\tilde{\psi}(h(\bx_i),y_i )-\tilde{\psi}(\b0,y_i)\bigg)\bigg\|_{\infty}  \bigg]+ \bbE_{\boldsymbol{\eps}}\bigg[\bigg|\sum_{i=1}^n \eps_i \tilde{\psi}(0,y_i) \bigg|\bigg]\bigg)\label{M1b}\\
&\qquad \leq \frac{1}{n}\bigg(\bbE_{\boldsymbol{\eps}}\bigg[\sup_{h \in \calH}\bigg\|\sum_{i=1}^n \eps_i \bigg(\tilde{\psi}(h(\bx_i),y_i )-\tilde{\psi}(\b0,y_i)\bigg)\bigg\|_{\infty}  \bigg]+ \sqrt{\bbE_{\boldsymbol{\eps}}\bigg[\bigg( \sum_{i=1}^n \eps_i \tilde{\psi}(0,y_i) \bigg)^2}\bigg]\bigg)\label{M1c}\\
&\qquad\leq \frac{1}{n}\bbE_{\boldsymbol{\eps}}\bigg[\sup_{h \in \calH}\bigg\|\sum_{i=1}^n \eps_i \bigg(\tilde{\psi} (h(\bx_i))-\tilde{\psi}(0,y_i)\bigg)\bigg\|_{\infty} \bigg]+ \sup_{y \in \calY} \big|\tilde{\psi}(0,y)\big|\frac{1}{\sqrt{n}} \label{M2},
\end{align} where \eqref{M1} follows from the triangular property of the $\infty$-norm \cite{Royden}, \eqref{M1b} follows from the element-wise mapping of $\tilde{\psi}$, and \eqref{M1c} follows Cauchy-Schwarz inequality.  

For any $\tilde{\psi}\in \Psi_{\mu}$, define $\psi(x,y):=\tilde{\psi}(x,y)-\tilde{\psi}(0,y)$ for all $x \in \bbR$ and  $y \in \calY$. Then, we have $\psi(0,y)=0$ for all $y \in \calY$, and $\psi$ is also $\mu$-Lipschitz in $\bx$ with respect to the $\infty$-norm for each fixed $y \in \calY$. 

Now, observe that
\begin{align}
&\frac{1}{n}\bbE_{\boldsymbol{\eps}}\bigg[\sup_{h \in \calH}\bigg\|\sum_{i=1}^n \eps_i \psi \big(h(\bx_i),y_i\big) \bigg\|_{\infty} \bigg]\\
&\qquad = \frac{1}{n}\bbE_{\boldsymbol{\eps}}\bigg[ \sup_{j \in [m]} \sup_{h \in \calH}\bigg|\sum_{i=1}^n \eps_i \big[\psi \big(h(\bx_i),y_i\big)\big]_j\bigg| \bigg]\\
&\qquad \leq  \frac{1}{n}\bbE_{\boldsymbol{\eps}}\bigg[\sup_{j \in [m]}  \sup_{h \in \calH} \sum_{i=1}^n \eps_i \big[\psi \big(h(\bx_i),y_i\big)\big]_j \bone\bigg\{  \sum_{i=1}^n \eps_i \big[\psi \big(h(\bx_i),y_i\big) \big]_j \geq 0 \bigg\}\bigg] \nn\\
&\qquad + \frac{1}{n}\bbE_{\boldsymbol{\eps}}\bigg[\sup_{j \in [m]}  \sup_{h \in \calH} \sum_{i=1}^n -\eps_i [\psi \big(h(\bx_i),y_i\big)]_j \bone\bigg\{   \sum_{i=1}^n -\eps_i [\psi \big(h(\bx_i),y_i\big)]_j  >0 \bigg\}\bigg] \\
&\qquad = \frac{2}{n}\bbE_{\boldsymbol{\eps}}\bigg[\sup_{j \in [m]}  \sup_{h \in \calH} \sum_{i=1}^n \eps_i [\psi \big(h(\bx_i),y_i\big)]_j  \bone\bigg\{   \sum_{i=1}^n \eps_i \big[\psi \big(h(\bx_i),y_i\big)]_j \geq 0 \bigg\}\bigg]  \label{espa}
\end{align} where \eqref{espa} follows from the fact that if $(\eps_1,\eps_2, \cdots, \eps_n)$ is a sequence of i.i.d. Rademacher random variables, $(- \eps_1,- \eps_2, \cdots,- \eps_n)$  
is also a sequence of i.i.d. Rademacher random variables. 

Now, for each fixed $j$ and arbitrarily small $\delta_n>0$, assume that 
\begin{align*}
&\sup_{h \in \calH} \sum_{i=1}^n \eps_i  [\psi \big(h(\bx_i),y_i\big)]_j  \bone\bigg\{   \sum_{i=1}^n \eps_i [\psi \big(h(\bx_i),y_i\big)]_j  \geq 0 \bigg\} \nn\\
&\qquad = \sum_{i=1}^n \eps_i [\psi \big(h_1(\bx_i),y_i\big)]_j  \bone\bigg\{   \sum_{i=1}^n \eps_i [\psi \big(h_1(\bx_i),y_i\big)]_j  \geq 0 \bigg\}+\delta_n
\end{align*}
 for some $h_1 \in \calH$. Consider two cases:
\begin{itemize}
\item Case 1: $\sum_{i=1}^n \eps_i [\psi \big(h_1(\bx_i),y_i\big)]_j  \geq 0$. Then, we have
\begin{align}
&\sup_{h \in \calH} \sum_{i=1}^n \eps_i [\psi \big(h(\bx_i),y_i\big)]_j  \bone\bigg\{   \sum_{i=1}^n \eps_i [\psi \big(h(\bx_i),y_i\big)]_j \geq 0 \bigg\}\nn\\
&\qquad \leq \sum_{i=1}^n \eps_i [\psi \big(h_1(\bx_i),y_i\big)]_j  \\
& \qquad \leq \sup_{h \in \calH_+} \sum_{i=1}^n \eps_i [\psi \big(h(\bx_i),y_i\big)]_j \label{case1}. 
\end{align}
\item Case 2: $\sum_{i=1}^n \eps_i [\psi \big(h_1(\bx_i),y_i\big)]_j <0$.  Then, we have
\begin{align}
&\sup_{h \in \calH} \sum_{i=1}^n \eps_i [\psi \big(h(\bx_i),y_i\big)]_j \bone\bigg\{   \sum_{i=1}^n \eps_i [\psi \big(h(\bx_i),y_i\big)]_j  \geq 0 \bigg\}\nn\\
&\qquad = 0 \\
& \qquad \leq \sup_{h \in \calH_+} \sum_{i=1}^n \eps_i [\psi \big(h(\bx_i),y_i\big)]_j  \label{last}
\end{align} where \eqref{last} follows from the fact that $\sup_{h \in \calH_+} \sum_{i=1}^n \eps_i [\psi \big(h(\bx_i),y_i\big)]_j  \geq 0$ since $\b0 \in \calH_+$ and $\psi(0,y)=0$ for all $y \in \calY$. 
\end{itemize} 
From \eqref{case1} and \eqref{last}, we have
\begin{align}
\frac{1}{n}\bbE_{\boldsymbol{\eps}}\bigg[\sup_{h \in \calH_+}\bigg\|\sum_{i=1}^n \eps_i \psi \big(h(\bx_i),y_i\big)  \bigg\|_{\infty} \bigg]  \leq \frac{2}{n} \bbE_{\boldsymbol{\eps}}\bigg[\sup_{j \in [m]} \sup_{h \in \calH_+} \sum_{i=1}^n \eps_i [\psi \big(h(\bx_i),y_i\big)]_j  \bigg] \label{akey}. 
\end{align}
Now, observe that
\begin{align}
&\frac{1}{n}\bbE_{\boldsymbol{\eps}}\bigg[ \sup_{j \in [m]} \sup_{h \in \calH_+} \sum_{i=1}^n \eps_i [\psi \big(h(\bx_i),y_i\big)]_j \bigg]\nn\\
&\qquad =\frac{1}{n}\bbE_{\eps_1,\eps_2,\cdots,\eps_{n-1}}\bigg[\bbE_{\eps_n}\bigg[\sup_{j \in [m]} \sup_{h \in \calH_+}u_{n-1}(h,j)+ \eps_n [\psi \big(h(\bx_n),y_n\big)]_j  \bigg]\bigg],
\end{align}
where
\begin{align}
u_{n-1}(h,j):=\sum_{i=1}^{n-1} \eps_i [\psi \big(h(\bx_i),y_i\big)]_j \label{eq116}.
\end{align}
On the other hand, by the definition of supreme, for any $\delta_n>0$ and fixed $j \in [k]$, there exist $h_1,h_2,h_3,h_4 \in \calH_+$ satisfying
\begin{align}
\sup_{h \in \calH_+} u_{n-1}(h,j)+[\psi \big(h(\bx_n),y_n\big)]_j &\leq u_{n-1}(h_1,j)+ [\psi \big(h_1(\bx_n),y_n\big)]_j +\delta_n,\\
\sup_{h \in \calH_+} u_{n-1}(h,j)-[\psi \big(h(\bx_n),y_n\big)]_j &\leq u_{n-1}(h_2,j)- [\psi \big(h_2(\bx_n),y_n\big)]_j +\delta_n\\
\sup_{h \in \calH_+} -u_{n-1}(h,j)-[\psi(h(\bx_n,y_n))]_j&\leq -u_{n-1}(h_3,j)- [\psi \big(h_3(\bx_n),y_n\big)]_j +\delta_n,\\
\sup_{h \in \calH_+} -u_{n-1}(h,j)+[\psi(h(\bx_n,y_n))]_j&\leq -u_{n-1}(h_4,j)+ [\psi \big(h_4(\bx_n),y_n\big)]_j +\delta_n.
\end{align}

Since $\eps_n$ is uniformly distributed over $\{-1,1\}$, we have
\begin{align}
&\bbE_{\eps_n}\bigg[\sup_{j \in [m]} \sup_{h \in \calH_+}u_{n-1}(h,j)+ \eps_n [\psi \big(h(\bx_n),y_n\big)]_j   \bigg]\nn\\
&\qquad =\frac{1}{2}\sup_{j \in [m]} \sup_{h \in \calH_+}u_{n-1}(h,j)
+  [(\psi ( h(\bx_n),y_n)]_j   \nn\\
&\qquad \qquad  + \frac{1}{2}\sup_{j \in [m]} \sup_{h \in \calH_+}u_{n-1}(h,j)- [\psi( h(\bx_n),y_n)]_j \label{c1}.
\end{align} 
Now, since $\b0 \in \calH_+$ (the zero-function) and $\psi(0,y)=0$ for all $y \in \calY$, it holds that
\begin{align}
\sup_{h \in \calH_+}u_{n-1}(h,j)+  \psi ([h(\bx_n)]_j,y_n)\geq u_{n-1}(\b0,j)+  \psi(0,y_n)=0.
\end{align}
Similarly, we also have
\begin{align}
\sup_{h \in \calH_+} -u_{n-1}(h,j)-   [\psi( h(\bx_n),y_n)]_j & \geq 0,\\
\sup_{h \in \calH_+}u_{n-1}(h,j)-  [\psi( h(\bx_n),y_n)]_j &\geq 0,\\
\sup_{h \in \calH_+} -u_{n-1}(h,j)+  [\psi( h(\bx_n),y_n)]_j &\geq 0. 
\end{align}
Hence, we have
\begin{align}
&\frac{1}{2} \sup_{j\in [m]} \sup_{h \in \calH_+}u_{n-1}(h,j)+   [\psi( h(\bx_n),y_n)]_j  + \frac{1}{2} \sup_{j \in [m]} \sup_{h \in \calH_+}u_{n-1}(h,j)- [\psi( h(\bx_n),y_n)]_j  \nn\\
&\qquad \leq 2\sup_{j \in [m]} \bigg[\frac{1}{2}\bigg(\sup_{h \in \calH_+}u_{n-1}(h,j)+   [\psi( h(\bx_n),y_n)]_j \bigg) \nn\\
&\qquad \qquad \qquad + \frac{1}{2}\bigg(\sup_{h \in \calH_+}  u_{n-1}(h,j)-   [\psi( h(\bx_n),y_n)]_j \bigg) \bigg] \label{M4}.
\end{align}
Note that for $m=1$, we have
\begin{align}
&\frac{1}{2} \sup_{j\in [m]} \sup_{h \in \calH_+}u_{n-1}(h,j)+   [\psi( h(\bx_n),y_n)]_j + \frac{1}{2} \sup_{j \in [m]} \sup_{h \in \calH_+}u_{n-1}(h,j)-  [\psi( h(\bx_n),y_n)]_j  \nn\\
&\qquad =\sup_{j \in [m]}\bigg[\frac{1}{2}\bigg(\sup_{h \in \calH_+}u_{n-1}(h,j)+   [\psi( h(\bx_n),y_n)]_j \bigg) \nn\\
&\qquad \qquad \qquad + \frac{1}{2}\bigg(\sup_{h \in \calH_+}  u_{n-1}(h,j)-   [\psi( h(\bx_n),y_n)]_j \bigg)\bigg]\label{M5}. 
\end{align}

On the other hand, for each fixed $j \in [m]$, we also have
\begin{align}
&\frac{1}{2}\bigg(\sup_{h \in \calH_+}u_{n-1}(h,j)+   [\psi( h(\bx_n),y_n)]_j  \bigg)+ \frac{1}{2}\bigg(\sup_{h \in \calH_+}u_{n-1}(h,j)-   [\psi( h(\bx_n),y_n)]_j  \bigg)\nn\\
&\leq \frac{1}{2}\bigg(u_{n-1}(h_1,j)+   [\psi( h_1(\bx_n),y_n)]_j \bigg)+ \frac{1}{2}\bigg(u_{n-1}(h_2,j)- [\psi( h_2(\bx_n),y_n)]_j  \bigg)+2 \delta_n\\
&=\frac{1}{2}\bigg\{u_{n-1}(h_1,j)+  u_{n-1}(h_2,j)+   \bigg( [\psi( h_1(\bx_n),y_n)]_j - [\psi( h_2(\bx_n),y_n)]_j  \bigg)\bigg\}+2\delta_n\\
& \leq \frac{1}{2}\bigg[u_{n-1}(h_1,j)+u_{n-1}(h_2,j)\bigg]+ \frac{1}{2} \mu \big|[h_1(\bx_n)]_j-[h_2(\bx_n)]_j\big|\bigg]
+2\delta_n \label{eq152p} \\
& \leq \frac{1}{2}\bigg[u_{n-1}(h_1,j)+u_{n-1}(h_2,j)\bigg]+ \frac{1}{2} \mu s_{12} \big([h_1(\bx_n)]_j-[h_2(\bx_n)]_j\big)
+2\delta_n  \label{eq144p},
\end{align} where \eqref{eq152p} follows from the $\mu$-Lipschitz of the function $\psi$ with respect to $x$ for each fixed $y$. Here, $s_{12}:= \sgn( [h_1(\bx_n)]_j-[h_2(\bx_n)]_j)$.  

Now, for a given $s_{12} \in \{-1,+1\}$, we have
\begin{align}
&\frac{1}{2}\bigg[u_{n-1}(h_1,j)+u_{n-1}(h_2,j)\bigg]+\frac{1}{2}\mu  s_{12} \big([h_1(\bx_n)]_j- [h_2(\bx_n)]_j\big)\nn\\
&\qquad= \frac{1}{2}\bigg[u_{n-1}(h_1,j)+ s_{12}\mu  [h_1(\bx_n)]_j\bigg]+ \frac{1}{2}\bigg[ u_{n-1}(h_2,j)-\mu s_{12}  [h_2(\bx_n)]_j \bigg]\\
&\qquad \leq \frac{1}{2}\sup_{h \in \calH_+} \bigg[u_{n-1}(h,j)+ s_{12} \mu   [h(\bx_n)]_j\bigg]+ \frac{1}{2}\sup_{h \in \calH_+}\bigg[ u_{n-1}(h,j)-\mu s_{12}  [h(\bx_n)]_j\bigg]\\
&\qquad =\bbE_{\eps_n} \bigg[\sup_{h \in \calH_+} u_{n-1}(h,j)+ \eps_n s_{12} \mu  [h(\bx_n)]_j \bigg]\\
&\qquad= \bbE_{\eps_n} \bigg[\sup_{h \in \calH_+} u_{n-1}(h,j)+ \eps_n\mu [h(\bx_n)]_j\bigg] \label{lady1c}
\end{align} where \eqref{lady1c} follows from $s_{12} \eps_n \sim \eps_n$ for all fixed $s_{12} \in \{-1,+1\}$.

From  \eqref{M4}, \eqref{M5}, \eqref{eq144p}, and  \eqref{lady1c}, we obtain
\begin{align}
&\bbE_{\eps_n}\bigg[\sup_{j \in [m]} \sup_{h \in \calH_+}u_{n-1}(h,j)+ \eps_n [\psi( h(\bx_n),y_n)]_j  \bigg]\nn\\
&\qquad \leq (2\{m>1\}+1\{m=1\}) \bbE_{\eps_n}\bigg[\sup_{j \in [m]} \sup_{h \in \calH_+} u_{n-1}(h,j)+ \eps_n\mu [h(\bx_n)]_j\bigg] \label{mut2p}.
\end{align}  

By combining \eqref{M2} and \eqref{mut2p} and the fact that $\delta_n$ can be arbitrarily small, we have
\begin{align}
&\frac{1}{n}\bbE_{\boldsymbol{\eps}}\bigg[\sup_{h \in \calH_+}\bigg\|\sum_{i=1}^n \eps_i  \psi( h(\bx_i),y_i) \bigg\|_{\infty} \bigg]\nn\\
&\qquad \qquad \leq  2(2\{m>1\}+1\{m=1\})\frac{1}{n}\bbE_{\eps}\bigg[\sup_{j \in [m]} \sup_{h \in \calH_+} u_{n-1}(h,j)+ \eps_n\mu [h(\bx_n)]_j\bigg]\nn\\
&\qquad \qquad \qquad +\sup_{y \in \calY} \big|\tilde{\psi}(0,y)\big|\frac{1}{\sqrt{n}} \label{M10p}.
\end{align}

By continuing this process  (peeling) for $n-1$ more times, we have
\begin{align}
&\bbE_{\boldsymbol{\eps}}\bigg[\sup_{j \in [m]} \sup_{h \in \calH_+} u_{n-1}(h,j)+ \eps_n\mu [h(\bx_n)]_j\bigg]\nn\\
&\qquad \leq 2(2\{m>1\}+1\{m=1\})\mu \bbE_{\boldsymbol{\eps}}\bigg[\sup_{j \in [m]} \sup_{h \in \calH_+}\sum_{i=1}^n \eps_i [h(\bx_n)]_j\bigg]\\
&\qquad \leq 2(2\{m>1\}+1\{m=1\})\mu \bbE_{\boldsymbol{\eps}}\bigg[ \sup_{h \in \calH_+}\sup_{j \in [m]}\bigg|\sum_{i=1}^n \eps_i [h(\bx_n)]_j\bigg|\bigg] \\
&\qquad= 2(2\{m>1\}+1\{m=1\})\mu \bbE_{\boldsymbol{\eps}}\bigg[ \sup_{h \in \calH_+}\bigg\|\sum_{i=1}^n \eps_i h(\bx_n)\bigg\|_{\infty}\bigg] \label{pkey1p}. 
\end{align}

From \eqref{M10p} and \eqref{pkey1p}, we obtain
\begin{align}
&\frac{1}{n}\bbE_{\boldsymbol{\eps}}\bigg[\sup_{h \in \calH_+}\bigg\|\sum_{i=1}^n \eps_i  \psi( h(\bx_i),y_i)  \bigg\|_{\infty} \bigg]\nn\\
&\qquad \leq \frac{2\mu}{n} (2\{m>1\}+1\{m=1\}) \bbE_{\boldsymbol{\eps}}\bigg[ \sup_{h \in \calH_+}\bigg\|\sum_{i=1}^n \eps_i  \psi( h(\bx_i),y_i) \bigg\|_{\infty}\bigg] \nn\\
&\qquad \qquad +\sup_{y \in \calY} \big|\tilde{\psi}(0,y)\big|\frac{1}{\sqrt{n}} \label{pkey3p}.
\end{align}

This concludes our Lemma \ref{lem:new}.

\section{Proof of Lemma \ref{mato1a}} \label{mato1a:proof}
Assume that $\psi(\bx)=(\psi_1(\bx),\psi_2(\bx),\cdots, \psi_K(\bx))$ for all $\bx \in \bbR^L$. Then, we have
\begin{align}
\bigg\|\sum_{i=1}^n \eps_i \psi(h(\bx_i))\bigg\|_1 =\sum_{k=1}^K \bigg|\sum_{i=1}^n \eps_i \psi_k(h(\bx_i))\bigg|  \label{AQ3}. 
\end{align}
Now, given each fixed  $\bold{\eps}=(\eps_1,\eps_2,\cdots,\eps_n) \in \bbR^n$, let $h_{\eps}\in \calH$ such that $\sum_{k=1}^K \big|\sum_{i=1}^n \psi_k(\eps_i  h(\bx_i)) \big| \leq \sum_{k=1}^K \big|\sum_{i=1}^n \psi_k(\eps_i  h_{\bold{\eps}}(\bx_i))) \big|+\delta_n$ for any $\delta_n>0$. Then, we have
\begin{align}
&\sum_{k=1}^K \bigg|\sum_{i=1}^n \psi_k(\eps_i  h_{\eps} (\bx_i)) \bigg|\nn\\
 &\qquad \leq \sup_{\nu_1,\nu_2,\cdots, \nu_K \in \{-1,+1\}} \sum_{k=1}^K\nu_k \bigg(\sum_{i=1}^n \eps_i \psi_k( h_{\eps} (\bx_i))\bigg)\\
 &\qquad = \sup_{\nu_1,\nu_2,\cdots, \nu_K \in \{-1,+1\}} \sum_{k=1}^K \bigg(\sum_{i=1}^n \eps_i \psi_k(\nu_k h_{\eps} (\bx_i))\bigg) \label{amet0}\\
 &\qquad \leq  \sup_{\nu_1,\nu_2,\cdots, \nu_K \in \{-1,+1\}} \sum_{k=1}^K \bigg(\sum_{i=1}^n \eps_i \psi_k^{(\nu)} (h_{\eps} (\bx_i))\bigg) \label{amet2},
\end{align}  where \eqref{amet0} follows from the assumption that $\psi$ is an odd function,  and \eqref{amet2} follows by defining  $\psi^{(\nu)}(\bx)=(\psi_1(\nu_1 \bx),\psi_2(\nu_2 \bx), \cdots, \psi_K(\nu_K \bx)$ for any $\nu \in \{-1,+1\}^K$. 

Now, define
\begin{align}
\Psi=\big\{\psi: \bbR^L \to \bbR^K, \mbox{st.}\enspace \|\psi(\bx)-\psi(\by)\|_1 \leq \mu \|\bx-\by\|_1, \psi(-\bx)=-\psi(\bx), \quad \forall \bx, \by \in \bbR^L \big\}.
\end{align}
Then, it is easy to see that $\psi^{(\nu)} \in \Psi$ for all $\nu \in \{-1,+1\}^K$ since $\psi \in \Psi$.

Hence, from \eqref{AQ3} and \eqref{amet2} we have
\begin{align}
\sup_{h \in \calH} \bigg\|\sum_{i=1}^n \eps_i \psi(h(\bx_i))\bigg\|_1 \leq \sup_{h \in \calH} \sup_{\psi \in \Psi} \sum_{k=1}^K \bigg(\sum_{i=1}^n \eps_i \psi_k(h(\bx_i))\bigg) \label{amet}.
\end{align}

It follows that
\begin{align}
\bbE_{\bold{\eps}}\bigg[\sup_{h \in \calH} \bigg\|\sum_{i=1}^n \eps_i \psi(h(\bx_i))\bigg\|_1 \bigg]&\leq \bbE_{\bold{\eps}}\bigg[\sup_{h \in \calH} \sup_{\psi \in \Psi}\sum_{k=1}^K \sum_{i=1}^n \eps_i \psi_k(  h(\bx_i))\bigg] \label{mutat}. 
\end{align}
Now, define the following function $g: \bbR^K \to \bbR$  such that $g(x_1,x_2,\cdots,x_K)=\sum_{k=1}^K x_k$. Then, we have
\begin{align}
\bbE_{\bold{\eps}}\bigg[\sup_{\psi \in \Psi} \sup_{h \in \calH} \bigg\|\sum_{i=1}^n \eps_i \psi(h(\bx_i))\bigg\|_1 \bigg]= \bbE_{\bold{\eps}}\bigg[\sup_{\psi \in \Psi} \sup_{h \in \calH}  \sum_{i=1}^n \eps_i g\circ \psi(h(\bx_i))\bigg]  \label{AQ7}. 
\end{align}
Let $\tilde{\psi}: \bbR^L \to \bbR$ such that $\tilde{\psi}(\bx)=g\circ \psi(\bx)=\sum_{k=1}^K \psi_k(\bx)$ for any $\bx \in \bbR^L$. Then, we have
\begin{align}
&\big|\tilde{\psi}(\bx)-\tilde{\psi}(\by)\big|\nn\\
&\qquad =\big| \bone^T (\psi (\bx)-\psi (\by))\big|\\
&\qquad \leq  \|\bone\|_{\infty} \|\psi(\bx)-\psi(\by)\|_1\\
&\qquad= \|\psi(\bx)-\psi(\by)\|_1\\
&\qquad \leq \mu \|\bx-\by\|_1 \label{eq37}
\end{align}  for all $\bx, \by \in \bbR^L$. 

Now, let 
\begin{align}
\tilde{\Psi}=\big\{\tilde{\psi}: \bbR^L \to \bbR \enspace \mbox{st.} \enspace |\tilde{\psi}(\bx)-\tilde{\psi}(\by)|\leq \mu \|\bx-\by\|_1, \tilde{\psi}(-\bx)=-\tilde{\psi}(\bx), \quad \forall \bx, \by \in \bbR^L \big\}.
\end{align}

Then,  from \eqref{AQ7} we have
\begin{align}
&\bbE_{\bold{\eps}}\bigg[\sup_{\psi \in \Psi} \sup_{h \in \calH} \bigg\|\sum_{i=1}^n \eps_i \psi(h(\bx_i))\bigg\|_1 \bigg]\nn\\
&\qquad \leq \bbE_{\bold{\eps}}\bigg[\sup_{\tilde{\psi}\in \tilde{\Psi}} \sup_{h \in \calH}  \sum_{i=1}^n \eps_i \tilde{\psi}(h(\bx_i))\bigg] \\
&= \bbE_{\eps_1,\eps_2,\cdots,\eps_{n-1}}\bigg[\sup_{\tilde{\psi}\in \tilde{\Psi}}\sup_{h \in \calH}  \sum_{i=1}^{n-1} \eps_i \tilde{\psi}(h(\bx_i))+ \frac{1}{2}\tilde{\psi}(h(\bx_n)) \nn\\
&\qquad +  \sup_{\tilde{\psi}\in \tilde{\Psi}}\sup_{h \in \calH}  \sum_{i=1}^{n-1} \eps_i \tilde{\psi}(h(\bx_i))- \frac{1}{2}\tilde{\psi}(h(\bx_n))\bigg] \label{B}\\
&= \bbE_{\eps_1,\eps_2,\cdots,\eps_{n-1}}\bigg[\sup_{\tilde{\psi}\in \tilde{\Psi}}\sup_{h \in \calH}  \sum_{i=1}^{n-1} \eps_i \tilde{\psi}(h(\bx_i))+ \frac{1}{2}\tilde{\psi}(h(\bx_n))\bigg] \nn\\
&\qquad +  \bbE_{\eps_1,\eps_2,\cdots,\eps_{n-1}}\bigg[\sup_{\tilde{\psi}\in \tilde{\Psi}}\sup_{h \in \calH}  \sum_{i=1}^{n-1} \eps_i \tilde{\psi}(h(\bx_i))- \frac{1}{2}\tilde{\psi}(h(\bx_n))\bigg] \label{B}\\
&= \bbE_{\eps_1,\eps_2,\cdots,\eps_{n-1}}\bigg[\sup_{\tilde{\psi}\in \tilde{\Psi}}\sup_{h \in \calH}  \sum_{i=1}^{n-1} \eps_i \tilde{\psi}(h(\bx_i))+ \frac{1}{2}\tilde{\psi}(h(\bx_n))\bigg] \nn\\
&\qquad +  \bbE_{\eps_1,\eps_2,\cdots,\eps_{n-1}}\bigg[\sup_{\tilde{\psi}\in \tilde{\Psi}}\sup_{h \in \calH} - \sum_{i=1}^{n-1} \eps_i \tilde{\psi}(h(\bx_i))- \frac{1}{2}\tilde{\psi}(h(\bx_n))\bigg]  \label{C}\\
&= \bbE_{\eps_1,\eps_2,\cdots,\eps_{n-1}}\bigg[\sup_{\tilde{\psi}\in \tilde{\Psi}}\sup_{h \in \calH}  \sum_{i=1}^{n-1} \eps_i \tilde{\psi}(h(\bx_i))+ \frac{1}{2}\tilde{\psi}(h(\bx_n))\bigg] \nn\\
&\qquad +  \sup_{\tilde{\psi}\in \tilde{\Psi}}\sup_{h \in \calH} - \sum_{i=1}^{n-1} \eps_i \tilde{\psi}(h(\bx_i))- \frac{1}{2}\tilde{\psi}(h(\bx_n))\bigg]  \label{D},
\end{align} where \eqref{C} follows from the fact that $(-\eps_1,-\eps_2, \cdots, -\eps_{n-1})$ are Rademacher random variables if and only if $(\eps_1,\eps_2, \cdots, \eps_{n-1})$ are Rademacher random variables.

On the other hand, since $\psi$ is odd, $\tilde{\psi}$ is also odd. Then, for any $\tilde{\psi}\in \tilde{\Psi}$ we have
\begin{align}
&\sup_{h \in \calH}  \sum_{i=1}^{n-1} \eps_i \tilde{\psi}(h(\bx_i))+ \frac{1}{2}\tilde{\psi}(h(\bx_n))\nn\\
&\qquad =\sup_{h \in \calH} - \sum_{i=1}^{n-1} \eps_i \tilde{\psi}(-h(\bx_i))- \frac{1}{2}\tilde{\psi}(-h(\bx_n)) \label{k1}\\
&\qquad = \sup_{h \in \calH} - \sum_{i=1}^{n-1} \eps_i \tilde{\psi}(h(\bx_i))- \frac{1}{2}\tilde{\psi}(h(\bx_n)) \label{k2},
\end{align} where \eqref{k1} follows from the fact that $\tilde{\psi}$ is odd, and \eqref{k2} follows from the assumption that $-h \in \calH$ iff $h \in \calH$.

Now, define $u_{n-1}(h):=\sum_{i=1}^{n-1} \eps_i \tilde{\psi}(h(\bx_i))) $.  Then, from \eqref{k3} for each given tuple $(\eps_1,\eps_2,\cdots,\eps_{n-1})$  any $\delta_n>0$, there exist $h_1\in \calH$ and $h_2 \in \calH$ and $\tilde{\psi} \in \tilde{\Psi}$ such that
\begin{align}
&\sup_{\tilde{\psi}\in \tilde{\Psi}}\sup_{h \in \calH}  \sum_{i=1}^{n-1} \eps_i \tilde{\psi}(h(\bx_i))+ \frac{1}{2}\tilde{\psi}(h(\bx_n))+  \sup_{\tilde{\psi}\in \tilde{\Psi}}\sup_{h \in \calH}  -\sum_{i=1}^{n-1} \eps_i \tilde{\psi}(h(\bx_i))- \frac{1}{2}\tilde{\psi}(h(\bx_n)))\\
&\qquad= \sum_{i=1}^{n-1} \eps_i \tilde{\psi}(h_1(\bx_i))+ \frac{1}{2}\tilde{\psi}_1(h_1(\bx_n))- \sum_{i=1}^{n-1} \eps_i \tilde{\psi}(h_2(\bx_i))- \frac{1}{2}\tilde{\psi}(h_2(\bx_n))+\delta_n\\
&\qquad= \sum_{i=1}^{n-1} \eps_i \tilde{\psi}(h_1(\bx_i))- \sum_{i=1}^{n-1} \eps_i \tilde{\psi}(h_2(\bx_i))+\frac{1}{2}\tilde{\psi}(h_1(\bx_n))- \frac{1}{2}\tilde{\psi}(h_2(\bx_n))+\delta_n \label{AQ10}. 
\end{align}
Observe that
\begin{align}
\big|\tilde{\psi}(h_1(\bx_n))-\tilde{\psi}(h_2(\bx_n))\big| &\leq \mu \big\|h_1(\bx_n)-h_2(\bx_n)\big\|_1 \label{eq44} \\
&=\mu \sum_{k=1}^L \big|[h_1(\bx_n)]_k-[h_2(\bx_n)]_k\big|\\
&=\mu \sum_{k=1}^L s_{12,k} \big( [h_1(\bx_n)]_k- [h_2(\bx_n)]_k \big) \label{M2a}
\end{align} where \eqref{eq44} follows from \eqref{eq37} and $s_{12,k}:=\sgn \big( [h_1(\bx_n)]_k- [h_2(\bx_n)]_k \big)$ for all $k \in [L]$. 

From \eqref{AQ10} and \eqref{M2a},  we obtain
\begin{align}
&\sup_{h \in \calH}  \sum_{i=1}^{n-1} \eps_i \tilde{\psi}(h(\bx_i)))+ \frac{1}{2}\tilde{\psi}(h(\bx_n))+  \sup_{h \in \calH}  -\sum_{i=1}^{n-1} \eps_i \tilde{\psi}(h(\bx_i))- \frac{1}{2}\tilde{\psi}(h(\bx_n))\\
&\qquad \leq \sum_{i=1}^{n-1} \eps_i \tilde{\psi}(h_1(\bx_i))- \sum_{i=1}^{n-1} \eps_i \tilde{\psi}(h_2(\bx_i))+\frac{\mu}{2}\sum_{k=1}^L s_{12,k} \big( [h_1(\bx_n)]_k- [h_2(\bx_n)]_k \big) \label{T1}.
\end{align} 
Now, given any fixed sequence $\{s_{12,k}\}_{k=1}^L \in \{-1,+1\}^L$, we have
\begin{align}
&\sum_{i=1}^{n-1} \eps_i \tilde{\psi}(h_1(\bx_i))-  \sum_{i=1}^{n-1} \eps_i \tilde{\psi}(h_2(\bx_i))+\frac{\mu}{2}\sum_{k=1}^L s_{12,k} \big( [h_1(\bx_n)]_k- [h_2(\bx_n)]_k\big) \nn\\
&\qquad =\sum_{i=1}^{n-1} \eps_i \tilde{\psi}(h_1(\bx_i))+ \frac{1}{2}\sum_{k=1}^L s_{12,k}[h_1(\bx_n)]_k - \sum_{i=1}^{n-1} \eps_i \tilde{\psi}(h_2(\bx_i))-\frac{\mu}{2}\sum_{k=1}^L s_{12,k}[h_2(\bx_n)]_k\\
&\qquad \leq \sup_{\tilde{\psi} \in \tilde{\Psi}} \sup_{h \in \calH} \sum_{i=1}^{n-1} \eps_i \tilde{\psi}(h(\bx_i))+ \frac{\mu}{2}\sum_{k=1}^L s_{12,k}[h(\bx_n)]_k \nn\\
&\qquad \qquad +\sup_{\tilde{\psi} \in \tilde{\Psi}} \sup_{h \in \calH} -\sum_{i=1}^{n-1} \eps_i \tilde{\psi}(h(\bx_i))-\frac{\mu}{2}\sum_{k=1}^L s_{12,k}[h(\bx_n)]_k\\
&\qquad = \sup_{\tilde{\psi} \in \tilde{\Psi}} \sup_{h \in \calH} \sum_{i=1}^{n-1} \eps_i \tilde{\psi}(h(\bx_i))+ \frac{\mu}{2}\sum_{k=1}^L s_{12,k}[h(\bx_n)]_k \nn\\
&\qquad \qquad +\sup_{\tilde{\psi} \in \tilde{\Psi}} \sup_{h \in \calH} \sum_{i=1}^{n-1} \eps_i \tilde{\psi}(h(\bx_i))-\frac{\mu}{2}\sum_{k=1}^L s_{12,k}[h(\bx_n)]_k\label{akamo}\\
&\qquad = \bbE_{\tilde{\eps}_n} \bigg[\sup_{\tilde{\psi} \in \tilde{\Psi}}\sup_{h \in \calH} \sum_{i=1}^{n-1} \eps_i \tilde{\psi}(h(\bx_i)))+\mu \tilde{\eps}_n \sum_{k=1}^L s_{12,k}[h(\bx_n)]_k\bigg] \label{A10}\\
&\qquad = \bbE_{\tilde{\eps}_n} \bigg[\sup_{\tilde{\psi} \in \tilde{\Psi}}\sup_{h \in \calH} \sum_{i=1}^{n-1} \eps_i \tilde{\psi}(h(\bx_i)))+\mu \tilde{\eps}_n \sum_{k=1}^L [s_{12}\odot h(\bx_n)]_k\bigg] \label{A11}\\
&\qquad = \bbE_{\tilde{\eps}_n} \bigg[\sup_{\tilde{\psi} \in \tilde{\Psi}}\sup_{h \in \calH} \sum_{i=1}^{n-1} \eps_i \tilde{\psi}(s_{12}\odot s_{12} \odot h(\bx_i)))+\mu \tilde{\eps}_n \sum_{k=1}^L [s_{12}\odot h(\bx_n)]_k\bigg] \label{A12}\\
&\qquad = \bbE_{\tilde{\eps}_n} \bigg[\sup_{s_{12 \in \{-1,+1\}^L}} \sup_{\tilde{\psi} \in \tilde{\Psi}}\sup_{h \in \calH} \sum_{i=1}^{n-1} \eps_i s_{12}\odot \tilde{\psi}(s_{12} \odot h(\bx_i)))+\mu \tilde{\eps}_n \sum_{k=1}^L [s_{12}\odot h(\bx_n)]_k\bigg] \label{A13}\\
&\qquad = \bbE_{\tilde{\eps}_n} \bigg[\sup_{s_{12 \in \{-1,+1\}^L}}\sup_{\tilde{\psi} \in \tilde{\Psi}}\sup_{h \in \calH} \sum_{i=1}^{n-1} \eps_i s_{12}\odot \tilde{\psi}(h_{s_{12}}(\bx_i)))+\mu \tilde{\eps}_n \sum_{k=1}^L [h_{s_{12}}(\bx_n)]_k\bigg] \label{A13b}\\
&\qquad \leq \bbE_{\tilde{\eps}_n} \bigg[\sup_{s_{12 \in \{-1,+1\}^L}}\sup_{\tilde{\psi} \in \tilde{\Psi}}\sup_{h \in \calH_+} \sum_{i=1}^{n-1} \eps_i s_{12}\odot \tilde{\psi}(h(\bx_i)))+\mu \tilde{\eps}_n \sum_{k=1}^L [h(\bx_n)]_k\bigg] \label{A14}\\
&\qquad \leq \bbE_{\tilde{\eps}_n} \bigg[\sup_{\tilde{\psi} \in \tilde{\Psi}}\sup_{h \in \calH_+} \sum_{i=1}^{n-1} \eps_i \tilde{\psi}(h(\bx_i)))+\mu \tilde{\eps}_n \sum_{k=1}^L [h(\bx_n)]_k\bigg] \label{A15},
\end{align} where \eqref{akamo} follows from the fact that $-\tilde{\psi} \in \tilde{\Psi}$ if and only if $\tilde{\psi} \in \tilde{\Psi}$, 
$\tilde{\eps}_n$ is a Rademacher random variable which is independent of $(\eps_1,\eps_2,\cdots,\eps_{n-1})$ in \eqref{A10}, \eqref{A13} follows from the fact that $\tilde{\psi}$ is odd, \eqref{A14} follows from the assumption that $h_{\sigma} \in \calH_+$ for all $\sigma \in \{-1,+1\}^L$, \eqref{A15} follows from the fact that $\tilde{\psi}^{s_{12}}(\bx):=s_{12}\odot \tilde{\psi}(\bx) \in \tilde{\Psi}$ for each fixed $s_{12} \in \{-1,+1\}^L$.

Continue this peeling process $n-1$ times more, we finally have
\begin{align}
&\bbE_{\bold{\eps}}\bigg[\sup_{\psi \in \Psi} \sup_{h \in \calH} \bigg\|\sum_{i=1}^n \eps_i \psi(h(\bx_i))\bigg\|_1 \bigg]\nn\\
&\qquad \leq \mu \bbE_{\tilde{\eps}_1,\tilde{\eps}_2,\cdots,\tilde{\eps}_{n-1},\tilde{\eps}_n} \bigg[ \sup_{h \in \calH}  \sum_{i=1}^n \tilde{\eps}_i \sum_{k=1}^K[h(\bx_i)]_k \bigg]\\
&\qquad=\mu \bbE_{\eps}\bigg[  \sup_{h \in \calH} \sum_{i=1}^n \eps_i \sum_{k=1}^K[h(\bx_i)]_k \bigg]\\
&\qquad= \mu \bbE_{\eps}\bigg[  \sup_{h \in \calH} \sum_{k=1}^K\sum_{i=1}^n \eps_i [h(\bx_i)]_k \bigg]\\
&\qquad \leq \mu \bbE_{\eps}\bigg[  \sup_{h \in \calH} \bigg\|\sum_{i=1}^n \eps_i  h(\bx_i) \bigg\|_1 \bigg] \label{Amat}. 
\end{align}

This concludes our proof of Lemma \ref{mato1a}.
\end{document}